\documentclass{article}

\usepackage{times}
\usepackage{graphicx} 
\usepackage{subfigure} 

\usepackage{natbib}

\usepackage{algorithm}
\usepackage{algorithmic}

\usepackage{hyperref}

\usepackage{jmlr2e1} 
\usepackage[toc,page]{appendix}
\usepackage{amssymb,amsmath,amsfonts}
\usepackage{mathrsfs}
\usepackage{dsfont}
\usepackage{wasysym}
\newcommand{\norm}[1]{\left\lVert#1\right\rVert}
\newcommand{\expect}[1]{\mathbb{E}\left[{#1}\right]}
\newcommand{\prob}[1]{\mathbb{P}\left[{#1}\right]}
\newcommand{\given}{\; \big\vert \;} 
\newcommand{\bydef}{:=}

\newcommand{\inner}[2]{\langle #1, #2 \rangle}

\usepackage{subfigure}
\usepackage{mathnotations2}

\graphicspath{{./Figures/}}

\newtheorem{mytheorem}{Theorem}

\newtheorem{mylemma}{Lemma}
\newtheorem{mycorollary}{Corollary}

\newtheorem{mydefinition}{Definition}

\newtheorem{myresult}{Result}
\newcommand{\beq}{\begin{equation}}
\newcommand{\eeq}{\end{equation}}
\newcommand{\beqn}{\begin{equation*}}
\newcommand{\eeqn}{\end{equation*}}
\newcommand{\beqa}{\begin{eqnarray}}
\newcommand{\eeqa}{\end{eqnarray}}
\newcommand{\beqan}{\begin{eqnarray*}}
\newcommand{\eeqan}{\end{eqnarray*}}

\newcommand{\argmin}{\mathop{\mathrm{argmin}}}

\begin{document} 

\title{Online Learning in Kernelized Markov Decision
Processes}
\author{\name{Sayak Ray Chowdhury} \email{sayak@iisc.ac.in}\\ 
\addr      Electrical Communication Engineering,\\
Indian Institute of Science,\\
Bangalore 560012, India\\
\\
\name{Aditya Gopalan} \email{aditya@iisc.ac.in}\\ 
\addr     Electrical Communication Engineering,\\
Indian Institute of Science,\\
Bangalore 560012, India\\
}


\maketitle

\begin{abstract}
We consider online learning for minimizing regret in unknown, episodic Markov decision processes (MDPs) with continuous states and actions. We develop variants of the UCRL and posterior sampling algorithms that employ nonparametric Gaussian process priors to generalize across the state and action spaces. When the transition and reward functions of the true MDP are members of the associated Reproducing Kernel Hilbert Spaces of functions induced by symmetric psd kernels (frequentist setting), we show that the algorithms enjoy sublinear regret bounds. The bounds are in terms of explicit structural parameters of the kernels, namely a novel generalization of the information gain metric from kernelized bandit, and highlight the influence of transition and reward function structure on the learning performance. Our results are applicable to multi-dimensional state and action spaces with composite kernel structures, and generalize results from the literature on kernelized bandits, and the adaptive control of parametric linear dynamical systems with quadratic costs.
\end{abstract}

\section{INTRODUCTION}


The goal of reinforcement learning (RL) is to learn optimal behavior (policies) by repeated interaction with an unknown environment, usually modelled as a Markov Decision Process (MDP). Performance is typically measured by the amount of interaction, in terms of episodes or rounds, needed to arrive at an optimal (or near-optimal) policy; this is also known as the sample complexity of RL \citep{StrLiLit09:PACMDP}. The sample complexity objective encourages efficient exploration across states and actions, but, at the same time, is indifferent to the reward earned during the learning phase. 

A related, but different, goal in RL is the {\em online} one, i.e., to learn to gather high cumulative reward, or to equivalently keep the learner's {\em regret} (the gap between its and the optimal policy's net reward) as low as possible. This is preferable in settings where experimentation comes at a premium and the reward earned in each round is of direct value, e.g., recommender systems (in which rewards correspond to clickthrough events and ultimately translate to revenue), dynamic pricing, automated trading -- in general, control of unknown dynamical systems with instantaneous costs. 


A primary challenge in RL is to learn efficiently across complex (very large or infinite) state and action spaces. In the most general {\em tabula rasa} MDP setting, the learner must explore each state-action transition before developing a reasonably clear understanding of the environment, which is prohibitive for large problems. Real-world domains, though, possess more structure: transition and reward behavior often varies smoothly over states and actions, making it possible to generalize via inductive inference -- observing a state transition or reward is informative of other, similar transitions or rewards. Scaling RL to large, complex, real-world domains  requires exploiting regularity structure in the environment, which has  typically been carried out via the use of parametric MDP models in model-based approaches, e.g., \cite{osband2014model}. 

This paper takes a step in developing theory and algorithms for online RL in environments with smooth transition and reward structure. We specifically consider the episodic online learning problem in the nonparametric, kernelizable MDP setting, i.e., of minimizing regret (relative to an optimal finite-horizon policy) in MDPs with continuous state and action spaces, whose transition and reward functions exhibit smoothness over states and actions compatible with the structure of a reproducing kernel. We develop variants of the well-known UCRL and posterior sampling algorithms for MDPs with continuous state and action spaces, and show that they enjoy sublinear, finite-time regret bounds when the  mean transition and reward functions are assumed to belong to the associated Reproducing Kernel Hilbert Space (RKHS) of functions. 

Our results bound the regret of the algorithms in terms of a novel generalization of the information gain of the state transition and reward function kernels, from the memoryless kernel bandit setting \citep{srinivas2009gaussian} to the state-based kernel MDP setting, and help shed light on how the choice of kernel model influences regret performance. We also leverage recent concentration of measure results for RKHS-valued martingales, developed originally for the kernelized bandit setting \citep{pmlr-v70-chowdhury17a,durand2017streaming}, to prove the results in the paper. To the best of our knowledge, these are the first concrete regret bounds for RL in the kernelizable setting, explicitly showing the dependence of regret on kernel structure.

Our results represent a generalisation of several streams of work. We generalise online learning in the kernelized bandit setting \citep{srinivas2009gaussian,valko2013finite,pmlr-v70-chowdhury17a} to kernelized MDPs, and {\em tabula rasa} online learning approaches for MDPs such as Upper Confidence Bound for Reinforcement Learning (UCRL) \citep{jaksch2010near} and Posterior Sampling for Reinforcement Learning (PSRL) \citep{osband2013more,ouyang2017learning} to MDPs with kernel structure. We also generalize regret minimization for an episodic variant of the well-known parametric Linear Quadratic Regulator (LQR) problem \citep{abbasi2011regret,abbasi2015bayesian,ibrahimi2012efficient,abeille2017thompson} to its nonlinear, nonparametric, infinite-dimensional, kernelizable counterpart. 

\subsection*{Overview of Main Results}

Our first main result gives an algorithm for learning MDPs with mean transition dynamics and reward structure assumed to belong to appropriate Reproducing Kernel Hilbert Spaces (RKHSs). This result is, to our knowledge, the first frequentist regret guarantee for general kernelized MDPs.
\vspace*{-1mm}
\begin{myresult}[Frequentist regret in kernelized MDPs, informal]
\label{result:ucrl}
Consider episodic learning under the unknown dynamics ${s}_{t+1}=\overline{P}_M(s_t,a_t)+\text{Noise} \in \Real^m$, and rewards $r_t=\overline{R}_M(s_t,a_t)+\text{Noise}$, where $\overline{P}_M$ and $\overline{R}_M$ are fixed RKHS functions with bounded norms. The regret of GP-UCRL (Algorithm \ref{algo:GP-UCRL}) is, with high probability\footnote{$\tilde{O}$ suppresses logarithmic factors.}, $\tilde{O}\Big(\big(\gamma_T(R)+\gamma_{mT}(P)\big)\sqrt{T}\Big)$.
\end{myresult}

Here, $\gamma_t(P)$ (resp. $\gamma_t(R)$) roughly represents the maximum information gain about the unknown dynamics (resp. rewards) after $t$ rounds, which, for instance is $\text{polylog}(t)$ for the squared exponential kernel. 

To put this in the perspective of existing work, \citet{osband2014model} also consider learning under dynamics and rewards coming from general function classes, and show (Bayesian) regret bounds depending on the eluder dimensions of the classes. However, when applied to RKHS function classes as we consider here, these dimensions can be infinitely large. In contrast, our results show that the maximum information gain is a suitable measure of complexity of the function class that serves to bound regret. 

An important corollary results when this is applied to the LQR problem, with a linear kernel structure for state transitions and a quadratic kernel structure for rewards:
\vspace*{-1mm}
\begin{myresult}[Frequentist regret for LQR, informal]
\label{res:lqr}
Consider episodic learning under unknown linear dynamics ${s}_{t+1} = As_t + Ba_t + \text{Noise}$, and quadratic rewards $r_t = s_t^TPs_t + a_t^TQa_t + \text{Noise}$. GP-UCRL (Algorithm \ref{algo:GP-UCRL}) instantiated with a linear transition kernel and quadratic reward kernel enjoys, with high probability, regret $\tilde{O}\Big(\big(m^2+n^2 + m(m+n)\big)\sqrt{T}\Big)$, where $m$ and $n$ are the state space and action space dimensions, respectively.
\end{myresult}

This recovers the bound of \citet{osband2014model} for the same bounded LQR problem. However, while they derive this via the eluder dimension approach, we arrive at this by a different bounding technique that applies more generally to any kernelized dynamics. The result also matches (order-wise) the bound of \citet{abbasi2011regret} restricted to the bounded LQR problem.

We also have the following Bayesian regret analogue for PSRL. 
\vspace*{-1mm}
\begin{myresult}[Bayesian regret in kernelized MDPs, informal]
Under dynamics as in Result \ref{result:ucrl} but drawn according to a known prior, the Bayes regret of PSRL (Algorithm \ref{algo:GP-PSRL}) is $\tilde{O}\Big(\big(\gamma_T(R)+\gamma_{mT}(P)\big)\sqrt{T}\Big)$. Consequently, if the dynamics are of the LQR form (Result \ref{res:lqr}), then PSRL instantiated with a linear transition kernel and quadratic reward kernel enjoys Bayes regret $\tilde{O}\Big(\big(m^2+n^2 + m(m+n)\big)\sqrt{T}\Big)$.
\end{myresult}

Note: All the above results are stated assuming that the episode duration $H = O(\ln T)$ for clarity; the explicit dependence on $H$ can be found in the theorem statements that follow.



\vspace*{-2mm}
\paragraph{Related Work}
Regret minimization has been studied with parametric MDPs \citep{jaksch2010near,osband2013more,GopMan15:MDP,agrawal2017optimistic}. For online regret minimization in complex MDPs, apart from the work of \citet{osband2014model}, \citet{ortner2012online} and \citet{lakshmanan2015improved} consider continuous state spaces with Lipschitz transition dynamics but unstructured, finite action spaces. Another important line of work  considers kernel structures for safe exploration in MDPs \citep{turchetta2016safe,berkenkamp2017safe,berkenkamp2016safe}. We, however, seek to demonstrate algorithms with provable regret guarantees in the kernelized MDP setting, which to our knowledge are the first of their kind.


\section{PROBLEM STATEMENT}
\label{sec:prob-stat}

We consider the problem of learning to optimize reward in an unknown finite-horizon MDP, $M_\star=\lbrace \cS,\cA,R_\star,P_\star,H\rbrace$, over repeated episodes of interaction. Here, $\cS \subset \Real^m$ represents the state space, 
$\cA \subset \Real^n$ the action space, $H$ the episode length, $R_\star(s,a)$ the reward distribution over $\Real$, and $P_\star(s,a)$ the transition distribution over $\cS$. At each period $h = 1,2,\ldots,H $ within an episode, an agent observes a state $s_h \in \cS$, takes an action $a_h \in \cA$, observes a reward $r_h \sim R_\star(s_h,a_h)$, and causes the MDP to transition to a next state $s_{h+1}\sim P_\star(s_h,a_h)$. We assume that the agent, while not possessing knowledge of the reward and transition distribution $R_\star,P_\star$ of the unknown MDP $M_\star$, knows $\cS$, $\cA$ and $H$.
\par 
A policy $\pi : \cS \times \lbrace 1,2,\ldots,H \rbrace \ra \cA$ is defined to be a mapping from a state $s\in \cS$ and a period $1 \le h \le H$ to an action $a \in \cA$.
For any MDP $M=\lbrace \cS,\cA,R_M,P_M,H\rbrace$ and policy $\pi$, the finite horizon, undiscounted, value function for every state $s \in \cS$ and every period $1 \le h \le H$ is defined as
$V^M_{\pi,h}(s)\bydef \mathbb{E}_{M,\pi}\Big[\sum_{j=h}^{H}\overline{R}_M(s_j,a_j)\given s_h=s\Big]$,
where the subscript $\pi$ indicates the application of the learning policy $\pi$, i.e., $a_j=\pi(s_j,j)$, and the subscript $M$ explicitly references the MDP environment $M$, i.e.,  $s_{j+1} \sim P_M(s_j,a_j)$, for all $j=h,\ldots,H$. 

We use $\overline{R}_M(s,a)=\expect{r\given r \sim R_M(s,a)}$ to denote the mean of the reward distribution $R_M(s,a)$ that corresponds to playing action $a$ at state $s$ in the MDP $M$. We can view a sample $r$ from the reward distribution $R_M(s,a)$ as $r=\overline{R}_M(s,a)+\epsilon_R$, where $\epsilon_R$ denotes a sample of zero-mean, real-valued additive noise. Similarly, the transition distribution $P_M(s,a)$ can also be decomposed as a mean value $\overline{P}_M(s,a)$ in $\Real^m$ plus a zero-mean additive noise $\epsilon_P$ in $\Real^m$ so that $s'=\overline{P}_M(s,a)+\epsilon_P$ lies in\footnote{\citet{osband2014model} argue that the assumption $\cS \subset \Real^m$ is not restrictive for most practical settings.} $\cS \subset \Real^m$. A policy $\pi_M$ is said to be optimal for the MDP $M$ if $V_{\pi_M,h}^M(s)=\max_{\pi}V_{\pi,h}^M(s)$ for all $s \in \cS$ and $h=1,\ldots,H$.


For an MDP $M$, a distribution $\phi$ over $\cS$ and period $1 \le h \le H$, we define the one step future value function as the expected value of the optimal policy $\pi_M$, with the next state distributed according to $\phi$, i.e. $
U_h^M(\phi) \bydef \mathbb{E}_{s' \sim \phi}\Big[V^M_{\pi_M,h+1}(s')\Big]$. We assume the following regularity condition on the future value function of any MDP in our uncertainty class, also made by \citet{osband2014model}.
\vspace*{-2mm}
\paragraph{Assumption (A1)} 

For any two single-step transition distributions $\phi_1, \phi_2$ over $\cS$, and $1 \le h \le H$,
\beq
\label{eqn: lipschitz}
\abs{U_h^M(\phi_1)-U_h^M(\phi_2)} \le L_M \norm{\overline{\phi}_1-\overline{\phi}_2}_2,
\eeq
where $\overline{\phi}\bydef \mathbb{E}_{s' \sim \phi}[s'] \in \cS$ denotes the mean of the distribution $\phi$. In other words, the one-step future value functions for each period $h$ are Lipschitz continuous with respect to the $\norm{\cdot}_2$-norm of the mean\footnote{Assumption (\ref{eqn: lipschitz}) is essentially equivalent to assuming knowledge of the centered state transition noise distributions, since it implies that any two transition distributions with the same means are identical.}, with global Lipschitz constant $L_M$. We also assume that there is a known constant $L$ such that $ L_\star \bydef L_{M_\star} \le L$.
\vspace*{-2mm}
\paragraph{Regret} 

At the beginning of each episode $l$, an RL algorithm chooses a policy $\pi_l$ depending upon the observed state-action-reward sequences upto episode $l-1$, denoted by the history $\cH_{l-1} \bydef \lbrace s_{j,k},a_{j,k},r_{j,k},s_{j,k+1}\rbrace_{1 \le j \le l-1,1 \le k \le H}$, and executes it for the entire duration of the episode. In other words, at each period $h$ of the $l$-th episode, the learning algorithm chooses action $a_{l,h} = \pi_l(s_{l,h},h)$, receives reward $r_{l,h} = \overline{R}_\star(s_{l,h},a_{l,h}) + \epsilon_{R,l,h}$ and observes the next state $s_{l,h+1}=\overline{P}_\star(s_{l,h},a_{l,h}) + \epsilon_{P,l,h}$. The goal of an episodic online RL algorithm is to maximize its cumulative reward across episodes, or, equivalently, minimize its cumulative {\em regret}: the loss incurred in terms of the value function due to not knowing the optimal policy $\pi_\star \bydef \pi_{M_\star}$ of the unknown MDP $M_\star$ beforehand and instead using the policy $\pi_l$ for each episode $l$, $l = 1, 2, \ldots$. The cumulative (expected) regret of an RL algorithm $\pi=\lbrace \pi_1,\pi_2,\ldots \rbrace$ upto time horizon $T=\tau H$ is defined as
$\text{Regret}(T)=\sum_{l=1}^{\tau} \Big[V^{M_\star}_{\pi_\star,1}(s_{l,1})-V^{M_\star}_{\pi_l,1}(s_{l,1})\Big]$,
where the initial states $s_{l,1},l \ge 1$ are assumed to be fixed.
\vspace*{-2mm}
\paragraph{Notations} For the rest of the paper, unless otherwise specified, we define $z \bydef (s,a)$, $z' \bydef (s',a')$ and $z_{l,h} \bydef (s_{l,h},a_{l,h})$ for all $l \ge 1$ and $1 \le h \le H$. 



\section{ALGORITHMS}
\label{sec:Algo}
\vspace*{-2mm}
\subsection{Representing Uncertainty}
\label{subsec:uncertainty}
The algorithms we design represent uncertainty in the reward and transition distribution $R_\star,P_\star$ by maintaining Gaussian process (GP) priors over the mean reward function $\overline{R}_\star:\cS \times \cA \ra \Real$ and the mean transition function $\overline{P}_\star:\cS \times \cA \times \lbrace 1,\ldots,m \rbrace \ra \Real $ of the unknown MDP $M_\star$. (We denote $\overline{P}_\star(s,a) := [\overline{P}_\star(s,a,1) \, \ldots \, \overline{P}_\star(s,a,m)]^T$.) A Gaussian Process over $\cX$, denoted by $GP_\cX(\mu(\cdot),k(\cdot,\cdot))$, is a collection of random variables $(f(x))_{x\in \cX}$, one for each $x \in \cX$, such that every finite sub-collection of random variables $(f(x_i))_{i = 1}^m$ is jointly Gaussian with mean $\expect{f(x_i)} = \mu(x_i)$ and covariance $\expect{(f(x_i)-\mu(x_i))(f(x_j)-\mu(x_j))} = k(x_i, x_j)$, $1 \le i, j \le m$, $m \in \mathbb{N}$. We use $GP_\cZ(0,k_R)$ and $GP_{\tilde{\cZ}}(0,k_P)$ as the initial prior distributions over $\overline{R}_\star$ and $\overline{P}_\star$, with positive semi-definite covariance (kernel) functions $k_R$ and $k_P$ respectively, where $\cZ\bydef \cS \times \cA$, $\tilde{\cZ} \bydef \cZ\times \lbrace 1,\ldots,m \rbrace$. We also assume that the noise variables $\epsilon_{R,l,h}$ and $\epsilon_{P,l,h}$ are drawn independently, across $l$ and $h$, from $\cN(0,\lambda_R)$ and $\cN(0,\lambda_PI)$ respectively, with $\lambda_R,\lambda_P \ge 0$.
Then, by standard properties of GPs  \citep{rasmussen2006gaussian}, conditioned on the history $\cH_{l}$, the posterior distribution over $\overline{R}_\star$ is also a Gaussian process, $GP_{\cZ}(\mu_{R,l},k_{R,l})$, with mean and kernel functions
\beqa
\label{eqn:post-reward}
\begin{aligned}
\mu_{R,l}(z) &= k_{R,l}(z)^T(K_{R,l} + \lambda_R I)^{-1}R_{l},\\
k_{R,l}(z,z') &= k_R(z,z') - k_{R,l}(z)^T(K_{R,l} + \lambda_R I)^{-1} k_{R,l}(z'),\\
\sigma_{R,l}^2(z) &= k_{R,l}(z,z).
\end{aligned}
\eeqa
Here $R_l \bydef [r_{1,1},\ldots,r_{l,H}]^T$ denotes the vector of  rewards observed at $\cZ_{l} \bydef \lbrace z_{j,k} \rbrace_{1 \le j \le l,1 \le k \le H} =\lbrace z_{1,1},\ldots,z_{l,H}\rbrace$, the set of all state-action pairs available at the end of episode $l$. $k_{R,l}(z) \bydef [k_R(z_{1,1},z),\ldots,k_R(z_{l,H},z)]^T$ denotes the vector of kernel evaluations between $z$ and elements of the set $\cZ_{l}$ and $K_{R,l} \bydef [k_R(u,v)]_{u,v \in \cZ_{l}}$ denotes the kernel matrix computed at $\cZ_{l}$. 

Similarly, conditioned on $\cH_{l}$, the posterior distribution over $\overline{P}_\star$ is $GP_{\tilde{\cZ}}(\mu_{P,l},k_{P,l})$ with mean and kernel functions 
\beqa
\label{eqn:post-state}
\begin{aligned}
\mu_{P,l}(z,i) &= k_{P,l}(z,i)^T(K_{P,l} + \lambda_P I)^{-1}S_l,\\
k_{P,l}\big((z,i),(z',j)\big) &= k_P((z,i),(z,i)) - k_{P,l}(z,i)^T(K_{P,l} + \lambda_P I)^{-1} k_{P,l}(z',j),\\
\sigma_{P,l}^2(z,i) &= k_{P,l}\big((z,i),(z,i)\big).
\end{aligned}
\eeqa
Here $S_{l} \bydef [s_{1,2}^T,\ldots,s_{l,H+1}^T]^T$ denotes the vector of state transitions at $\cZ_{l}=\lbrace z_{1,1},\ldots,z_{l,H}\rbrace$, the set of all state-action pairs available at the end of episode $l$. $k_{P,l}(z,i) \bydef \Big[k_P\big((z_{1,1},1),(z,i)\big),\ldots,k_P\big((z_{l,H},m),(z,i)\big)\Big]^T$ denotes the vector of kernel evaluations between $(z,i)$ and elements of the set $\tilde{\cZ}_{l}\bydef \big\lbrace (z_{j,k},i) \big\rbrace_{1 \le j \le l,1 \le k \le H, 1 \le i \le m}\\ =\big\lbrace  (z_{1,1},1),\ldots,(z_{l,H},m)\big\rbrace $ and $K_{P,l} \bydef \big[k_P(u,v)\big]_{u,v \in \tilde{\cZ}_{l}}$ denotes the kernel matrix computed at $\tilde{\cZ}_{l}$.

Thus, at the end of episode $l$, conditioned on the history $\cH_l$, the posterior distributions over $\overline{R}_\star(z)$ and $\overline{P}_\star(z,i)$ is updated and maintained as $\cN\big(\mu_{R,l}(z),\sigma^2_{R,l}(z)\big)$ and $\cN\big(\mu_{P,l}(z,i),\sigma^2_{P,l}(z,i)\big)$ respectively, for every $z \in \cZ$ and $i=1,\ldots,m$. This representation not only permits generalization via inductive inference across state and action spaces, but also allows for tractable updates. We now present our online algorithms GP-UCRL and PSRL for kernelized MDPs. 


\subsection{GP-UCRL Algorithm}
\label{subsec:GP-UCRL}

GP-UCRL (Algorithm \ref{algo:GP-UCRL}) is an optimistic algorithm based on the Upper Confidence Bound principle, which adapts the confidence sets of UCRL2 \citep{jaksch2010near} to exploit the kernel structure. At the start of every episode $l$, GP-UCRL constructs confidence sets $\cC_{R,l}$ and $\cC_{P,l}$ for the mean reward function and transition function, respectively, using the parameters of GP posteriors as given in Section \ref{subsec:uncertainty}. 
The exact forms of the confidence sets appear in the theoretical result later, e.g., (\ref{eqn:confidence-set-reward-rkhs}) and (\ref{eqn:confidence-set-state-rkhs}). It then builds the set $\cM_l$ of all plausible MDPs $M$ with the mean reward function $\overline{R}_M \in \cC_{R,l}$, the mean transition function $\overline{P}_M \in \cC_{P,l}$ and the global Lipschitz constant  (\ref{eqn: lipschitz}) $L_M$ of future value functions upper bounded by a known constant $L$, where $L_\star \le L$. It then selects an optimistic policy $\pi_l$ for the family of MDPs $\cM_l$ in the sense that $V^{M_l}_{\pi_l,1}(s_{l,1})=\max_{\pi}\max_{M \in \cM_l}V^{M}_{\pi,1}(s_{l,1})$, where $s_{l,1}$ is the initial state and $M_l$ is the most optimistic realization from $\cM_l$,
and executes $\pi_l$ for the entire episode. The pseudo-code of GP-UCRL is given in Algorithm \ref{algo:GP-UCRL}. (Note: Even though GP-UCRL is described using the language of GP priors/posteriors, it can also be understood as kernelized regression with appropriately designed confidence sets.)

\begin{algorithm}
\renewcommand\thealgorithm{1}
\caption{GP-UCRL}\label{algo:GP-UCRL}
\begin{algorithmic}
\STATE \textbf{Input:} Kernel functions $k_R$ and $k_P$.
\STATE Set $\mu_{R,0}(z)=\mu_{P,0}(z,i)=0$, $\sigma^2_{R,0}(z)=k_R(z,z)$, $\sigma^2_{P,0}(z,i)=k_{P}\big((z,i),(z,i)\big) \forall z \in \cZ, \forall i=1,\ldots,m$.
\FOR{episode $l = 1, 2, 3, \ldots$}
\STATE Construct confidence sets $\cC_{R,l}$ and $\cC_{P,l}$.
\STATE Construct the set of all plausible MDPs $\cM_l =\lbrace M : L_M \le L, \overline{R}_M \in \cC_{R,l}, \overline{P}_M \in \cC_{P,l} \rbrace.$
\STATE Choose policy  
$\pi_l$ such that $V^{M_l}_{\pi_l,1}(s_{l,1})=\max_{\pi}\max_{M \in \cM_l}V^{M}_{\pi,1}(s_{l,1})$.
\STATE \FOR{period $h=1,2,3, \ldots, H$} 
\STATE Choose action $a_{l,h} = \pi_l(s_{l,h},h)$.
\STATE Observe reward $r_{l,h} = \overline{R}_\star(z_{l,h}) + \epsilon_{R,l,h}$.
\STATE Observe next state $s_{l,h+1}=\overline{P}_\star(z_{l,h}) + \epsilon_{P,l,h}$.
\ENDFOR
\STATE Update $\mu_{R,l}$, $\sigma_{R,l}$ using (\ref{eqn:post-reward}) and $\mu_{P,l}$, $\sigma_{P,l}$ using (\ref{eqn:post-state}).
\ENDFOR
\end{algorithmic}
\addtocounter{algorithm}{-1}
\end{algorithm}

Optimizing for an optimistic policy is not computationally tractable in general, even though planning for the optimal policy is possible for a given MDP. A popular approach to overcome this difficulty is to sample a random MDP at every episode and solve for its optimal policy, called posterior sampling \citep{osband2016posterior}.

\subsection{PSRL Algorithm}

PSRL (Algorithm \ref{algo:GP-PSRL}), in its most general form, starts with a prior distribution $\Phi\equiv(\Phi_R,\Phi_P)$ over MDPs, where $\Phi_R$ and $\Phi_P$ are priors over reward and transition distributions respectively. At the beginning of episode $l$, using the history of observations $\cH_{l-1}$, it updates the posterior $\Phi_l\equiv (\Phi_{R,l}, \Phi_{P,l})$ and samples an MDP $M_l$ from it\footnote{Sampling can be done using MCMC methods even if $\Phi_{l }$ doesn't admit any closed form.}  ($\Phi_{R,l}$ and $\Phi_{P,l}$ are posteriors over reward and transition distributions respectively). It then selects an optimal policy $\pi_l$ of the sampled MDP $M_l$, in the sense that $V^{M_l}_{\pi_l,h}(s)=\max_{\pi}V^{M_l}_{\pi,h}(s)$ for all $s \in \cS$ and for all $h=1,2,\ldots,H$,
and executes $\pi_l$ for the entire episode.

\begin{algorithm}
\renewcommand\thealgorithm{2}
\caption{PSRL}\label{algo:GP-PSRL}
\begin{algorithmic}
\STATE \textbf{Input:} Prior $\Phi$. 
\STATE Set $\Phi_1=\Phi$.
\FOR{ episode $l = 1, 2, 3, \ldots$}
\STATE Sample $M_l \sim \Phi_l$.
\STATE Choose policy 
$\pi_l$ such that $V^{M_l}_{\pi_l,h}(s)=\max_{\pi}V^{M_l}_{\pi,h}(s)\; \forall s \in \cS,\forall h=1,2,\ldots,H$.
\STATE \FOR{period $h=1,2,3, \ldots, H$} 
\STATE Choose action $a_{l,h} = \pi_l(s_{l,h},h)$.
\STATE Observe reward $r_{l,h} = \overline{R}_\star(z_{l,h}) + \epsilon_{R,l,h}$.
\STATE Observe next state $s_{l,h+1}=\overline{P}_\star(z_{l,h}) + \epsilon_{P,l,h}$.
\ENDFOR
\STATE Update $\Phi_l$ to $\Phi_{l+1}$, using $\left \lbrace s_{l,h}, a_{l,h}, s_{l,h+1} \right\rbrace_{1 \le h \le H}$.
\ENDFOR
\end{algorithmic}
\addtocounter{algorithm}{-2}
\end{algorithm}


For example, if $\Phi_R$ and $\Phi_P$ are specified by GPs $GP_\cZ(0,k_R)$ and $GP_{\tilde{\cZ}}(0,k_P)$ respectively with Gaussian observation model, then the posteriors $\Phi_{R,l}$ and $\Phi_{P,l}$ are given by GP posteriors as discussed in Section \ref{subsec:uncertainty}.
Here at every episode $l$, PSRL samples an MDP $M_l$ with mean reward function $\overline{R}_{M_l} \sim GP_\cZ(\mu_{R,l-1},k_{R,l-1})$ and mean transition function $\overline{P}_{M_l} \sim GP_{\tilde{\cZ}}(\mu_{P,l-1},k_{P,l-1})$. 


\vspace*{-2mm}

\paragraph{Computational issues} 
Optimal planning may be computationally intractable even for a given MDP, so it is common in the literature to assume access to an approximate MDP planner $\Gamma (M,\eps)$ which returns an $\epsilon$-optimal policy for $M$. Given such a planner $\Gamma$, if it is possible to obtain (through extended value iteration \citep{jaksch2010near} or otherwise) an efficient planner $\tilde{\Gamma} (\cM,\eps)$ which returns an $\epsilon$-optimal policy for the most optimistic MDP from a  family $\cM$, then we modify PSRL and GP-UCRL to choose $\pi_l=\Gamma (M_l,\sqrt{H/l})$ and $\pi_l=\tilde{\Gamma} (\cM_l,\sqrt{H/l})$ respectively at every episode $l$. It follows that this adds only an $O(\sqrt{T})$ factor in the respective regret bounds. The design of such approximate planners for continuous state and action spaces remains a subject of active research, whereas our focus in this work is on the statistical efficiency of the online learning problem.

\section{MAIN RESULTS}
\label{sec:Results}

In this section, we provide our main theoretical upper bounds on the cumulative regret. All the proofs are deferred to the appendix for lack of space.

\subsection{Preliminaries and assumptions}

\vspace*{-2mm}
\paragraph{Maximum Information Gain (MIG)} Let $f:\cX \ra \Real$ be a (possibly random) real-valued function defined on a domain $\cX$. For each $A \subset \cX$, let $f_A := [f(x)]_{x\in A}$ denote a vector containing $f$'s evaluations at each point in $A$ and $Y_A$ denote a noisy version of $f_A$ obtained by passing $f_A$ through a channel $\prob{Y_A | f_A}$. The \textit{Maximum Information Gain (MIG)} about $f$ after $t$ noisy observations is defined as $\gamma_t(f, \cX) \bydef \max_{A \subset \cX : \abs{A}=t} I(f_A;Y_A)$, where $I(X;Y)$ denotes the Shannon mutual information between two jointly distributed random variables $X, Y$. If $f \sim GP_{\cX}(0,k)$ and the channel is iid Gaussian $\cN(0,\lambda)$, then $\gamma_t(f, \cX)$ depends only on $k,\cX,\lambda$ \cite{srinivas2009gaussian}. But the dependency on $\lambda$ is only of $\tilde{O}(1/\lambda)$ and hence in this setting we denote MIG as $\gamma_t(k,\cX)$ to indicate the dependencies on $k$ and $\cX$ explicitly. If $\cX \subset \Real^d$ is compact and convex, then $\gamma_t(k,\cX)$ is sublinear in $t$ for different classes of kernels; e.g. for the linear kernel\footnote{$k(x,x') = x^Tx'$.} $\gamma_t(k,\cX)=\tilde{O}(d\ln t)$ and for the Squared Exponential (SE) kernel\footnote{$
k(x,x') = \exp\left(-\norm{x-x'}_2^2/2l^2\right),l > 0
$.}, $\gamma_t(k,\cX)=\tilde{O}\left((\ln t)^{d}\right)$ \cite{srinivas2009gaussian}.
\vspace*{-2mm}
\paragraph{Composite kernels} Let $\cX=\cX_1 \times \cX_2$. A composite kernel $k:\cX \times \cX \ra \Real$ can be constructed by using individual kernels $k_1:\cX_1 \times \cX_1 \ra \Real$ and $k_2:\cX_2 \times \cX_2 \ra \Real$. For instance, a \textit{product kernel} $k=k_1 \otimes k_2$ is obtained by setting $(k_1 \otimes k_2)\big((x_1,x_2),(x_1',x_2')\big)\bydef k_1(x_1,x_1')k_2(x_2,x_2')$. Another example is that of an \textit{additive kernel} $k=k_1 \oplus k_2$ by setting $(k_1 \oplus k_2)\big((x_1,x_2),(x_1',x_2')\big)=k_1(x_1,x_1')+k_2(x_2,x_2')$. \citet{krause2011contextual} bound the MIG for additive and product kernels in terms of the MIG for individual kernels as
\beq
\label{eqn:additive-kernel}
\gamma_t(k_1 \oplus k_2,\cX) \le \gamma_t(k_1,\cX_1) + \gamma_t(k_2,\cX_2)+2 \ln t,
\eeq
and, if $k_2$ has rank at most $d$ (i.e. all kernel matrices over any finite subset of $\cX_2$ have rank at most $d$), as
\beq
\label{eqn:product-kernel}
\gamma_t(k_1 \otimes k_2,\cX) \le d\gamma_t(k_1,\cX_1)  + d\ln t.
\eeq
Therefore, if the MIGs for individual kernels are sublinear in $t$, then the same is true for their products and additions.  For example, the MIG for the product of a $d_1$-dimensional linear kernel and a $d_2$-dimensional SE kernel is $\tilde{O}\big(d_1(\ln t)^{d_2}\big)$.

%
\vspace*{-2mm}
\paragraph{Regularity assumptions (A2)} 
Each of our results in this section will assume that $\overline{R}_\star$ and $\overline{P}_\star$ have small norms in the Reproducing Kernel Hilbert Spaces (RKHSs) associated with kernels $k_R$ and $k_P$ respectively. An RKHS of real-valued functions $\cX \to \Real$, denoted by $\cH_k(\cX)$, is completely specified by its kernel function $k(\cdot,\cdot)$ and vice-versa, with an inner product $\inner{\cdot}{\cdot}_k$ obeying the reproducing property $f(x)=\inner{f}{k(x,\cdot)}_k$ for all $f \in \cH_k(\cX)$. The induced RKHS norm $\norm{f}_k = \sqrt{\inner{f}{f}}_k$ is a measure of  smoothness of $f$ with respect to the kernel function $k$. We assume known bounds on the RKHS norms of the mean reward and mean transition functions: $\overline{R}_\star \in \cH_{k_R}(\cZ),\norm{\overline{R}_\star}_{k_R} \le B_R$ and $\overline{P}_\star \in \cH_{k_P}(\tilde{\cZ}),\norm{\overline{P}_\star}_{k_P} \le B_P$, where $\cZ \bydef \cS \times \cA$ and $\tilde{\cZ}\bydef \cZ \times \lbrace 1,\ldots,m \rbrace$.
\vspace*{-2mm}
\paragraph{Noise assumptions (A3)} For the purpose of this section, we assume that the noise sequence $\left\lbrace\epsilon_{R,l,h}\right\rbrace_{l\ge 1, 1 \le h \le H}$ is conditionally $\sigma_R$-sub-Gaussian, i.e., there exists a known $\sigma_R \ge 0$ such that for any $\eta \in \Real$, 
\beq
\label{eqn:noise-reward}
 \expect{\exp(\eta\;\epsilon_{R,l,h})\given \cF_{R,l,h-1}} \le \exp\big(\eta^2\sigma_R^2/2\big), 
\eeq
where $\cF_{R,l,h-1}$ is the sigma algebra generated by the random variables $\lbrace s_{j,k},a_{j,k},\epsilon_{R,j,k}\rbrace_{1 \le j\le l-1,1 \le k \le H}$, $\lbrace s_{l,k},a_{l,k},\epsilon_{R,l,k}\rbrace_{1 \le k \le h-1}$, $s_{l,h}$ and $a_{l,h}$. Similarly, the noise sequence $\left\lbrace\epsilon_{P,l,h}\right\rbrace_{l\ge 1,1 \le h \le H}$ is assumed to be conditionally component-wise independent and $\sigma_P$-sub-Gaussian, in the sense that there exists a known $\sigma_P \ge 0$ such that for any $\eta \in \Real$ and $1 \le i \le m$,
\beq
\label{eqn:noise-state}
\begin{aligned}
\expect{\exp\big(\eta \epsilon_{P,l,h}(i)\big)\given \cF_{P,l,h-1}} &\le \exp\big(\eta^2\sigma_P^2/2\big),\\
\expect{\epsilon_{P,l,h}\epsilon_{P,l,h}^T\given \cF_{P,l,h-1}} &=I,
\end{aligned}
\eeq
where $\cF_{P,l,h-1}$ is the sigma algebra generated by the random variables $\lbrace s_{j,k},a_{j,k},\epsilon_{P,j,k}\rbrace_{1 \le j\le l-1,1 \le k \le H}$, $\lbrace s_{l,k},a_{l,k},\epsilon_{P,l,k}\rbrace_{1 \le k \le h-1}$, $s_{l,h}$ and $a_{l,h}$.

\subsection{Regret Bound for GP-UCRL in Kernelized MDPs}
\label{subsec:results-mis-specified}

We run GP-UCRL (Algorithm \ref{algo:GP-UCRL}) using GP priors and Gaussian noise models as given in Section \ref{subsec:GP-UCRL}. Note that, in this section, though the algorithm relies on GP priors, the setting under which it is analyzed is `agnostic', i.e., under a {\em fixed} but unknown true MDP environment.

\vspace*{-2mm}
\paragraph{Choice of confidence sets}
At the beginning of each episode $l$, GP-UCRL constructs the confidence set $\cC_{R,l}$ as
\beq
\cC_{R,l}=\big\lbrace f: \abs{f(z)-\mu_{R,l-1}(z)} \le \beta_{R,l}\sigma_{R,l-1}(z) \forall z \in \cZ \big\rbrace,\label{eqn:confidence-set-reward-rkhs}
\eeq
where $\mu_{R,l}(z)$ and $\sigma_{R,l}(z)$ are as defined in (\ref{eqn:post-reward}) with $\lambda_R=H$.
$\beta_{R,l}\bydef B_R + \dfrac{\sigma_R}{\sqrt{H}}\sqrt{2\big(\ln(3/\delta)+\gamma_{(l-1)H}(R)\big)}$, where $\gamma_t(R)\equiv\gamma_{t}(k_R,\cZ)$ denotes the maximum information gain (or an upper bound on the maximum information gain) about any $f \sim GP_{\cZ}(0,k_R)$ after $t$ noisy observations obtained by passing $f$ through an iid Gaussian channel $\cN(0,H)$.

Similarly, GP-UCRL constructs the confidence set $\cC_{P,l}$ as
\beq
\cC_{P,l}=\big\lbrace f: \norm{f(z)-\mu_{P,l-1}(z)}_2 \le \beta_{P,l}\norm{\sigma_{P,l-1}(z)}_2 \forall z \in \cZ\big\rbrace.\label{eqn:confidence-set-state-rkhs}
\eeq
Here, $\mu_{P,l}(z)\bydef [\mu_{P,l}(z,1),\ldots,\mu_{P,l}(z,m)]^T$ and $\sigma_{P,l}(z)\bydef [\sigma_{P,l}(z,1),\ldots,\sigma_{P,l}(z,m)]^T$, where $\mu_{P,l}(z,i)$ and $\sigma_{P,l}(z,i)$ are as defined in (\ref{eqn:post-state}) with $\lambda_P=mH$. $\beta_{P,l} \bydef  B_P + \dfrac{\sigma_P}{\sqrt{mH}}\sqrt{2\big(\ln(3/\delta)+\gamma_{m(l-1)H}(P)\big)}$, where $\gamma_t(P)\equiv\gamma_{t}(k_P,\tilde{\cZ})$ denotes the maximum information gain about any $f \sim GP_{\tilde{\cZ}}(0,k_P)$ after $t$ noisy observations obtained by passing $f$ through an iid Gaussian channel $\cN(0,mH)$.

\begin{mytheorem}[Frequentist regret bound for GP-UCRL]
\label{thm:regret-bound-RKHS}


Let assumptions (A1) - (A3) hold, 
$k_R(z,z) \le 1$ and $k_P((z,i),(z,i)) \le 1$ for all $z \in \cZ$ and $1 \le i \le m$\footnote{This is called the bounded variance property of kernels and it holds for most of the common kernels (e.g. Squared Exponential).}. Then for any $0 \le \delta \le 1$, GP-UCRL, with confidence sets (\ref{eqn:confidence-set-reward-rkhs}) and (\ref{eqn:confidence-set-state-rkhs}), enjoys, with probability at least $1-\delta$, the regret bound 
\beqn
Regret(T) \le 2\beta_{R,\tau}\sqrt{2 e H\gamma_{T}(R)T}+ 2L\beta_{P,\tau}\sqrt{2emH \gamma_{mT}(P)T}+(LD+2B_RH) \sqrt{2T\ln(3/\delta)},
\eeqn
where $T \bydef \tau H$ is the total time in $\tau$ episodes, $\beta_{R,\tau}=B_R + \dfrac{\sigma_R}{\sqrt{H}}\sqrt{2\big(\ln(3/\delta)+\gamma_{(\tau-1)H}(R)\big)}$ and $\beta_{P,\tau} = B_P + \dfrac{\sigma_P}{\sqrt{mH}}\sqrt{2\big(\ln(3/\delta)+\gamma_{m(\tau-1)H}(P)\big)}$, $L$ is a known upper bound over the global Lipschitz constant (\ref{eqn: lipschitz}) $L_\star$ for the future value function of $M_\star$ and $D \bydef \max_{s,s'\in \cS}\norm{s-s'}_2$ denotes the diameter of $\cS$.
\end{mytheorem}
\vspace*{-2mm}
\paragraph{Interpretation of the bound} As MIG increases with the number of observations, $\beta_{R,l}$ and $\beta_{P,l}$ increase with $l$. Hence $\beta_{R,\tau}=\tilde{O}\Big(B_R + \dfrac{\sigma_R}{\sqrt{H}}\sqrt{\gamma_{T}(R)}\Big)$ and $\beta_{P,\tau}=\tilde{O}\Big( B_P + \dfrac{\sigma_P}{\sqrt{mH}}\sqrt{\gamma_{mT}(P)}\Big)$. Thus, Theorem \ref{thm:regret-bound-RKHS} implies that the cumulative regret of GP-UCRL after $T$ timesteps is
$\tilde{O}\Big(\big(\sqrt{H\gamma_T(R)} +\gamma_T(R)\big)\sqrt{T} + L\big(\sqrt{mH\gamma_{mT}(P)} + \gamma_{mT}(P)\big)\sqrt{T}+ H\sqrt{T}\Big)$
with high probability. Hence, we see that the cumulative regret of GP-UCRL scales linearly with $\gamma_T(R)$ and $\gamma_{mT}(P)$. As $\gamma_T(R)$ and $\gamma_{mT}(P)$ grow  sublinearly with $T$ for most  popular kernels (eg. Squared Exponential (SE), polynomial), the cumulative regret of GP-UCRL can grow sublinearly with $T$. We illustrate this with the following concrete examples:
\vspace*{-2mm}
\paragraph{(a) Example bound on $\gamma_T(R)$:} Recall that $\gamma_T(R)\equiv\gamma_T(k_R,\cZ)$, where the kernel $k_R$ is defined on the product space $\cZ =\cS \times \cA$. If $k_1$ and $k_2$ are kernels on the state space $\cS \subset\Real^m$ and the action space $\cA \subset \Real^n$, respectively, and $k_R$ is an additive kernel of $k_1$ and $k_2$, then  (\ref{eqn:additive-kernel}) implies that $\gamma_T(k_R,\cZ) \le \gamma_T(k_1,\cS)+\gamma_T(k_2,\cA)+2 \ln T$. Further, if both $\cS$, $\cA$ are  compact and convex, and both $k_1,k_2$ are Squared Exponential (SE) kernels, then $\gamma_T(k_1,\cS)=\tilde{O}\big((\ln T)^m\big)$ and $\gamma_T(k_2,\cA)=\tilde{O}\big((\ln T)^n\big)$. Hence in this case
$\gamma_T(R)=\tilde{O}\big((\ln T)^{\max\lbrace m,n \rbrace}\big)$.
\vspace*{-2mm}
\paragraph{(b) Example bound on $\gamma_{mT}(P)$:} Recall that $\gamma_{mT}(P)\equiv \gamma_{mT}(k_P,\tilde{\cZ})$, where the kernel $k_P$ is defined on the product space $\tilde{\cZ} =\cZ \times \lbrace 1,\ldots,m \rbrace$. If $k_3$ and $k_4$ are kernels on the product space $\cZ$ and the index set $\lbrace 1,\ldots,m \rbrace$, respectively, and $k_P$ is a product kernel of $k_3$ and $k_4$, then (\ref{eqn:product-kernel}) implies that $\gamma_{mT}(k_P,\tilde{\cZ}) \le m\gamma_{mT}(k_3,\cZ)+m\ln (mT)$, since all kernel matrices over any subset of $\lbrace 1,\ldots,m \rbrace$ have rank at most $m$. Further, if $k_5$ is a SE kernel on the state space $\cS$, $k_6$ is a linear kernel on the action space $\cA$ and $k_3$ is a product kernel, then (\ref{eqn:product-kernel}) implies that $\gamma_{mT}(k_3,\cZ) = \tilde{O}\big(n\big(\ln (mT)\big)^m\big)$, as the rank of an $n$-dimensional linear kernel is at most $n$. Hence, in this case, $\gamma_{mT}(P)=\tilde{O}\big(mn\big(\ln (mT)\big)^m\big)$.
\vspace*{-2mm}
\paragraph{Proof Sketch for Theorem \ref{thm:regret-bound-RKHS}} 

First, see that when $\overline{R}_\star \in \cC_{R,l}$ and $\overline{P}_\star \in \cC_{P,l}$, then the following are true:
 
(a) $M_\star$ lies in $\cM_l$, the family of MDPs constructed by GP-UCRL. Hence $V^{M_l}_{\pi_l,1}(s_{l,1}) \ge V^{M_\star}_{\pi_\star,1}(s_{l,1})$, where $M_l$ is the most optimistic realization from $\cM_l$, and thus $ \text{Regret}(T) \le \sum_{l=1}^{\tau}\big(V^{M_l}_{\pi_l,1}(s_{l,1})-V^{M_\star}_{\pi_l,1}(s_{l,1})\big)$.

(b) Optimistic rewards/transitions do not deviate too much: $\abs{\overline{R}_{M_l}(z_{l,h}) - \overline{R}_\star(z_{l,h})}  
\le 2\beta_{R,l}\;\sigma_{R,l-1}(z_{l,h})$ and $\norm{\overline{P}_{M_l}(z_{l,h}) - \overline{P}_\star(z_{l,h})}_2 \le  2 \beta_{P,l}\norm{\sigma_{P,l-1}(z_{l,h})}_2$,
since by construction  $\overline{R}_{M_l} \in \cC_{R,l}$ and $\overline{P}_{M_l} \in \cC_{P,l}$. 

Further it can be shown that:

(c) Cumulative predictive variances are bounded: $\sum_{l=1}^{\tau}\sum_{h=1}^{H} \sigma_{R,l-1}(z_{l,h}) \le  \sqrt{2 eH\gamma_{T}(R)T}$
and 
\\
$\sum_{l=1}^{\tau}\sum_{h=1}^{H} \norm{\sigma_{P,l-1}(z_{l,h})}_2 \le\sqrt{2emH\gamma_{mT}(P)T}$. 

(d) Bounds on deviations of rewards and transitions imply bounds on deviation of the value function: $\sum_{l=1}^{\tau}\big(V^{M_l}_{\pi_l,1}(s_{l,1})-V^{M_\star}_{\pi_l,1}(s_{l,1})\big) \le  \sum_{l=1}^{\tau}\sum_{h=1}^{H} \abs{\overline{R}_{M_l}(z_{l,h}) - \overline{R}_\star(z_{l,h})}+L\sum_{l=1}^{\tau}\sum_{h=1}^{H} \norm{\overline{P}_{M_l}(z_{l,h}) - \overline{P}_\star(z_{l,h})}_2+ (LD+2B_RH) \sqrt{2\tau H \ln(3/\delta)}$, with probability at least $1 - \delta/3$.

The proof now follows by combining (a), (b), (c), and (d) and showing the confidence set properties
$\prob{\overline{R}_\star \in \cC_{R,l}} \ge 1-\delta/3$ and $\prob{\overline{P}_\star \in \cC_{P,l}} \ge 1-\delta/3$.

\subsection{Regret Bound for PSRL}

\citet{osband2016posterior} show that if we have a frequentist regret bound for UCRL in hand, then we can obtain a similar bound (upto a constant factor) on the \emph{Bayes regret} (defined as the \textit{expected regret under the prior distribution $\Phi$}) of PSRL. We use this idea to obtain a sublinear bound on the Bayes regret of PSRL in the setting of kernelized MDPs.
\begin{mytheorem}[Bayes regret of PSRL under RKHS prior]
\label{thm:regret-bound-PSRL}
Let assumptions (A1) - (A3) hold, 
$k_R(z,z) \le 1$ and $k_P((z,i),(z,i)) \le 1$ for all $z \in \cZ$ and $1 \le i \le m$.
Let $\Phi$ be a (known) prior distribution over MDPs $M_\star$. Then, the Bayes regret of PSRL satisfies
\beqn
\expect{Regret(T)} \le 2 \alpha_{R,\tau} \sqrt{2 e  H\gamma_{T}(R)T} + 3\expect{L_\star}\alpha_{P,\tau}\sqrt{2em H\gamma_{mT}(P)T}+3B_R,
\eeqn
where $T=\tau H$ is the total time in $\tau$ episodes, $L_\star$ is the global Lipschitz constant for the future value function (\ref{eqn: lipschitz}) of $M_\star$, $\alpha_{R,\tau}=B_R + \dfrac{\sigma_R}{\sqrt{H}}\sqrt{2\big(\ln(3T)+\gamma_{(\tau-1)H}(R)\big)}$ and $\alpha_{P,\tau} = B_P + \dfrac{\sigma_P}{\sqrt{mH}}\sqrt{2\big(\ln(3T)+\gamma_{m(\tau-1)H}(P)\big)}$.
\end{mytheorem}
Theorem \ref{thm:regret-bound-PSRL} implies that the Bayes regret of PSRL after $T$ timesteps is
$\tilde{O}\Big(\big(\sqrt{H\gamma_T(R)} +\gamma_T(R)\big)\sqrt{T} + \expect{L_\star}\big(\sqrt{mH\gamma_{mT}(P)} + \gamma_{mT}(P)\big)\sqrt{T}+ H\sqrt{T}\Big)$, and thus has the same scaling as the bound for GP-UCRL.
\\
\textbf{\textit{Remark.}}
Observe that when $H \le \gamma_T(R)$ and $mH \le  \gamma_{mT}(P)$\footnote{Both conditions hold, for instance, if $H = O(\ln T)$ with a polynomial or SE kernel.}, then the regret of GP-UCRL is $\tilde{O}\Big(\big(\gamma_T(R)+L\gamma_{mT}(P)\big)\sqrt{T}\Big)$ with high probability and the Bayes regret of PSRL is $\tilde{O}\Big(\big(\gamma_T(R)+\expect{L_\star}\gamma_{mT}(P)\big)\sqrt{T}\Big)$, where both $L$ (an upper bound on $L_\star$) and $\expect{L_\star}$ basically measure  the connectedness of the MDP $M_\star$.

\vspace*{-2mm}
\paragraph{Comparison with the eluder dimension results.} \citet{osband2014model}
assume that $\overline{R}_\star$ and $\overline{P}_\star$ are elements from two function classes $\cR$ and $\cP$, respectively, with bounded $\norm{\cdot}_2$-norm, and show that PSRL obtains  Bayes regret $\tilde{O}\Big(\big(\sqrt{d_K(\cR)d_E(\cR)}+\expect{L_\star} \sqrt{d_K(\cP)d_E(\cP)}\big)\sqrt{T}\Big)$, where $d_K(\cF)$ (Kolmogorov dimension) and $d_E(\cF)$ (eluder dimension)  measure the ``complexity" of a function class $\cF$. As a special case, if both $\overline{R}_\star$ and $\overline{P}_\star$ are linear functions in finite dimension $d$, then they show that $d_E(\cR), d_K(\cR)=\tilde{O}(d)$ and $d_E(\cP), d_K(\cP)=\tilde{O}(md)$. In our setting, $\overline{R}_\star$ and $\overline{P}_\star$ are RKHS functions, and hence, by the reproducing property, they are linear functionals in (possibly) infinite dimension. From this viewpoint, all of $d_E(\cR)$, $d_K(\cR)$, $d_E(\cP)$ and $d_K(\cP)$ can blow upto infinity yielding trivial bounds. Therefore, we need a suitable measure of complexity of the RKHS spaces, and a single information-theoretic quantity, namely the Maximum Information Gain (MIG), is seen to serve this purpose.

To the best of our knowledge, Theorem \ref{thm:regret-bound-RKHS} is the first frequentist regret bound and Theorem \ref{thm:regret-bound-PSRL} is the first Bayesian regret bound in the kernelized MDP setting (i.e., when the MDP model is from an RKHS class). We see that both algorithms achieve similar regret bounds in terms of dependencies on time, MDP connectedness and Maximum Information Gain. However, Theorem \ref{thm:regret-bound-RKHS} is a stronger probabilistic guarantee than Theorem \ref{thm:regret-bound-PSRL} since it holds with high probability for any MDP $M_\star$ and not just in expectation over the draw from the prior distribution.

As special cases of our results, we now derive regret bounds for two representative RL domains, namely tabular MDPs and linear quadratic control systems.

\vspace*{-2mm}
\paragraph{1. Tabula-rasa MDPs} In this case, both $\cS$ and $\cA$ are finite and expressed as $\cS=\lbrace 1,\ldots,\abs{\cS} \rbrace$ and $\cA=\lbrace 1,\ldots,\abs{\cA} \rbrace$. This corresponds to taking $k_R$ and $k_P$ as product kernels, i.e., $k_R((i,j),(i',j'))=k_P((i,j),(i',j'))=k_1(i,i')k_2(j,j')$ for all $1 \le i,i' \le \abs{\cS}$ and $1 \le j,j' \le \abs{\cA}$, where both $k_1$ and $k_2$ are linear kernels with dimensions $\abs{\cS}$ and $\abs{\cA}$, respectively, such that $k_1(i,i')=\mathds{1}_{\lbrace i=i'\rbrace}$ and $k_2(j,j')=\mathds{1}_{\lbrace j=j'\rbrace}$. Hence (\ref{eqn:product-kernel}) implies that $\gamma_T(k_R,\cZ) \le \abs{\cA}\gamma_T(k_1,\cS)+\abs{\cA}\ln T$, as the rank of $k_2$ is at most $\abs{\cA}$. Further, as $k_1$ is a linear kernel, $\gamma_T(k_1,\cS)=\tilde{O}(\abs{\cS}\ln T)$, hence $\gamma_T(R)\equiv \gamma_T(k_R,\cZ)= \tilde{O}(\abs{\cS}\abs{\cA}\ln T)$. Similarly, $\gamma_{mT}(P)=\tilde{O}(\abs{\cS}\abs{\cA}\ln T)$ as in this case $m=1$. Further, in this case, the Lipschitz constant $L_\star= O(H)$. Plugging these into our bounds, we see that both GP-UCRL and PSRL suffer regret $\tilde{O}(H\abs{\cS}\abs{\cA}\sqrt{T})$ for tabula-rasa MDPs.

\paragraph{2. Control of Bounded Linear Quadratic Systems} 
\label{subsec:LQR}

Consider learning under the standard discrete-time, episodic, linear quadratic regulator (LQR) model: at period $h$ of episode $l$,
\beq
\label{eqn:LQR}
\begin{aligned}
s_{l,h+1} &= As_{l,h} + Ba_{l,h} + \epsilon_{P,l,h}, \\
r_{l,h}&= s_{l,h}^TPs_{l,h} + a_{l,h}^TQa_{l,h} + \epsilon_{R,l,h},
\end{aligned}
\eeq
where $r_{l,h} \in \Real$ is the reward obtained by executing action $a_{l,h} \in \cA \subset \Real^n$ at state $s_{l,h}\in \cS \subset \Real^m$ and $s_{l,h+1}$ is the next state.
$P \in \Real^{m\times m}$, $Q\in \Real^{n \times n}$, $A \in \Real^{m\times m}$ and $B \in \Real^{m\times n}$ are unknown matrices with $P$ and $Q$ assumed positive-definite, and $\epsilon_{R,l,h}$, $\epsilon_{P,l,h}$ follow a sub-Gaussian noise model as per \ref{eqn:noise-reward} and \ref{eqn:noise-state}, respectively.

\begin{mycorollary}[Regret of GP-UCRL for LQR]
\label{cor:LQR-freq}
Let $M_\star$ be a linear quadratic system (\ref{eqn:LQR}), and both $\cS$ and $\cA$ be compact and convex. Let $H \le \min\lbrace(m^2+n^2)\ln T,(m+n)\ln(mT)\rbrace$\footnote{This assumption naturally holds in most settings and is used here only for brevity.}. Then, for any $0 < \delta \le 1$, GP-UCRL enjoys, with probability at least $1-\delta$, the regret bound
\beqn
Regret(T)= \tilde{O}\Big(\big(B_R(m^2+n^2) + LB_P m(m+n)\big)\sqrt{T \ln(1/\delta)}\Big).
\eeqn
Here, $B_R=\big(\norm{P}_F^2+\norm{Q}_F^2\big)^{1/2}$ and $B_P=\big(\norm{A}_F^2+\norm{B}_F^2\big)^{1/2}$, $L$ is a known upper bound over $D\lambda_1$, $D = \max_{s,s'\in \cS}\norm{s-s'}_2$ is the diameter of $\cS$ and $\lambda_1$ is the largest eigenvalue of the (positive-definite) matrix $G$, which is a unique solution to the Riccati equations \cite{lancaster1995algebraic} for the unconstrained optimal value function $V(s)=s^TGs$. 

\end{mycorollary}

\begin{proof} 
The proof is essentially by using composite kernels, based on linear and quadratic kernels, to represent the LQR model. 
First, note that the mean reward function is $\overline{R}_\star(s,a)=s^TPs+a^TQa$
and the mean transition function is $\overline{P}_\star(s,a)=As+Ba$. Now recall our notation $z=(s,a)$, $z'=(s',a')$, $\cZ=\cS \times \cA$ and $\tilde{\cZ}=\cZ \times \lbrace 1,\ldots,m\rbrace$. Then defining $\overline{P}_\star=[A_1,\ldots,A_m,B_1,\ldots,B_m]^T$, where $A_i, 1 \le i \le m$ and $B_i,1\le i \le m$ are the rows $A$ and $B$ respectively, we see that $\overline{P}_\star$ lies in the RKHS $\cH_{k_P}(\tilde{\cZ})$ with the kernel $k_P\big((z,i),(z',j)\big)=k_1(z,z')k_2(i,j)$, where $k_1(z,z')=s^Ts'+a^Ta'$ and $k_2(i,j)=\mathds{1}_{\lbrace i=j \rbrace}, 1\le i,j \le m$. Since $k_P$ is a product of the kernels $k_1$ and $k_2$, (\ref{eqn:product-kernel}) implies that
\beq
\label{eqn:composite-state-one}
\gamma_t(P)\equiv\gamma_t(k_P,\tilde{\cZ})\le m \gamma_t(k_1,\cZ) + m \ln t,
\eeq
as the rank of $k_2$ is atmost $m$. Further $k_1$ is a sum of two linear kernels, defined over $\cS$ and $\cA$ respectively. Hence (\ref{eqn:additive-kernel}) implies
$\gamma_t(k_1,\cZ) \le \tilde{O}(m\ln t) + \tilde{O}(n\ln t) + 2 \ln t=\tilde{O}\big((m+n)\ln t\big)$, since the MIG of a $d$-dimensional linear kernel is $\tilde{O}(d\ln t)$.
Hence, by (\ref{eqn:composite-state-one}), we have $\gamma_t(P) = \tilde{O}\big(m(m+n)\ln t\big)$.\\ 
Similarly defining $\overline{R}_\star=[P_1,\ldots,P_m,Q_1,\ldots,Q_n]^T$, where $P_i, 1 \le i \le m$ and $Q_i,1\le i \le n$ are the rows $P$ and $Q$ respectively, we see that $\overline{R}_\star$ lies in the RKHS $\cH_{k_R}(\cZ)$ with the quadratic kernel $k_R(z,z')=(s^Ts')^2+(a^Ta')^2$. Since $k_R$ is an additive kernel, (\ref{eqn:additive-kernel}) implies that
\beq
\label{eqn:composite-reward-one}
\gamma_t(R)=\gamma_t(k_R,\cZ) \le \gamma_t(k_3,\cS) + \gamma_t(k_3,\cA) + 2 \ln t,
\eeq
where $k_3(x,x')\bydef (x^Tx')^2=(x^Tx')(x^Tx')$ is a quadratic kernel and thus a product of two linear kernels. Hence, (\ref{eqn:product-kernel}) implies that
$\gamma_t(k_3,\cS) \le m\; \tilde{O}(m\ln t) + m \ln t =\tilde{O}(m^2\ln t)$, since the rank of an $m$-dimensional linear kernel is at most $m$. Similarly $\gamma_t(k_3,\cA)=\tilde{O}(n^2\ln t)$. Hence from (\ref{eqn:composite-reward-one}), we have $\gamma_t(R)=\tilde{O}\big((m^2+n^2)\ln t\big)$. Now, following a similar argument as by \citet[Corollary 2]{osband2014model}, we can show that the Lipschitz constant  $L_\star=D\lambda_1$. Further, in this setting, we take $B_R=\norm{\overline{R}_\star}_{k_R}=\big(\norm{P}_F^2+\norm{Q}_F^2\big)^{1/2}$
and  $B_P=\norm{\overline{P}_\star}_{k_P}=\big(\norm{A}_F^2+\norm{B}_F^2\big)^{1/2}$. Now the result follows from Theorem \ref{thm:regret-bound-RKHS} using $H \le \min\lbrace(m^2+n^2)\ln T,(m+n)\ln(mT)\rbrace$.
\end{proof}
\vspace*{-20pt}

\begin{mycorollary}[Bayes regret of PSRL for LQR] 
\label{cor:LQR-bayes}
Let $M_\star$ be a linear quadratic system defined as per (\ref{eqn:LQR}),
$\Phi$ be the (known) distribution of $M_\star$ and both $\cS$ and $\cA$ be compact and convex.  Let $H \le \min\lbrace(m^2+n^2)\ln T,(m+n)\ln(mT)\rbrace$. Then the Bayes regret of PSRL satisfies
\beqn
\expect{Regret(T)}= \tilde{O}\Big(\big(B_R(m^2+n^2) + D\lambda_1 B_Pm(m+n)\big)\sqrt{T}\Big),
\eeqn
where $B_R$, $B_P$, $D$ and $\lambda_1$ are as given in Corollary \ref{cor:LQR-freq}.
\end{mycorollary}
 \begin{proof}
Using the similar arguments as above and noting that $\expect{L_\star}=D \lambda_1$, the result follows from Theorem \ref{thm:regret-bound-PSRL}.
\end{proof}
\textit{\textbf{Remark.}} Corollary \ref{cor:LQR-bayes} matches the bound given in \citet{osband2014model} for the same bounded LQR problem. But the analysis technique is different here, and this result is derived as a special case of more general kernelized dynamics. Corollary \ref{cor:LQR-freq} (order-wise) matches the bound given in \citet{abbasi2011regret} if we restrict their result to the bounded LQR problem.

\section{DISCUSSION}
\vspace*{-2mm}
We have derived the first regret bounds for RL in the kernelized MDP setup with continuous state and action spaces, with explicit dependence of the bounds on the maximum information gains of the transition and reward function classes. In Appendix \ref{appendix:Bayesian}, we have also developed the Bayesian RL analogue of Gaussian process bandits \cite{srinivas2009gaussian}, i.e., learning under the assumption that MDP dynamics and reward behavior are sampled according to Gaussian process priors. We have proved Bayesian regret bounds for GP-UCRL and PSRL under GP priors.
We only have a (weak) Bayes regret bound for PSRL in kernelized MDPs, and would like to examine if a frequentist bound also holds. Another concrete direction is to examine if similar guarantees can be attained in the model-free setup, which may obviate complicated planning in the model-based setup here.

\bibliography{paper}

\begin{appendices}
\section{PRELIMINARIES}
\subsection{Relevant Results on Gaussian Process Multi-armed Bandits}
\label{appendix:GP}
We first review some relevant definitions and results from the Gaussian process multi-armed bandits literature, which will be useful in the analysis of our algorithms. We first begin with the definition of \textit{Maximum Information Gain}, first appeared in \citet{srinivas2009gaussian}, which basically measures the reduction in uncertainty about the unknown function after some noisy observations (rewards).

For a function $f:\cX \ra \Real$ and any subset $A \subset \cX$ of its domain, we use $f_A := [f(x)]_{x\in A}$ to denote its restriction to $A$, i.e., a vector containing $f$'s evaluations at each point in $A$ (under an implicitly understood bijection from coordinates of the vector to points in $A$). In case $f$ is a random function, $f_A$ will be understood to be a random vector. For jointly distributed random variables $X, Y$, $I(X;Y)$ denotes the Shannon mutual information between them.

\begin{mydefinition}[Maximum Information Gain (MIG)] 
\label{def:mig}
Let $f:\cX \ra \Real$ be a (possibly random) real-valued function defined on a domain $\cX$, and $t$ a positive integer. For each subset $A \subset \cX$, let $Y_A$ denote a noisy version of $f_A$ obtained by passing $f_A$ through a channel $\prob{Y_A | f_A}$. The \textit{Maximum Information Gain (MIG)} about $f$ after $t$ noisy observations is defined as
\beqn
\gamma_t(f, \cX) \bydef \max_{A \subset \cX : \abs{A}=t} I(f_A;Y_A).
\eeqn
(We omit mentioning explicitly the dependence on the channels for ease of notation.)
\end{mydefinition}

\textit{MIG} will serve as a key instrument to obtain our regret bounds by virtue of Lemma \ref{lem:sum-of-sd}.

For a kernel function $k: \cX \times \cX \ra \Real$ and points $x, x_1,\ldots,x_s \in \cX$, we define the vector $k_{s}(x)\bydef [k(x_1,x),\ldots,k(x_{s},x)]^T$ of kernel evaluations between $x$ and $x_1, \ldots, x_s$, and $K_{\{x_1, \ldots, x_s\}} \equiv K_{s} \bydef [k(x_i,x_j)]_{1 \le i,j \le s}$ be the kernel matrix induced by the $x_i$s. Also for each $x \in \cX$ and $\lambda > 0$, let  $\sigma_{s}^2(x) \bydef k(x,x) - k_{s}(x)^T(K_{s} + \lambda I)^{-1} k_{s}(x)$.

\begin{mylemma}[Information Gain and Predictive Variances under GP prior and additive Gaussian noise]
\label{lem:sum-of-sd}
Let $k: \cX \times \cX \ra \Real$ be a symmetric positive semi-definite kernel and $f \sim GP_{\cX}(0,k)$ a sample from the associated Gaussian process over $\cX$. For each subset $A \subset \cX$, let $Y_A$ denote a noisy version of $f_A$ obtained by passing $f_A$ through a channel that adds iid $\cN(0,\lambda)$ noise to each element of $f_A$. Then, \beq
\label{eqn:info-gain-zero}
\gamma_t(f, \cX)  =  \max_{A \subset \cX : \abs{A}=t} \frac{1}{2} \ln \abs{I + \lambda^{-1}K_A},
\eeq
and
\beq
\label{eqn:info-gain-one}
\gamma_t(f, \cX) = \max_{ \lbrace x_1,\ldots,x_t \rbrace \subset \cX}\frac{1}{2}\sum_{s=1}^{t}\ln\left(1 +\lambda^{-1}\sigma_{s-1}^2(x_s) \right).
\eeq
\end{mylemma}
\begin{proof}
The proofs follow from \citet{srinivas2009gaussian}. 
\end{proof}
\textit{\textbf{Remark.}}
Note that the right hand sides of (\ref{eqn:info-gain-zero}) and (\ref{eqn:info-gain-one}) depend only on the kernel function $k$, domain $\cX$, constant $\lambda$ and number of observations $t$. Further, as shown in Theorem 8 of \citet{srinivas2009gaussian}, the dependency on $\lambda$ is only of $\tilde{O}(1/\lambda)$. Hence to indicate these dependencies on $k$ and $\cX$ more explicitly, we denote the Maximum Information Gain $\gamma_t(f, \cX)$ in the setting of Lemma \ref{lem:sum-of-sd} as $\gamma_t(k,\cX)$.

\begin{mylemma}[Sum of Predictive variances is bounded by MIG]
\label{lem:sum-of-sd-two}
Let $k: \cX \times \cX \ra \Real$ be a symmetric positive semi-definite kernel such that it has bounded variance, i.e. $k(x,x) \le 1$ for all $x \in \cX$ and $f \sim GP_{\cX}(0,k)$ be a sample from the associated Gaussian process over $\cX$, then for all $s \ge 1$ and $x \in \cX$,
\beq
\label{eqn:info-gain-two}
\sigma^2_{s-1}(x) \le (1+\lambda^{-1})\sigma^2_s(x),
\eeq
and 
\beq
\label{eqn:info-gain-three}
\sum_{s=1}^{t}\sigma^2_{s-1}(x_s) \le (2\lambda+1)\gamma_t(k, \cX).
\eeq
\end{mylemma}

\begin{proof} 
From our assumption $k(x,x) \le 1$, we have $0 \le \sigma_{s-1}^2(x)\le 1$ for all $x \in \cX$, and hence $\sigma_{s-1}^2(x_s) \le \ln\big(1+\lambda^{-1}\sigma_{s-1}^2(x_s)\big)/\ln(1+\lambda^{-1})$ since $\alpha/\ln(1+\alpha)$ is non-decreasing for any $\alpha \in [0,\infty)$. Therefore
\beqn
\sum_{s=1}^{t}\sigma^2_{s-1}(x_s) \le 2/\ln(1+\lambda^{-1}) \sum_{s=1}^{t}\dfrac{1}{2}\ln\left(1+\lambda^{-1}\sigma_{s-1}^2(x_s)\right)\le 2\gamma_t(k, \cX)/\ln(1+\lambda^{-1}),
\eeqn
where the last inequality follows from (\ref{eqn:info-gain-one}). Now see that $2/\ln(1+\lambda^{-1}) \le (2+\lambda^{-1})/\lambda^{-1} =2\lambda+1$, since $\ln(1+\alpha) \ge 2\alpha/(2+\alpha)$ for any $\alpha \in [0,\infty)$. Hence $\sum_{s=1}^{t}\sigma^2_{s-1}(x_s) \le (2\lambda +1)\gamma_t(k,\cX)$.

Further from Appendix F in \citet{chowdhury2017kernelized}, see that $\sigma^2_{s}(x)=\sigma^2_{s-1}(x) -k^2_{s-1}(x_s,x)/\big(\lambda+\sigma^2_{s-1}(x_s))$ for all $x \in \cX$, where $k_{s}(x,x')\bydef k(x,x')-k_{s}(x)^T(K_s+\lambda I)^{-1}k_s(x')$. Since $k_{s-1}(x,\cdot), x\in \cX$ lie in the reproducing kernel Hilbert space (RKHS) of $k_{s-1}$,
the reproducing property implies that $k_{s-1}(x_s,x)=\inner{k_{s-1}(x_s,\cdot)}{k_{s-1}(x,\cdot)}_{k_{s-1}}$. Hence by Cauchy-Schwartz inequality $k^2_{s-1}(x_s,x)\le \norm{k_{s-1}(x_s,\cdot)}^2_{k_{s-1}} \norm{k_{s-1}(x,\cdot)}^2_{k_{s-1}}=k_{s-1}(x_s,x_s)k_{s-1}(x,x)=\sigma^2_{s-1}(x_s)\sigma^2_{s-1}(x)$, where the second last step follows from the reproducing property and the last step is due to $\sigma^2_s(x)=k_s(x.x)$. Therefore $\sigma^2_{s}(x) \ge \sigma^2_{s-1}(x)\Big(1-\dfrac{\sigma^2_{s-1}(x_s)}{\lambda+\sigma^2_{s-1}(x_s)}\Big)=\lambda \sigma^2_{s-1}(x)/\big(\lambda+\sigma^2_{s-1}(x_s)\big)$. Further by the bounded variance assumption, $\sigma^2_{s-1}(x_s) \le 1$ and hence $\lambda/\big(\lambda+\sigma^2_{s-1}(x_s)\big) \ge \lambda/(1+\lambda)$. This implies $\sigma^2_s(x)/\sigma^2_{s-1}(x) \ge \lambda/(1+\lambda)$ and hence $\sigma^2_{s-1}(x) \le (1+\lambda^{-1})\sigma^2_s(x)$.
\end{proof}



\begin{mylemma}[Ratio of predictive variances is bounded by Information Gain  \cite{kandasamy2018parallelised}]
\label{lem:conditional-mi}
Let $k: \cX \times \cX \to \Real$ be a symmetric, positive-semidefinite kernel and $f \sim GP_{\cX}(0, k)$. Further, let $A$ and $B$ be finite subsets of $\cX$, and for a positive constant $\lambda$, let $\sigma_A$ and  $\sigma_{A \cup B}$ be the posterior
standard deviations conditioned on queries $A$ and $A \cup B$ respectively (similarly defined as in Lemma \ref{lem:sum-of-sd}). Also, let $\gamma_(k,\cX)$ denote the maximum information gain after $t$ noisy observations. Then the following holds for all $x \in \cX$: 
\beq
\label{eqn:info-gain-in-process}
\max_{A,B \subset \cX : \abs{B} = t}\;\dfrac{\sigma_{A}(x)}{\sigma_{A \cup B}(x)} \le  \exp\big(\gamma_t(k,\cX)\big).
\eeq
\end{mylemma}
\begin{proof}
The proof can be figured out from \citet{desautels2014parallelizing}, but we include it here for completeness. Let $Y_A$ and $Y_B$ are vectors of noisy observations when we query $f$ at $A$ and $B$ respectively, and $I\left(f(x);Y_{B} \given Y_A\right)$ denotes the \textit{mutual information} between $f(x)$ and $Y_B$, conditioned on $Y_A$. Note that 
\beqn
I\big(f(x);Y_{B} \given Y_A\big)=H\big(f(x)\given Y_A\big)-H\big(f(x)\given Y_{A\cup B}\big)
=\dfrac{1}{2}\ln\big(2\pi e \sigma_A^2(x)\big)-\dfrac{1}{2}\ln\big(2\pi e \sigma_{A\cup B}^2(x)\big)
= \ln\Bigg(\dfrac{\sigma_{A}(x)}{\sigma_{A \cup B}(x)}\Bigg).
\eeqn
Hence for all $x \in \cX$ and for all finite subsets $A,B$ of $\cX$, we have
\beq
\label{eqn:combine-one}
\dfrac{\sigma_A(x)}{\sigma_{A\cup B}(x)}=\exp \Big(I\big(f(x);Y_{B} \given Y_A\big)\Big).
\eeq
Now, by monotonicity of \textit{mutual information}, we have $I\left(f(x);Y_{B}\given Y_{A}\right) \le I\left(f;Y_{B}\given Y_{A}\right)$ for all $x \in \cX$. 
Further, if $\abs{B} = t$, we have $I\left(f;Y_{B}\given Y_{A}\right) \le \max_{B \subset \cX : \abs{B} = t}\; I\left(f;Y_{B}\given Y_{A}\right)$. Thus for all $x \in \cX$ and for all finite subset $B$ of $\cX$ for which $\abs{B} = t$, we have 
\beqn
I\left(f(x);Y_{B}\given Y_{A}\right) \le \max_{B \subset \cX : \abs{B} = t}\; I\left(f;Y_{B}\given Y_{A}\right).
\eeqn
Now by submodularity of \textit{conditional mutual information}, for all finite subset $A$ of $\cX$, we have
\beqn
\max_{B \subset \cX : \abs{B} = t}\; I\left(f;Y_{B}\given Y_{A}\right) \le \max_{B \subset \cX : \abs{B} = t}\; I\left(f;Y_{B}\right). 
\eeqn
Further see that $I\left(f;Y_{B}\right) = I\left(f_B;Y_{B}\right)$, since $H(Y_B\given f)=H(Y_B\given f_B)$. This implies, 
for all $x \in \cX$ and for all finite subsets $A,B$ of $\cX$ for which $\abs{B} = t$, that
\beq
\label{eqn:combine-two}
I\left(f(x);Y_{B}\given Y_{A}\right) \le \max_{B \subset \cX : \abs{B} = t}\; I\left(f_B;Y_{B}\right) = \gamma_t(k,\lambda,\cX).
\eeq
Now the result follows by combining (\ref{eqn:combine-one}) and (\ref{eqn:combine-two}).
\end{proof}
%
%
\paragraph{Bound on Maximum Information Gain} For any positive constant $\lambda$,
\citet{srinivas2009gaussian} proved upper bounds over $\gamma_t(k,\lambda,\cX)$ for three commonly used kernels, namely \textit{Linear}, \textit{Squared Exponential} and \textit{Mat$\acute{e}$rn}, defined respectively as
\beqan
k_{Linear}(x,x') &=& x^Tx',\\
k_{SE}(x,x') &=& \exp\left(-s^2/2l^2\right), \\
k_{Mat\acute{e}rn}(x,x') &=& \frac{2^{1-\nu}}{\Gamma(\nu)}\left(\frac{s\sqrt{2\nu}}{l}\right)^\nu B_\nu\left(\frac{s\sqrt{2\nu}}{l}\right),
\eeqan
where $l > 0$ and $\nu > 0$ are hyper-parameters of the kernels, $s = \norm{x-x'}_2$ encodes the similarity between two points $x,x'\in \cX$ and $B_\nu$ denotes the \textit{modified Bessel function}. The bounds are given in Lemma \ref{lem:info-gain-bound}.
\begin{mylemma}[MIG for common kernels]
\label{lem:info-gain-bound}
Let $k: \cX \times \cX \ra \Real$ be a symmetric positive semi-definite kernel and $f \sim GP_{\cX}(0,k)$. Let $\cX$ be a compact and convex subset of $\Real^d$ and the kernel $k$ satisfies $k(x,x') \le 1$ for all $x,x' \in \cX$. Then for
\begin{itemize}
\item Linear kernel: $\gamma_t(k_{Linear},\cX)=\tilde{O}(d\ln t)$.
\item Squared Exponential kernel: $\gamma_t(k_{SE},\cX)=\tilde{O}\left((\ln t)^{d}\right)$.
\item Mat$\acute{e}$rn kernel: $\gamma_t(k_{Mat\acute{e}rn},\cX)=\tilde{O}\left(t^{d(d+1)/(2\nu+d(d+1))}\ln t\right)$.
\end{itemize}
\end{mylemma}

Note that, the Maximum Information Gain $\gamma_t(k,\cX)$ depends only \textit{poly-logarithmically} on the number of observations $t$ for all these kernels.

\paragraph{Reproducing kernel Hilbert spaces (RKHS)}
A Reproducing kernel Hilbert space (RKHS) $\cH_k(\cX)$ is a complete
subspace of the space of square integrable functions $L_2(\cX)$ defined over the domain $\cX$. It
includes functions of the form $f(x) = \sum_i \alpha_i k(x,x_i)$
with $\alpha_i \in \Real$ and $x_i \in \cX$, where $k$ is a symmetric, positive-
definite kernel function. The RKHS has an inner product $\inner{\cdot}{\cdot}_k$, which obeys the reproducing property: $f(x)=\inner{f}{k(x,\cdot)}_k$ for all $f \in \cH_k(\cX)$, and the induced RKHS norm $\norm{f}_k^2=\inner{f}{f}_k$ measures smoothness
of $f$ with respect to the kernel $k$. 
Lemma \ref{lem:true-function-bound} gives a concentration bound for a member $f$ of $\cH_k(\cX)$. A (slightly) modified version of Lemma \ref{lem:true-function-bound} has appeared independently in \citet{pmlr-v70-chowdhury17a} and 
\citet{durand2017streaming}.

\begin{mylemma}[Concentration of an RKHS member]
\label{lem:true-function-bound}
Let $k: \cX \times \cX \to \Real$ be a symmetric, positive-semidefinite kernel and $f:\cX \to \Real$ be a member of the RKHS $\cH_k(\cX)$ of real-valued functions on $\cX$ with kernel $k$. Let $\lbrace x_t \rbrace_{t \ge 1}$ and $\lbrace \epsilon_t \rbrace_{t \ge 1}$ be stochastic processes such that $\lbrace x_t \rbrace_{t \ge 1}$ form a predictable process, i.e., $x_t \in \sigma(\lbrace x_s, \epsilon_s\rbrace_{s = 1}^{t-1})$ for each $t$, and $\lbrace \epsilon_t \rbrace_{t \ge 1}$ is conditionally $R$-sub-Gaussian for a positive constant $R$, i.e.,
\beqn
\forall t \ge 0,\;\;\forall \lambda \in \Real, \;\; \expect{e^{\lambda \epsilon_t} \given \cF_{t-1}} \le \exp\left(\frac{\lambda^2R^2}{2}\right),
\eeqn
where $\cF_{t-1}$ is the $\sigma$-algebra generated by $\lbrace x_s, \epsilon_s\rbrace_{s = 1}^{t-1}$ and $x_t$. Let $\lbrace y_t\rbrace_{t\ge 1}$ be a sequence of noisy observations at the query points $\lbrace x_t \rbrace_{t \ge 1}$, where $y_t=f(x_t)+\epsilon_t$. For $\lambda > 0$ and $x \in \cX$, let
\beqan
\mu_{t-1}(x)&:=& k_{t-1}(x)^T(K_{t-1} + \lambda I)^{-1}Y_{t-1},\\
\sigma_{t-1}^2(x)&:=&k(x,x) - k_{t-1}(x)^T(K_{t-1} + \lambda I)^{-1} k_{t-1}(x),
\eeqan
where $Y_{t-1}\bydef[y_1,\ldots,y_{t-1}]^T$ denotes the vector of observations at $\lbrace x_1,\ldots,x_{t-1}\rbrace$.
Then, for any $0 < \delta \le 1$, with probability at least $1-\delta$, uniformly over $t \ge 1, x \in \cX$,
\beqn
\abs{f(x)-\mu_{t-1}(x)}\le \Bigg(\norm{f}_k + \dfrac{R}{\sqrt{\lambda}}\sqrt{2\Big(\ln(1/\delta)+\frac{1}{2}\sum_{s=1}^{t-1}\ln\big(1 +\lambda^{-1}\sigma_{s-1}^2(x_s) \big) \Big)}\Bigg)\sigma_{t-1}(x).
\eeqn
\end{mylemma}
\begin{proof}
The proof follows from the proof of Theorem 2.1 in \citet{durand2017streaming}.
\end{proof}

\subsection{Relevent Results on Episodic Continuous Markov Decision Processes}
\label{appendix:common}
\begin{mydefinition}[Bellman operator]
\label{def:bellman-operator}
For any MDP $M=\lbrace \cS,\cA,R_M,P_M,H \rbrace$, any policy $\pi:\cS \times \lbrace 1,\ldots,H \rbrace \ra \cA$, any period $1 \le h \le H$, any value function $V:\cS \ra \Real$ and any state $s \in \cS$, the Bellman operator $T^M_{\pi,h}$ is defined as
\beqn
\big(T^M_{\pi,h} V\big)(s)=\overline{R}_M\big(s,\pi(s,h)\big) + \mathbb{E}_{s'}\big[V(s')\big],
\eeqn
where the subscript $s'$ implies that $s' \sim P_M\big(s,\pi(s,h)\big)$ and $\overline{R}_M$ denotes the mean reward function.
\end{mydefinition}
This operator returns the expected value of the state $s$, where we follow the policy $\pi(s,h)$ for one step under $P_M$.
\begin{mylemma}[Bellman equation]
\label{lem:dp}
For any MDP $M=\lbrace \cS,\cA,R_M,P_M,H \rbrace$, any policy $\pi:\cS \times \lbrace 1,\ldots,H \rbrace \ra \cA$ and any period $1 \le h \le H$, the value functions $V^M_{\pi,h}$ satisfy
\beqn
V^M_{\pi,h}(s)=\big(T^M_{\pi,h} V^M_{\pi,h+1}\big)(s)
\eeqn
for all $s \in \cS$, with $V^M_{\pi,H+1}\bydef 0$.
\end{mylemma}
\begin{proof}
For any MDP $M=\lbrace \cS,\cA,R_M,P_M,H,\rbrace$ and policy $\pi:\cS \times \lbrace 1,\ldots,H \rbrace \ra \cA$, period $h \in \lbrace 1,\ldots,H \rbrace$ and state $s \in \cS$, recall the finite horizon, undiscounted, value function
\beqn
V^M_{\pi,h}(s) \bydef  \mathbb{E}_{M,\pi}\Bigg[\sum_{j=h}^{H}\overline{R}_M(s_j,a_j)\given s_h=s\Bigg],
\eeqn
where the subscript $\pi$ indicates the application of the learning policy $\pi$, i.e., $a_j=\pi(s_j,j)$, and the subscript $M$ explicitly references the MDP environment $M$, i.e.,  $s_{j+1} \sim P_M(s_j,a_j)$, for all $j=h,\ldots,H$. See that, by definition, $V^M_{\pi,H+1}(s) = 0$ for all $s \in \cS$. Further $V^M_{\pi,h}(s)$ can be rewritten as 
\beqan
V^M_{\pi,h}(s) &=& \overline{R}_M\big(s,\pi(s,h)\big)+\mathbb{E}_{M,\pi}\Bigg[\sum_{j=h+1}^{H}\overline{R}_M(s_j,a_j)\given s_h=s\Bigg]\\
&=&\overline{R}_M\big(s,\pi(s,h)\big)+\mathbb{E}_{s'}\Bigg[\mathbb{E}_{M,\pi}\bigg[\sum_{j=h+1}^{H}\overline{R}_M(s_j,a_j)\given s_{h+1}=s'\bigg]\Bigg]\\
&=&\overline{R}_M(s,\pi(s,h))+\mathbb{E}_{s'}\Big[ V^M_{\pi,h+1}(s') \Big],
\eeqan
where the subscript $s'$ implies that $s' \sim P_M\big(s,\pi(s,h)\big)$. Now the result follows from Definition \ref{def:bellman-operator}.
\end{proof}
\begin{mylemma}[Bounds on deviations of rewards and transitions imply bounds on deviation of the value function]
\label{lem:common}
Let $M_l,l \ge 1$ be a sequence of MDPs and for each $l \ge 1$ and $\pi_l$ be the optimal policy for the MDP $M_l$. Let $M_\star$ be an MDP with the transition function $P_\star$ and  $s_{l,h+1}\sim P_\star(s_{l,h},a_{l,h})$, where $a_{l,h}=\pi_l(s_{l,h},h)$. Now for all $l \ge 1$ and $1 \le h \le H$, define
\beqn
\Delta_{l,h} \bydef \mathbb{E}_{s' \sim P_\star(z_{l,h})}\Big[V^{M_l}_{\pi_l,h+1}(s')-V^{M_\star}_{\pi_l,h+1}(s')\Big]-\Big(V^{M_l}_{\pi_l,h+1}(s_{l,h+1})-V^{M_\star}_{\pi_l,h+1}(s_{l,h+1})\Big),
\eeqn
where $z_{l,h} \bydef (s_{l,h},a_{l,h})$.
Then for any $\tau \ge 1$, 
\beqn
\sum_{l=1}^{\tau}\Big(V^{M_l}_{\pi_l,1}(s_{l,1})-V^{M_\star}_{\pi_l,1}(s_{l,1})\Big) \le \sum_{l=1}^{\tau}\sum_{h=1}^{H}\Big( \abs{\overline{R}_{M_l}(z_{l,h}) - \overline{R}_\star(z_{l,h})}+L_{M_l} \norm{\overline{P}_{M_l}(z_{l,h}) - \overline{P}_\star(z_{l,h})}_2+\Delta_{l,h}\Big),
\eeqn
where $L_{M_l}$ is defined to be the global Lipschitz constant (\ref{eqn: lipschitz}) of one step future value function for MDP $M_l$.
\end{mylemma}
\begin{proof}
The arguments in this proof borrow ideas from \citet{osband2013more}. Applying Lemma \ref{lem:dp} for $h=1$, $s=s_{l,1}$ and two MDP-policy pairs $(M_l,\pi_l)$ and $(M_\star,\pi_l)$, we have 
\beqan
V^{M_l}_{\pi_l,1}(s_{l,1})-V^{M_\star}_{\pi_l,1}(s_{l,1}) &=& \big(T^{M_l}_{\pi_l,1} V^{M_l}_{\pi_l,2}\big)(s_{l,1})-\big(T^{M_\star}_{\pi_l,1} V^{M_\star}_{\pi_l,2}\big)(s_{l,1})\\
&= &\big(T^{M_l}_{\pi_l,1} V^{M_l}_{\pi_l,2}\big)(s_{l,1})-\big(T^{M_\star}_{\pi_l,1} V^{M_l}_{\pi_l,2}\big)(s_{l,1}) + \big(T^{M_\star}_{\pi_l,1} V^{M_l}_{\pi_l,2}\big)(s_{l,1})-\big(T^{M_\star}_{\pi_l,1} V^{M_\star}_{\pi_l,2}\big)(s_{l,1}).
\label{eqn:value-zero}
\eeqan
Further using Definition \ref{def:bellman-operator} for $M=M_\star$, $\pi=\pi_l$, $h=1$, $V=V^{M_l}_{\pi_l,2}$ and $s=s_{l,1}$, we have
\beq
\label{eqn:value-one}
\big(T^{M_\star}_{\pi_l,1} V^{M_l}_{\pi_l,2})(s_{l,1}\big)=\overline{R}_\star\big(s_{l,1},\pi_l(s_{l,1},1)\big) + \mathbb{E}_{s' \sim P_\star(s_{l,1},\pi_l(s_{l,1},1))}\Big[V^{M_l}_{\pi_l,2}(s')\Big],
\eeq
where $R_\star$ and $P_\star$ denote the reward and transition functions of the MDP $M_\star$ respectively. Again using Definition \ref{def:bellman-operator} for $M=M_\star$, $\pi=\pi_l$, $h=1$, $V=V^{M_\star}_{\pi_l,2}$ and $s=s_{l,1}$, we have
\beq
\label{eqn:value-two}
\big(T^{M_\star}_{\pi_l,1} V^{M_\star}_{\pi_l,2}\big)(s_{l,1})=\overline{R}_\star\big(s_{l,1},\pi_l(s_{l,1},1)\big) + \mathbb{E}_{s' \sim P_\star(s_{l,1},\pi_l(s_{l,1},1))}\Big[V^{M_\star}_{\pi_l,2}(s')\Big].
\eeq
Subtracting (\ref{eqn:value-two}) from (\ref{eqn:value-one}), we have
\beqan
\big(T^{M_\star}_{\pi_l,1} V^{M_l}_{\pi_l,2}\big)(s_{l,1})-\big(T^{M_\star}_{\pi_l,1} V^{M_\star}_{\pi_l,2})(s_{l,1}\big)&=&\mathbb{E}_{s' \sim P_\star(s_{l,1},\pi_l(s_{l,1},1))}\Big[V^{M_l}_{\pi_l,2}(s')-V^{M_\star}_{\pi_l,2}(s')\Big]\\
&=& V^{M_l}_{\pi_l,2}(s_{l,2})-V^{M_\star}_{\pi_l,2}(s_{l,2})+\Delta_{l,1},
\eeqan
where $\Delta_{l,1} \bydef \mathbb{E}_{s' \sim P_\star(s_{l,1},\pi_l(s_{l,1},1))}\Big[V^{M_l}_{\pi_l,2}(s')-V^{M_\star}_{\pi_l,2}(s')\Big]-\Big(V^{M_l}_{\pi_l,2}(s_{l,2})-V^{M_\star}_{\pi_l,2}(s_{l,2})\Big)$. Then (\ref{eqn:value-zero}) implies
\beqn
V^{M_l}_{\pi_l,1}(s_{l,1})-V^{M_\star}_{\pi_l,1}(s_{l,1})=V^{M_l}_{\pi_l,2}(s_{l,2})-V^{M_\star}_{\pi_l,2}(s_{l,2}) + (T^{M_l}_{\pi_l,1} V^{M_l}_{\pi_l,2})(s_{l,1})-(T^{M_\star}_{\pi_l,1} V^{M_l}_{\pi_l,2})(s_{l,1})+ \Delta_{l,1}. 
\eeqn
Now since $V^M_{\pi,H+1}(s)=0$ for any MDP $M$, policy $\pi$ and state $s$, an inductive argument gives 
\beq
\label{eqn:recursion}
V^{M_l}_{\pi_l,1}(s_{l,1})-V^{M_\star}_{\pi_l,1}(s_{l,1})=\sum_{h=1}^{H}\Big(\big(T^{M_l}_{\pi_l,h} V^{M_l}_{\pi_l,h+1}\big)(s_{l,h})-\big(T^{M_\star}_{\pi_l,h} V^{M_l}_{\pi_l,h+1}\big)(s_{l,h})+ \Delta_{l,h}\Big),
\eeq
where $\Delta_{l,h}\bydef \mathbb{E}_{s' \sim P_\star(s_{l,h},\pi_l(s_{l,h},h))}\Big[V^{M_l}_{\pi_l,h+1}(s')-V^{M_\star}_{\pi_l,h+1}(s')\Big]-\Big(V^{M_l}_{\pi_l,h+1}(s_{l,h+1})-V^{M_\star}_{\pi_l,h+1}(s_{l,h+1})\Big)$.
\par
Now using Definition \ref{def:bellman-operator} respectively for $M=M_l$ and $M=M_\star$ with $\pi=\pi_l$, $V=V^{M_l}_{\pi_l,h+1}$ and $s=s_{l,h}$, we have
\beqan
\big(T^{M_l}_{\pi_l,h} V^{M_l}_{\pi_l,h+1}\big)(s_{l,h})-\big(T^{M_\star}_{\pi_l,h} V^{M_l}_{\pi_l,h+1}\big)(s_{l,h})=\bigg(\overline{R}_{M_l}\big(s_{l,h},\pi_l(s_{l,h},h)\big)+\mathbb{E}_{s'\sim P_{M_l}(s_{l,h},\pi_l(s_{l,h},h))}\Big[V^{M_l}_{\pi_l,h+1}(s')\Big]\bigg)\\-\bigg(\overline{R}_\star\big(s_{l,h},\pi_l(s_{l,h},h)\big)+\mathbb{E}_{s'\sim P_\star(s_{l,h},\pi_l(s_{l,h},h))}\Big[V^{M_l}_{\pi_l,h+1}(s')\Big]\bigg).
\eeqan
Further using the fact that $a_{l,h}=\pi_l(s_{l,h},h))$ and defining $z_{l,h}=(s_{l,h},a_{l,h})$, we have
\beqan
&&\big(T^{M_l}_{\pi_l,h} V^{M_l}_{\pi_l,h+1}\big)(s_{l,h})-\big(T^{M_\star}_{\pi_l,h} V^{M_l}_{\pi_l,h+1}\big)(s_{l,h})\\&&=\overline{R}_{M_l}(s_{l,h},a_{l,h}) - \overline{R}_\star(s_{l,h},a_{l,h})+ \mathbb{E}_{s'\sim P_{M_l}(s_{l,h},a_{l,h})}\Big[V^{M_l}_{\pi_l,h+1}(s')\Big]-\mathbb{E}_{s'\sim P_\star(s_{l,h},a_{l,h})}\Big[V^{M_l}_{\pi_l,h+1}(s')\Big]\\
&&=\overline{R}_{M_l}(z_{l,h}) - \overline{R}_\star(z_{l,h})+ \mathbb{E}_{s'\sim P_{M_l}(z_{l,h})}\Big[V^{M_l}_{\pi_l,h+1}(s')\Big]-\mathbb{E}_{s'\sim P_\star(z_{l,h})}\Big[V^{M_l}_{\pi_l,h+1}(s')\Big].
\eeqan
and $\Delta_{l,h}= \mathbb{E}_{s' \sim P_\star(z_{l,h})}\Big[V^{M_l}_{\pi_l,h+1}(s')-V^{M_\star}_{\pi_l,h+1}(s')\Big]-\Big(V^{M_l}_{\pi_l,h+1}(s_{l,h+1})-V^{M_\star}_{\pi_l,h+1}(s_{l,h+1})\Big)$.

Now for an MDP $M$, a distribution $\phi$ over $\cS$ and for every period $1 \le h \le H$, recall that the \textit{one step future value function} is defined as  
\beqn
U_h^M(\phi) \bydef \mathbb{E}_{s' \sim \phi}\Big[V^M_{\pi_M,h+1}(s')\Big],
\eeqn
where $\pi_M$ denotes the optimal policy for the MDP $M$. Observe that $\pi_l$ is the optimal policy for the MDP $M_l$. This implies
\beqn
\big(T^{M_l}_{\pi_l,h} V^{M_l}_{\pi_l,h+1}\big)(s_{l,h})-\big(T^{M_\star}_{\pi_l,h} V^{M_l}_{\pi_l,h+1}\big)(s_{l,h})=\overline{R}_{M_l}(z_{l,h}) - \overline{R}_\star(z_{l,h})+U_h^{M_l}\big(P_{M_l}(z_{l,h})\big) - U_h^{M_l}\big(P_\star(z_{l,h})\big).
\eeqn 
Further (\ref{eqn: lipschitz}) implies
\beqn
U_h^{M_l}\big(P_{M_l}(z_{l,h})\big) - U_h^{M_l}\big(P_\star(z_{l,h})\big)
\le  L_{M_l} \norm{\overline{P}_{M_l}(z_{l,h}) - \overline{P}_\star(z_{l,h})}_2, 
\eeqn
where $L_{M_l}$ is defined to be the global Lipschitz constant (\ref{eqn: lipschitz}) of one step future value function for MDP $M_l$.
Hence we have
\beq
\label{eqn:future-value-function-bound}
\big(T^{M_l}_{\pi_l,h} V^{M_l}_{\pi_l,h+1}\big)(s_{l,h})-\big(T^{M_\star}_{\pi_l,h} V^{M_l}_{\pi_l,h+1}\big)(s_{l,h}) \le \abs{\overline{R}_{M_l}(z_{l,h}) - \overline{R}_\star(z_{l,h})} + L_{M_l} \norm{\overline{P}_{M_l}(z_{l,h}) - \overline{P}_\star(z_{l,h})}_2.
\eeq
Now the result follows by plugging (\ref{eqn:future-value-function-bound}) back in (\ref{eqn:recursion}) and summing over $l=1,\ldots,\tau$.
\end{proof}

\section{ANALYSIS OF GP-UCRL AND PSRL IN THE KERNELIZED MDPs} 
\subsection{Preliminary Definitions and Results}
\label{appendix:definitions}
Now we define the \textit{span} of an MDP, which is crucial to measure the difficulties in learning the optimal policy of the MDP \citep{jaksch2010near, bartlett2009regal}.
\begin{mydefinition}[Span of an MDP]
\label{def:span}
The span of an MDP $M$ is the maximum difference in value of any two states under the optimal policy, i.e.
\beqn
\Psi_{M}\bydef \max\limits_{s,s'\in \cS}V^{M}_{\pi_M,1}(s)-V^{M}_{\pi_M,1}(s').
\eeqn
\end{mydefinition}
Now define
$\Psi_{M,h}\bydef \max\limits_{s,s'\in \cS}V^{M}_{\pi_M,h}(s)-V^{M}_{\pi_M,h}(s')$ as the span of $M$ at period $h$ and let $\tilde{\Psi}_M\bydef \max_{h \in \lbrace 1,\ldots,H \rbrace}\Psi_{M,h}$ as the maximum possible span in an episode. Clearly $\Psi_M \le \tilde{\Psi}_M$. 


\begin{mydefinition}
A sequence of random variables $\lbrace Z_t \rbrace_{t\ge 1}$ is called a martingale difference sequence corresponding to a filtration $\lbrace \cF_t \rbrace_{t\ge 0}$, if for all $t \ge 1$, $Z_t$ is $\cF_t$-measurable, and for all $t\ge 1$,
\beqn
\expect{Z_t\given \cF_{t-1}} = 0.
\eeqn
\label{def:martingale-diff-seq}
\end{mydefinition}
\begin{mylemma}[Azuma-Hoeffding Inequality] 
If a martingale difference sequence $\lbrace Z_t \rbrace_{t\ge 1}$, corresponding to filtration $\lbrace \cF_t \rbrace_{t\ge 0}$, satisfies $\abs{Z_t} \le \alpha_t$ for some constant $\alpha_t$, for all $t = 1,\ldots,T$, then for any $0 < \delta \le 1$,
\beqn
\prob{\sum_{t=1}^{T}Z_t \le \sqrt{2\ln(1/\delta)\sum_{t=1}^{T}\alpha_t^2}\;} \ge 1-\delta. 	
\eeqn
\label{lem:azuma-hoeffding}
\end{mylemma}

\begin{mylemma}[Bound on Martingale difference sequence]
\label{lem:GP-UCRL}
Let $\cM_l$ be the set of plausible MDPs constructed by GP-UCRL (Algorithm \ref{algo:GP-UCRL}) at episode $l, l \ge 1$ and $M_l, l \ge 1$ be a sequence of MDPs such that $M_l \in \cM_l$ for each $l \ge 1$ and $\pi_l$ be the optimal policy for the MDP $M_l$ for each $l \ge 1$. Let $M_\star$ be an MDP with reward function $R_\star$ and transition function $P_\star$. Let $s_{l,h+1}\sim P_\star(s_{l,h},a_{l,h})$, where $a_{l,h}=\pi_l(s_{l,h},h)$. Now define $\Delta_{l,h}:= \mathbb{E}_{s' \sim P_\star(z_{l,h})}\Big[V^{M_l}_{\pi_l,h+1}(s')-V^{M_\star}_{\pi_l,h+1}(s')\Big]-\Big(V^{M_l}_{\pi_l,h+1}(s_{l,h+1})-V^{M_\star}_{\pi_l,h+1}(s_{l,h+1})\Big)$, with $z_{l,h}:=(s_{l,h},a_{l,h})$.
Then for any $0 \le \delta \le 1$ and $\tau \ge 1$, with probability at least $1-\delta$,
\beqan
\sum_{l=1}^{\tau}\sum_{h=1}^{H}\Delta_{l,h} \le (LD+2CH) \sqrt{2\tau H \ln(1/\delta)},
\eeqan
where $D \bydef \max_{s,s'\in \cS}\norm{s-s'}_2$ is the diameter of the state space $\cS$, $C$ is a uniform upper bound over the absolute value of the mean reward function $\overline{R}_\star$, i.e. $\abs{\overline{R}_\star(z)} \le C$ for all $z \in \cZ$ and $L $ is an upper bound over the global Lipschitz constant (\ref{eqn: lipschitz}) of one step future value function for MDP $M_\star$, i.e. $L_{M_\star} \le L$.
\end{mylemma}

\begin{proof}
First assume that $M_\star$ is fixed in advance. For each $l \ge 1$ and $h \in \lbrace 1,\ldots H\rbrace$, we define $\cH_{l-1} \bydef \lbrace s_{j,k},a_{j,k},r_{j,k},s_{j,k+1}\rbrace_{1 \le j \le l-1,1 \le k \le H}$ as the history of all observations till episode $l-1$ and $\cG_{l,h} \bydef \cH_{l-1} \cup \lbrace s_{l,k},a_{l,k},r_{l,k},s_{l,k+1}\rbrace_{1 \le k \le h}$ as the history of all observations till episode $l$ and period $h$. See that $\cH_0=\emptyset$ and $\cH_{l}=G_{l,H}$ for all $l \ge 1$. Further defining $\cG_{l,0} \bydef \cH_{l-1} \cup \lbrace s_{l,1}\rbrace$, we see that $\cG_{l,h} = \cG_{l,h-1}\cup \lbrace a_{l,h},r_{l,h},s_{l,h+1} \rbrace$ for all $h \in \lbrace 1,\ldots H\rbrace$. Clearly the sets $\cG_{l,h}$ satisfy $G_{l,0} \subset G_{l,1} \subset G_{l,2} \subset \ldots \subset G_{l,H} \subset G_{l+1,0}$ for all $l \ge 1$. Hence, the sequence of sets $\lbrace \cG_{l,h}\rbrace_{l \ge 1,0 \le h \le H}$ defines a filtration.

Now by construction in Algorithm \ref{algo:GP-UCRL}, $\cM_l$ is deterministic given $\cH_{l-1}$. Hence, $M_l$ and $\pi_l$ are also deterministic given $\cH_{l-1}$. This implies $\Delta_{l,h} = \mathbb{E}_{s' \sim P_\star(z_{l,h})}\Big[V^{M_l}_{\pi_l,h+1}(s')-V^{M_\star}_{\pi_l,h+1}(s')\Big]-\Big(V^{M_l}_{\pi_l,h+1}(s_{l,h+1})-V^{M_\star}_{\pi_l,h+1}(s_{l,h+1})\Big)$ is $G_{l,h}$ - measurable. Further note that $a_{l,h}=\pi_l(s_{l,h},h)$ is deterministic given $\cG_{l,h-1}$, as both $\pi_l$ and $s_{l,h}$ are deterministic given $\cG_{l,h-1}$. This implies
\beqa
\expect{\Delta_{l,h}\given \cG_{l,h-1}}&=&\mathbb{E}_{s' \sim P_\star(z_{l,h})}\Big[V^{M_l}_{\pi_l,h+1}(s')-V^{M_\star}_{\pi_l,h+1}(s')\Big]-\mathbb{E}_{s_{l,h+1} \sim P_\star(z_{l,h})}\Big[V^{M_l}_{\pi_l,h+1}(s_{l,h+1})-V^{M_\star}_{\pi_l,h+1}(s_{l,h+1})\Big]\nonumber\\ 
&=& 0. \label{eqn:expect-martingale}
\eeqa
Further, observe that $\abs{\Delta_{l,h}}\le \big(\max\limits_s V^{M_l}_{\pi_l,h+1}(s)-\min\limits_s V^{M_l}_{\pi_l,h+1}(s)\big)+\big(\max\limits_s V^{M_\star}_{\pi_l,h+1}(s)-\min\limits_s V^{M_\star}_{\pi_l,h+1}(s)\big)$.
The first term $\max\limits_s V^{M_l}_{\pi_l,h+1}(s)-\min\limits_s V^{M_l}_{\pi_l,h+1}(s)$ is upper bounded by $\tilde{\Psi}_{M_l}$, which is an upper bound over the \textit{span} of the MDP $M_l$ (Definition \ref{def:span}). Now from (\ref{eqn: lipschitz}), we get $\Psi_{M_l} \le L_{M_l}D$, where $D\bydef\max_{s,s'\in \cS}\norm{s-s'}_2$ is the diameter of the state space $\cS$ and $L_{M_l}$ is a global Lipschitz constant for the one step future value function. Further by construction of the set of plausible MDPs $\cM_l$ and as $M_l \in \cM_l$, we have $L_{M_l} \le L$. Hence, we have
$\max\limits_s V^{M_l}_{\pi_l,h+1}(s)-\min\limits_s V^{M_l}_{\pi_l,h+1}(s) \le LD$.
Now, since by our hypothesis $\abs{\overline{R}_\star(z)} \le C$ for all $z \in \cZ$, see that $V^{M_\star}_{\pi,h}(s) \le CH$ for all $\pi$, $1 \le h \le H$ and $s \in \cS$. Hence, we have 
$\max\limits_s V^{M_\star}_{\pi_l,h+1}(s)-\min\limits_s V^{M_\star}_{\pi_l,h+1}(s) \le 2CH$.

Therefore the sequence of random variables $\lbrace \Delta_{l,h}\rbrace_{l\ge 1,1 \le h \le H}$ is a martingale difference sequence (Definition \ref{def:martingale-diff-seq}) with respect to the filtration $\lbrace \cG_{l,h}\rbrace_{l \ge 1,0 \le h \le H}$, with $\abs{\Delta_{l,h}} \le LD+2CH$ for all $l \ge 1$ and $1 \le h \le H$. Thus, by Lemma \ref{lem:azuma-hoeffding}, for any $\tau \ge 1$ and $0 < \delta \le 1$, we have with probability at least $1-\delta$,
\beq
\label{eqn:martingale-diff-sum}
\sum_{l=1}^{\tau}\sum_{h=1}^{H} \Delta_{l,h} \le \sqrt{2\ln(1/\delta)\sum_{l=1}^{\tau}\sum_{h=1}^{H}(LD+2CH)^2} = (LD+2CH) \sqrt{2\tau H \ln(1/\delta)}.
\eeq
Now consider the case when $M_\star$ is random. Then we define $\tilde{\cH}_{l-1}\bydef \cH_{l-1}\cup M_\star$ and $\tilde\cG_{l,h} \bydef \cG_{l,h} \cup M_\star$. Then $\lbrace \Delta_{l,h}\rbrace_{l\ge 1,1 \le h \le H}$ is a martingale difference sequence with respect to the filtration $\lbrace \tilde{\cG}_{l,h}\rbrace_{l \ge 1,0 \le h \le H}$, and hence (\ref{eqn:martingale-diff-sum}) holds in this case also.
\end{proof}

\subsection{Regret Analysis}

Recall that at each episode $l \ge 1$, GP-UCRL constructs confidence sets $\cC_{R,l}$ and $\cC_{P,l}$ as 
\beq
\label{eqn:confidence-set-agnostic}
\begin{aligned}
\cC_{R,l}&=\big\lbrace f: \cZ \ra \Real \given \abs{f(z)-\mu_{R,l-1}(z)} \le \beta_{R,l}\sigma_{R,l-1}(z) \forall z \in \cZ \big\rbrace,\\
\cC_{P,l}&=\big\lbrace f: \cZ \ra \Real^m \given \norm{f(z)-\mu_{P,l-1}(z)}_2 \le \beta_{P,l}\norm{\sigma_{P,l-1}(z)}_2 \forall z \in \cZ\big\rbrace,
\end{aligned}
\eeq
where $\mu_{R,0}(z)=0$, $\sigma^2_{R,0}(z)=k_R(z,z)$ and for each $l \ge 1$,  
\beq
\label{eqn:post-reward-agnostic}
\begin{aligned}
\mu_{R,l}(z) &= k_{R,l}(z)^T(K_{R,l} + H I)^{-1}R_{l},\\
\sigma^2_{R,l}(z) &= k_R(z,z) - k_{R,l}(z)^T(K_{R,l} + H I)^{-1} k_{R,l}(z).
\end{aligned}
\eeq
Here $H$ is the number of periods, $I$ is the $(lH)\times (lH)$ identity matrix, $R_l = [r_{1,1},\ldots,r_{l,H}]^T$ is the vector of rewards observed at $\cZ_{l} = \lbrace z_{j,k} \rbrace_{1 \le j \le l,1 \le k \le H} =\lbrace z_{1,1},\ldots,z_{l,H}\rbrace$, the set of all state-action pairs available at the end of episode $l$. $k_{R,l}(z) = [k_R(z_{1,1},z),\ldots,k_R(z_{l,H},z)]^T$ is the vector kernel evaluations between $z$ and elements of the set $\cZ_{l}$ and $K_{R,l} = [k_R(u,v)]_{u,v \in \cZ_{l}}$ is the kernel matrix computed at $\cZ_{l}$. Further
$\mu_{P,l}(z) = [\mu_{P,l-1}(z,1),\ldots,\mu_{P,l-1}(z,m)]^T$ and $\sigma_{P,l}(z) = [\sigma_{P,l-1}(z,1),\ldots,\sigma_{P,l-1}(z,m)]^T$, where $\mu_{P,0}(z,i)=0$, $\sigma_{P,0}(z,i)=k_P\big((z,i),(z,i)\big)$ and for each $l \ge 1$,
\beq
\label{eqn:post-state-agnostic}
\begin{aligned}
\mu_{P,l}(z,i) &= k_{P,l}(z,i)^T(K_{P,l} + mH I)^{-1}S_l,\\
\sigma^2_{P,l}\big((z,i))\big) &= k_P((z,i),(z,i)) - k_{P,l}(z,i)^T(K_{P,l} + mH I)^{-1} k_{P,l}(z,i).
\end{aligned}
\eeq
Here $m$ is the dimension of the state space, $H$ is the number of periods, $I$ is the $(mlH)\times (mlH)$ identity matrix, $S_{l} = [s_{1,2}^T,\ldots,s_{l,H+1}^T]^T$ denotes the vector of state transitions at $\cZ_{l}=\lbrace z_{1,1},\ldots,z_{l,H}\rbrace$, the set of all state-action pairs available at the end of episode $l$. $k_{P,l}(z,i) = \Big[k_P\big((z_{1,1},1),(z,i)\big),\ldots,k_P\big((z_{l,H},m),(z,i)\big)\Big]^T$ is the vector of kernel evaluations between $(z,i)$ and elements of the set $\tilde{\cZ}_{l} = \big\lbrace (z_{j,k},i) \big\rbrace_{1 \le j \le l,1 \le k \le H, 1 \le i \le m} =\big\lbrace  (z_{1,1},1),\ldots,(z_{l,H},m)\big\rbrace $ and $K_{P,l} = \big[k_P(u,v)\big]_{u,v \in \tilde{\cZ}_{l}}$ is the kernel matrix computed at $\tilde{\cZ}_{l}$.
Here for any $0 < \delta \le 1$, $B_R, B_P,\sigma_R,\sigma_P > 0$,
$\beta_{R,l}\bydef B_R + \dfrac{\sigma_R}{\sqrt{H}}\sqrt{2\big(\ln(3/\delta)+\gamma_{(l-1)H}(k_R,\lambda_R,\cZ)\big)}$ and $\beta_{P,l} \bydef B_P + \dfrac{\sigma_P}{\sqrt{mH}}\sqrt{2\big(\ln(3/\delta)+\gamma_{m(l-1)H}(k_P,\lambda_P,\tilde{\cZ})\big)}$ are properly chosen confidence widths of $\cC_{R,l}$ and $\cC_{P,l}$ respectively.

\begin{mylemma}[Concentration of mean reward and mean transition functions]
\label{lem:kernel-least-squares}
Let $M_\star=\lbrace \cS,\cA,R_\star,P_\star,H\rbrace$ be an MDP with period $H$, state space $\cS \subset \Real^m$ and action space $\cA \subset \Real^n$. Let the mean reward function $\overline{R}_\star$ be a member of the RKHS $\cH_{k_R}(\cZ)$ corresponding to the kernel $k_R:\cZ \times \cZ \ra \Real$, where $\cZ \bydef \cS \times \cA$ and let the noise variables $\epsilon_{R,l,h}$ be conditionally $\sigma_R$-sub-Gaussian (\ref{eqn:noise-reward}). Let the mean transition function $\overline{P}_\star$ be a member of the RKHS $\cH_{k_P}(\tilde{\cZ})$ corresponding to the kernel $k_P:\tilde{\cZ \times \cZ \ra \Real}$, where $\tilde{\cZ}\bydef \cZ \times \lbrace 1,\ldots,m \rbrace$ and let the noise variables $\epsilon_{P,l,h}$ be conditionally component-wise independent and $\sigma_P$-sub-Gaussian (\ref{eqn:noise-state}). Further let $\norm{\overline{R}_\star}_{k_R} \le B_R$ and $\norm{\overline{P}_\star}_{k_P} \le B_P$. Then, for any $0 < \delta \le 1$, the following holds:
\beqa
\prob{\forall z \in \cZ, \forall l \ge 1, \abs{\overline{R}_\star(z)-\mu_{R,l-1}(z)}\le \beta_{R,l}\;\sigma_{R,l-1}(z)}&\ge& 1-\delta/3,\label{eqn:concentartion-agnostic-reward}\\
\prob{\forall z \in \cZ, \forall l \ge 1, \norm{\overline{P}_\star(z)-\mu_{P,l-1}(z)}_2 \le \beta_{P,l}\norm{\sigma_{P,l-1}(z)}_2} &\ge& 1-\delta/3,\label{eqn:concentartion-agnostic-state}
\eeqa
\end{mylemma}

\begin{proof}
First fix any $0 < \delta \le 1$. Now for all $l \ge 1$ and $1 \le h \le H+1$, let us define  $\cZ_{l,h-1}=\lbrace z_{j,k} \rbrace_{1 \le j \le l-1,1 \le k \le H} \cup \lbrace z_{l,k} \rbrace_{1 \le k \le h-1}=\lbrace z_{1,1},\ldots,z_{1,H},\ldots,z_{l,1},\ldots,z_{l,h-1}\rbrace$ as the set of all state-action pairs available till period $h-1$ of episode $l$. Further, let $R_{l,h-1} = [r_{1,1},\ldots,r_{1,H},\ldots,r_{l,1},\ldots,r_{l,h-1}]^T$ denotes the vector of rewards observed at $\cZ_{l,h-1}$. See that $\cZ_{l,0} = \cZ_{l-1,H}$ and $R_{l,0}= R_{l-1,H}$ for all $l \ge 2$. Also $R_{1,0}= 0$ and $\cZ_{1,0}=\emptyset$. Now we define, for all $z \in \cZ$, $l \ge 1$ and $1 \le h \le H+1$, the following:
\beq
\label{eqn:post-reward-def}
\begin{aligned}
\mu_{R,l,h-1}(z) &= k_{R,l,h-1}(z)^T(K_{R,l,h-1} + H I)^{-1}R_{l,h-1},\\
\sigma^2_{R,l,h-1}(z) &= k_R(z,z) - k_{R,l,h-1}(z)^T(K_{R,l,h-1} + HI)^{-1} k_{R,l,h-1}(z),
\end{aligned}
\eeq
where $k_{R,l,h-1}(z) = [k_R(z_{1,1},z),\ldots,k_R(z_{1,H},z),\ldots,k_R(z_{l,1},z),\ldots,k_R(z_{l,h-1},z)]^T$ is the vector kernel evaluations between $z$ and elements of the set $\cZ_{l,h-1}$, $K_{R,l,h-1} = [k_R(z,z')]_{z,z' \in \cZ_{l,h-1}}$ is the kernel matrix computed at $\cZ_{l,h-1}$. See that $\mu_{R,l,0}(z)= \mu_{R,l-1,H}(z)$, and $\sigma_{R,l,0}(z)= \sigma_{R,l-1,H}(z)$ for all $l \ge 2$ and $z \in \cZ$. Also
$\mu_{R,1,0}(z)=0$ and $\sigma_{R,1,0}=k_R(z,z)$ for all $z \in \cZ$.

At the state-action pair $z_{l,h}$, the reward observed is  $r_{l,h}=\overline{R}_\star(z_{l,h})+\epsilon_{R,l,h}$. Here, by our hypothesis, the mean reward function $\overline{R}_\star \in \cH_{k_R}(\cZ)$ and the noise sequence $\lbrace \epsilon_{R,l,h}\rbrace_{l\ge 1,1\le h \le H}$ is conditionally $\sigma_R$-sub-Gaussian.
Now Lemma \ref{lem:true-function-bound} implies that, with probability at least $1-\delta/3$, uniformly over all $z \in \cZ$, $l \ge 1$ and $1 \le h \le H$,
\beqn
\abs{\overline{R}_\star(z)-\mu_{R,l,h-1}(z)}\le \Bigg(\norm{\overline{R}_\star}_{k_R} + \dfrac{\sigma_R}{\sqrt{H}}\sqrt{2\Big(\ln(3/\delta)+\dfrac{1}{2}\sum_{(j,k)=(1,1)}^{(l,h-1)}\ln\big(1+H^{-1}\sigma_{R,j,k-1}^2(z_{j,k})\big)\Big)}\Bigg)\sigma_{R,l,h-1}(z).
\eeqn
Again, from Lemma \ref{lem:sum-of-sd}, we have
\beqn
\frac{1}{2}\sum_{(j,k)=(1,1)}^{(l,h-1)}\ln\big(1+H^{-1}\sigma_{R,j,k-1}^2(z_{j,k})\big) \le \gamma_{(l-1)H+h-1}(k_R,\cZ),
\eeqn
where $\gamma_{t}(k_R,\cZ)$ denotes the maximum information gain about an $f \sim GP_{\cZ}(0,k_R)$ after $t$ noisy observations with iid Gaussian noise $\cN(0,H)$. Therefore, with probability at least $1-\delta/3$, uniformly over all $z \in \cZ$, $l \ge 1$ and $1 \le h \le H$,
\beq
\label{eqn:reward-bound}
\abs{\overline{R}_\star(z)-\mu_{R,l,h-1}(z)}\le \Big(B_R + \dfrac{\sigma_R}{\sqrt{H}}\sqrt{2\big(\ln(3/\delta)+\gamma_{(l-1)H+h-1}(k_R,\cZ)}\big)\Big)\sigma_{R,l,h-1}(z),
\eeq
since by our hypothesis $\norm{\overline{R}_\star}_{k_R} \le B_R$.
Now see that $\mu_{R,l,0}=\mu_{R,l-1}$ and $\sigma_{R,l,0}=\sigma_{R,l-1}$ and $\beta_{R,l}=B_R + \dfrac{\sigma_R}{\sqrt{H}}\sqrt{2\big(\ln(3/\delta)+\gamma_{(l-1)H}(k_R,\lambda_R,\cZ)\big)}$  for every $l \ge 1$. Hence (\ref{eqn:concentartion-agnostic-reward}) follows by using (\ref{eqn:reward-bound}) with $h=1$.

Further for all $l \ge 1$ and $1 \le h \le H+1$, let $S_{l,h} = [s_{1,2}^T,\ldots,s_{1,H+1}^T,\ldots,s_{l,2}^T,\ldots,s_{l,h}^T]^T$ denotes the vector of state transitions at $\cZ_{l,h-1}=\lbrace z_{1,1},\ldots,z_{1,H},\ldots,z_{l,1},\ldots,z_{l,h-1}\rbrace$, where every state $s_{j,h}=[s_{j,h}(1),\ldots,s_{j,h}(m)]^T, 1 \le j \le l, 1\le h \le H+1$, is an $m$ - dimensional vector.  
Further for all $l \ge 1$, $1 \le h \le H$ and $1 \le b \le m+1$, define the following set:
\beqan
\tilde{\cZ}_{l,h,b-1}&=&\big\lbrace (z_{j,k},i) \big\rbrace_{1 \le j \le l-1,1 \le k \le H, 1 \le i \le m} \cup \big\lbrace (z_{l,k},i) \big\rbrace_{1 \le k \le h-1, 1 \le i \le m} \cup \big\lbrace (z_{l,h},i) \big\rbrace_{1 \le i \le b-1}\\
&=& \big \lbrace \cZ_{l,h-1}\times \lbrace 1,\ldots,m \rbrace \big \rbrace \cup \big \lbrace \lbrace z_{l,h} \rbrace \times \lbrace 1,\ldots,b-1 \rbrace \big \rbrace\\
&=&\big\lbrace  (z_{1,1},1),\ldots,(z_{1,1},m),\ldots,(z_{l,h-1},1),\ldots,(z_{l,h-1},m),(z_{l,h},1),\ldots,(z_{l,h},b-1)\big\rbrace,
\eeqan
and the following vector
\beqan
S_{l,h,b-1}&=&[S_{l,h}^T,s_{l,h+1}(1),\ldots,s_{l,h+1}(b-1)]^T\\
&=& [s_{1,2}^T,\ldots,s_{1,H+1}^T,\ldots,s_{l,2}^T,\ldots,s_{l,h}^T,s_{l,h+1}(1),\ldots,s_{l,h+1}(b-1)]^T\\
&=&[s_{1,2}(1),\ldots,s_{1,2}(m),\ldots,s_{l,h}(1),\ldots,s_{l,h}(m),s_{l,h+1}(1),\ldots,s_{l,h+1}(b-1)]^T.
\eeqan 
 See that $\cZ_{l,h,0}= \cZ_{l,h-1,m}$ and $S_{l,h,0}= S_{l,h-1,m}$ for all $l \ge 1$ and $2 \le h \le H$. Further $\cZ_{l,1,0}= \cZ_{l-1,H,m}$ and $S_{l,1,0}= S_{l-1,H,m}$ for all $l \ge 2$. Also $\cZ_{1,1,0}= 0$ and $S_{1,1,0}= 0$. Now we define, for all $z \in \cZ$, $1 \le i \le m$, $l \ge 1$ and $1 \le h \le H+1$, the following:
\beq
\label{eqn:post-state-def}
\begin{aligned}
\mu_{P,l,h,b-1}(z,i) &= k_{P,l,h,b-1}(z,i)^T(K_{P,l,h,b-1} + mH I)^{-1}S_{l,h,b-1},\\ 
\sigma^2_{P,l,h,b-1}(z,i) &= k_P\big((z,i),(z,i)\big) - k_{P,l,h,b-1}(z,i)^T(K_{P,l,h,b-1} + mH I)^{-1} k_{P,l,h,b-1}(z,i),
\end{aligned}
\eeq
where $k_{P,l,h,b-1}(z,i) = \Big[k_P\big((z_{1,1},1),(z,i)\big),\ldots,k_P\big((z_{l,h},b-1),(z,i)\big)\Big]^T$ is the vector of kernel evaluations between $(z,i)$ and elements of the set $\tilde{\cZ}_{l,h,b-1}$, $K_{P,l,h,b-1} = \big[k_P(z,z')\big]_{z,z' \in \tilde{\cZ}_{l,h,b-1}}$ is the kernel matrix computed at $\tilde{\cZ}_{l,h,b-1}$. 
See that $\mu_{P,l,h,0}= \mu_{P,l,h-1,m}$ and $\sigma_{P,l,h,0}= \sigma_{P,l,h-1,m}$ for all $l \ge 1$ and $2 \le h \le H$. Further, 
$\mu_{P,l,1,0}=\mu_{P,l-1,H,m}$ and $\sigma_{P,l,1,0}=\sigma_{P,l-1,H,m}$ for all $l \ge 2$.
Also
$\mu_{P,1,1,0}(z,i)=0$ and $\sigma_{P,1,1,0}=k_P\big((z,i),(z,i)\big)$ for all $z \in \cZ$ and $1 \le i \le m$.

At the state-action pair $z_{l,h}$, the MDP transitions to the state $s_{l,h+1}$, where 
$s_{l,h+1}(i)=\overline{P}_\star(z_{l,h},i)+\epsilon_{P,l,h}(i)$, $1 \le i \le m$. Thus, we can view $s_{l,h+1}(i)$ as a noisy observation of $\overline{P}_\star$ at the query $(z_{l,h},i) \in \tilde{\cZ}$. Here, by our hypothesis, the mean transition function $\overline{P}_\star \in \cH_{\tilde{k}_P}(\tilde{\cZ})$ and the noise sequence $\lbrace \epsilon_{P,l,h}(i)\rbrace_{l\ge 1,1\le h \le H,1\le i \le m}$ is conditionally $\sigma_P$-sub-Gaussian. Now Lemma \ref{lem:true-function-bound} implies that, with probability at least $1-\delta/3$, uniformly over all $z \in \cZ$, $1 \le i \le m$, $l \ge 1$, $1 \le h \le H$ and $1 \le b \le m$:
\footnotesize
\beqn
\abs{\overline{P}_\star(z,i)-\mu_{P,l,h,b-1}(z,i)}\le \Bigg(\norm{\overline{P}_\star}_{k_P} + \dfrac{\sigma_P}{\sqrt{mH}}\sqrt{2\Big(\ln(3/\delta)+\dfrac{1}{2}\sum_{(j,k,q)=(1,1,1)}^{(l,h,b-1)}\ln\Big(1+\dfrac{\sigma_{P,j,k,q-1}^2(z_{j,k},q)}{mH}\Big)\Big)}\Bigg)\\
\sigma_{P,l,h,b-1}(z,i).
\eeqn
\normalsize
Again, from Lemma \ref{lem:sum-of-sd}, we have
\beqn
\frac{1}{2}\sum_{(j,k,q)=(1,1,1)}^{(l,h,b-1)}\ln\Big(1+\dfrac{\sigma_{P,j,k,q-1}^2(z_{j,k},q)}{mH}\Big) \le \gamma_{m(l-1)H+m(h-1)+b-1}(k_P,\tilde{\cZ}),
\eeqn
where $\gamma_{t}(k_P,\tilde{\cZ})$ denotes the maximum information gain about an $f \sim GP_{\tilde{\cZ}}(0,k_P)$ after $t$ noisy observations with iid Gaussian noise $\cN(0,mH)$.
Therefore, with probability at least $1-\delta/3$, uniformly over all $z \in \cZ$, $1 \le i \le m$, $l \ge 1$, $1 \le h \le H$ and $1 \le b \le m$,
\beq
\label{eqn:state-bound}
\abs{\overline{P}_\star(z,i)-\mu_{P,l,h,b-1}(z,i)}\le \Big(B_P + \dfrac{\sigma_P}{\sqrt{mH}}\sqrt{2\big(\ln(3/\delta)+\gamma_{m(l-1)H+m(h-1)+b-1}(k_P,\tilde{\cZ})\big)}\Big)\sigma_{P,l,h,b-1}(z,i),
\eeq
since by our hypothesis $\norm{\overline{P}_\star}_{k_P} \le B_P$.
Now see that $\mu_{P,l,1,0}=\mu_{P,l-1}$ and $\sigma_{P,l,1,0}=\sigma_{P,l-1}$ for every $l \ge 1$. 
Hence using (\ref{eqn:state-bound}) for $h=1$ and $b=1$, see that, with probability at least $1-\delta/3$, uniformly over all $z \in \cZ$, $1 \le i \le m$, and $l \ge 1$, 
\beqn
\abs{\overline{P}_\star(z,i)-\mu_{P,l-1}(z,i)}\le \Big(B_P + \dfrac{\sigma_P}{\sqrt{mH}}\sqrt{2\big(\ln(3/\delta)+\gamma_{m(l-1)H}(k_P,\tilde{\cZ})\big)}\Big)\sigma_{P,l-1}(z,i).
\eeqn
Now recall that  $\beta_{P,l} = B_P + \dfrac{\sigma_P}{\sqrt{mH}}\sqrt{2\big(\ln(3/\delta)+\gamma_{m(l-1)H}(k_P,\lambda_P,\tilde{\cZ})\big)}$, $\overline{P}_\star(z)=[\overline{P}_\star(z,1),\ldots,\overline{P}_\star(z,m)]^T$, $\mu_{P,l-1}(z)=[\mu_{P,l-1}(z,1),\ldots,\mu_{P,l-1}(z,m)]^T$ and $\sigma_{P,l-1}(z)= [\sigma_{P,l-1}(z,1),\ldots,\sigma_{P,l-1}(z,m)]^T$. Then, with probability at least $1-\delta/3$, uniformly over all $z \in \cZ$ and $l \ge 1$,
\beqn
\norm{\overline{P}_\star(z)-\mu_{P,l-1}(z)}_2 \le \sqrt{\sum_{i=1}^{m}\beta^2_{P,l}\sigma^2_{P,l-1}(z,i)}= \beta_{P,l}\sqrt{\sum_{i=1}^{m}\sigma^2_{P,l-1}(z,i)}=\beta_{P,l}\norm{\sigma_{P,l-1}(z)}_2,
\eeqn
and hence (\ref{eqn:concentartion-agnostic-state}) follows.
\end{proof}

\begin{mylemma}[Sum of predictive variances upper bounded by Maximum Information Gain]
\label{lem:predictive-variance-sum}
Let $\sigma_{R,l}$ and $\sigma_{P,l}$ be defined as in \ref{eqn:post-reward-agnostic} and \ref{eqn:post-state-agnostic} respectively and let the kernels $k_R$ and $k_P$ satisfy $k_R(z,z) \le 1$ and $k_P((z,i),(z,i)) \le 1$ for all $z \in \cZ$ and $1 \le i \le m$. Then for any $\tau \ge 1$,
\beqa
\label{eqn:sum-of-sd-reward}
\sum_{l=1}^{\tau}\sum_{h=1}^{H} \sigma_{R,l-1}(z_{l,h}) &\le & \sqrt{2 e \tau H^2\gamma_{\tau H}(k_R,\cZ)}\\
\label{eqn:sum-of-sd-state}
\sum_{l=1}^{\tau}\sum_{h=1}^{H} \norm{\sigma_{P,l-1}(z_{l,h})}_2 &\le & \sqrt{2em\tau H^2\gamma_{m\tau H}(k_P,\tilde{\cZ})},
\eeqa
where $\gamma_{t}(k_R,\cZ)$ denotes the maximum information gain about an $f \sim GP_{\cZ}(0,k_R)$ after $t$ noisy observations with iid Gaussian noise $\cN(0,H)$ and $\gamma_{t}(k_P,\tilde{\cZ})$ denotes the maximum information gain about an $f \sim GP_{\tilde{\cZ}}(0,k_P)$ after $t$ noisy observations with iid Gaussian noise $\cN(0,mH)$.
\end{mylemma}
\begin{proof}
Note that $\sigma_{R,l-1}(z) = \sigma_{R,l,0}(z)$, where $\sigma_{R,l,0}(z)$ is defined in (\ref{eqn:post-reward-def}). Now from (\ref{eqn:info-gain-two}), see that $\sigma^2_{R,l,0}(z) \le (1+1/H)\sigma^2_{R,l,1}(z) \le (1+1/H)^2\sigma^2_{R,l,2}(z)\le \cdots \le (1+1/H)^{H-1}\sigma^2_{R,l,H-1}(z)$, i.e. $\sigma^2_{R,l,0}(z) \le (1+1/H)^{h-1}\sigma^2_{R,l,h-1}(z)$ for all $z \in \cZ$ and $1 \le h \le H$. This implies 
\beqa
\sum_{l=1}^{\tau}\sum_{h=1}^{H} \sigma^2_{R,l-1}(z_{l,h})=\sum_{l=1}^{\tau}\sum_{h=1}^{H} \sigma^2_{R,l,0}(z_{l,h}) &\le& \sum_{l=1}^{\tau}\sum_{h=1}^{H} (1+1/H)^{h-1} \sigma^2_{R,l,h-1}(z_{l,h})\nonumber\\
&\le& (1+1/H)^{H-1}\sum_{l=1}^{\tau}\sum_{h=1}^{H}\sigma^2_{R,l,h-1}(z_{l,h})\nonumber\\
&\le& (1+1/H)^{H-1}(2H+1)\gamma_{\tau H}(k_R,\cZ)\nonumber\\
& \le & 2eH\gamma_{\tau H}(k_R,\cZ)\label{eqn:comb-one},
\eeqa
where the second last inequality follows from  (\ref{eqn:info-gain-three}) and last inequality is due to the fact that $(1+1/\alpha)^{\alpha} \le e$ and $(1+1/\alpha)^{-1}(2\alpha+1) \le 2\alpha$ for all $\alpha > 0$. Further by Cauchy-Schwartz inequality
\beq
\label{eqn:comb-two}
\sum_{l=1}^{\tau}\sum_{h=1}^{H} \sigma_{R,l-1}(z_{l,h}) \le \sqrt{\tau H \sum_{l=1}^{\tau}\sum_{h=1}^{H} \sigma^2_{R,l-1}(z_{l,h})}.
\eeq
Now (\ref{eqn:sum-of-sd-reward}) follows by combining (\ref{eqn:comb-one}) and (\ref{eqn:comb-two}).

Similarly Note that $\sigma_{P,l-1}(z,i) = \sigma_{P,l,1,0}(z,i)$, where $\sigma_{P,l,1,0}(z,i)$ is defined in (\ref{eqn:post-state-def}). Now from (\ref{eqn:info-gain-two}), see that
$\sigma^2_{P,l,1,0}(z,i)\le (1+1/mH)^{m(h-1)+b-1} \sigma^2_{P,l,h,b-1}(z,i)$ for all $z \in \cZ$, $1 \le i \le m$, $1 \le h \le H$ and $1 \le b \le m$. This implies
\beqa
\sum_{l=1}^{\tau}\sum_{h=1}^{H}\sum_{b=1}^{m} \sigma^2_{P,l-1}(z_{l,h},b)=\sum_{l=1}^{\tau}\sum_{h=1}^{H}\sum_{b=1}^{m} \sigma^2_{P,l,1,0}(z_{l,h},b) &\le& \sum_{l=1}^{\tau}\sum_{h=1}^{H}\sum_{b=1}^{m} (1+1/mH)^{m(h-1)+b-1} \sigma^2_{P,l,h,b-1}(z_{l,h},b)\nonumber\\
&\le& (1+1/mH)^{m(H-1)+m-1}\sum_{l=1}^{\tau}\sum_{h=1}^{H}\sum_{b=1}^{m}\sigma^2_{P,l,h,b-1}(z_{l,h},b)\nonumber\\
&\le& (1+1/mH)^{mH-1}(2mH+1)\gamma_{m\tau H}(k_P,\tilde{\cZ})\nonumber\\
&\le& 2emH\gamma_{m\tau H}(k_P,\tilde{\cZ})\label{eqn:comb-three},
\eeqa
where the second last inequality follows from  (\ref{eqn:info-gain-three}) and last inequality is due to the fact that $(1+1/\alpha)^{\alpha} \le e$ and $(1+1/\alpha)^{-1}(2\alpha+1) \le 2\alpha$ for all $\alpha > 0$. Further by Cauchy-Schwartz inequality
\beq
\label{eqn:comb-four}
\sum_{l=1}^{\tau}\sum_{h=1}^{H} \norm{\sigma_{P,l-1}(z_{l,h})}_2 \le \sqrt{\tau H \sum_{l=1}^{\tau}\sum_{h=1}^{H} \norm{\sigma_{P,l-1}(z_{l,h})}_2^2} = \sqrt{\tau H \sum_{l=1}^{\tau}\sum_{h=1}^{H}\sum_{b=1}^{m} \sigma^2_{P,l-1}(z_{l,h},b)}.
\eeq
Now (\ref{eqn:sum-of-sd-state}) follows by combining (\ref{eqn:comb-three}) and (\ref{eqn:comb-four}).
\end{proof}

\subsubsection{Frequentist Regret Bound for GP-UCRL in Kernelized MDPs: Proof of Theorem \ref{thm:regret-bound-RKHS}}

Note that at every episode $l$, GP-UCRL (Algorithm \ref{algo:GP-UCRL}) selects the policy $\pi_l$ such that
\beq
\label{eqn:GP-UCRL-rule}
V^{M_l}_{\pi_l,1}(s_{l,1})=\max_{\pi}\max_{M \in \cM_l}V^{M}_{\pi,1}(s_{l,1}),
\eeq
where $s_{l,1}$ is the initial state, $\cM_l$ is the family of MDPs constructed by GP-UCRL and $M_l$ is the most optimistic realization from $\cM_l$. 
Further see that the mean reward function $R_\star$ of the unknown MDP $M_\star$ lies in the RKHS $\cH_{k_R}(\cZ)$. Thus for all $z \in \cZ$,
\beq
\label{eqn:rkhs-norm-bound}
\abs{\overline{R}_\star(z)}=\abs{\inner{\overline{R}_\star}{k_R(z,\cdot)}_{k_R}} \le \norm{\overline{R}_\star}_{k_R}k_R(z,z) \le B_R, 
\eeq
where the first equality is due to the reproducing property of RKHS, the first inequality is the Cauchy-Schwartz inequality and the final inequality is due to hypothesis that $\norm{\overline{R}_\star}_{k_R} \le B_R$ and $k_R(z,z) \le 1$ for all $z \in \cZ$. Now, (\ref{eqn:rkhs-norm-bound}), Lemma \ref{lem:common} and Lemma \ref{lem:GP-UCRL} together imply that for any $0 < \delta \le 1$ and $\tau \ge 1$ , with probability at least $1-\delta/3$,
\beqa
\sum_{l=1}^{\tau}\big(V^{M_l}_{\pi_l,1}(s_{l,1})-V^{M_\star}_{\pi_l,1}(s_{l,1})\big) \le \sum_{l=1}^{\tau}\sum_{h=1}^{H}\Big( \abs{\overline{R}_{M_l}(z_{l,h}) - \overline{R}_\star(z_{l,h})}+L_{M_l} \norm{\overline{P}_{M_l}(z_{l,h}) - \overline{P}_\star(z_{l,h})}_2\Big)\nonumber\\+ (LD+2B_RH) \sqrt{2\tau H \ln(3/\delta)}\label{eqn:combine-rkhs-one}.
\eeqa
Now for each $l \ge 1$, we define the following events:
\beqan
E_{R,l} &\bydef& \big\lbrace \forall z \in \cZ, \abs{\overline{R}_\star(z)-\mu_{R,l-1}(z)}\le \beta_{R,l}\sigma_{R,l-1}(z)\big\rbrace,\\
E_{P,l} &\bydef& \big\lbrace \forall z \in \cZ, \norm{\overline{P}_\star(z)-\mu_{P,l-1}(z)}_2\le \beta_{P,l}\norm{\sigma_{P,l-1}(z)}_2\big\rbrace.
\eeqan 
By construction of the set of MDPs $\cM_l$ in Algorithm \ref{algo:GP-UCRL}, it follows that when both the events $E_{R,l}$ and $E_{P,l}$ hold for all $l \ge 1$, the unknown MDP $M_\star$ lies in $\cM_l$ for all $l \ge 1$. Thus (\ref{eqn:GP-UCRL-rule}) implies $V^{M_l}_{\pi_l,1}(s_{l,1}) \ge V^{M_\star}_{\pi_\star,1}(s_{l,1})$ for all $l \ge 1$. This in turn implies, for every episode $l \ge 1$,
\beq
\label{eqn:GP-UCRL-rule-imply}
V^{M_\star}_{\pi_\star,1}(s_{l,1})-V^{M_\star}_{\pi_l,1}(s_{l,1})\le V^{M_l}_{\pi_l,1}(s_{l,1})-V^{M_\star}_{\pi_l,1}(s_{l,1}).
\eeq
Further when $E_{R,l}$ holds for all $l \ge 1$, then
\beq
\label{eqn:combine-rkhs-two}
\abs{\overline{R}_{M_l}(z_{l,h}) - \overline{R}_\star(z_{l,h})} \le \abs{\overline{R}_{M_l}(z_{l,h})-\mu_{R,l-1}(z_{l,h})}+\abs{\overline{R}_\star(z_{l,h})-\mu_{R,l-1}(z_{l,h})}
\le 2\beta_{R,l}\;\sigma_{R,l-1}(z_{l,h}),
\eeq
since the mean reward function $\overline{R}_{M_l}$ lies in the confidence set $\cC_{R,l}$ (\ref{eqn:confidence-set-agnostic}). Similarly when $E_{P,l}$ holds for all $l \ge 1$,
\beq
\label{eqn:combine-rkhs-three}
\norm{\overline{P}_{M_l}(z_{l,h}) - \overline{P}_\star(z_{l,h})}_2 \le \norm{\overline{P}_{M_l}(z_{l,h})-\mu_{P,l-1}(z_{l,h})}_2+\norm{\overline{P}_\star(z_{l,h})-\mu_{P,l-1}(z_{l,h})}_2
\le 2 \beta_{P,l}\norm{\sigma_{P,l-1}(z_{l,h})}_2,
\eeq
since the mean transition function $\overline{P}_{M_l}$ lies in the confidence set $\cC_{P,l}$  (\ref{eqn:confidence-set-agnostic}).

Now combining (\ref{eqn:combine-rkhs-one}), (\ref{eqn:GP-UCRL-rule-imply}), (\ref{eqn:combine-rkhs-two}) and (\ref{eqn:combine-rkhs-three}), when both the events $E_{R,l}$ and $E_{P,l}$ hold for all $l \ge 1$, then with probability at least $1-\delta/3$,
\beqn
\sum_{l=1}^{\tau}\big(V^{M_\star}_{\pi_\star,1}(s_{l,1})-V^{M_\star}_{\pi_l,1}(s_{l,1})\big)\le 2\sum_{l=1}^{\tau}\sum_{h=1}^{H} \big(\beta_{R,l}\sigma_{R,l-1}(z_{l,h})+L_{M_l}\beta_{P,l}\norm{\sigma_{P,l-1}(z_{l,h})}_2\big)+ (LD+2B_RH) \sqrt{2\tau H \ln(3/\delta)}.
\eeqn
Now Lemma \ref{lem:kernel-least-squares} implies that $\prob{\forall l\ge 1, E_{R,l}} \ge 1-\delta/3$ and $\prob{\forall l\ge 1, E_{P,l}} \ge 1-\delta/3$. Hence, by a union bound, for any $\tau \ge 1$, with probability at least $1-\delta$, 
\beq
\label{eqn:GP-UCRL-agnostic}
\sum_{l=1}^{\tau}\big(V^{M_\star}_{\pi_\star,1}(s_{l,1})-V^{M_\star}_{\pi_l,1}(s_{l,1})\big) \le 2\beta_{R,\tau}\sum_{l=1}^{\tau}\sum_{h=1}^{H} \sigma_{R,l-1}(z_{l,h})+2L\beta_{P,\tau}\sum_{l=1}^{\tau}\sum_{h=1}^{H}\norm{\sigma_{P,l-1}(z_{l,h})}_2+ (LD+2B_RH) \sqrt{2\tau H \ln(3/\delta)}.
\eeq
Here we have used the fact that both $\beta_{R,l}$ and $\beta_{P,l}$ are non-decreasing with the number of episodes $l$ and that $L_{M_l} \le L$ by construction of $\cM_l$ (and since $M_l \in \cM_l$). Now from Lemma \ref{lem:predictive-variance-sum}, we have 
$\sum_{l=1}^{\tau}\sum_{h=1}^{H} \sigma_{R,l-1}(z_{l,h}) \le \sqrt{2 e \tau H^2\gamma_{\tau H}(k_R,\cZ)}$ and
$\sum_{l=1}^{\tau}\sum_{h=1}^{H} \norm{\sigma_{P,l-1}(z_{l,h})}_2 \le \sqrt{2em\tau H^2\gamma_{m\tau H}(k_P,\tilde{\cZ})}$. Therefore
with probability at least $1-\delta$, the cumulative regret of GP-UCRL after $\tau$ episodes, i.e. after $T=\tau H$ timesteps is
\beqan
Regret(T) &=& \sum_{l=1}^{\tau}\big(V^{M_\star}_{\pi_\star,1}(s_{l,1})-V^{M_\star}_{\pi_l,1}(s_{l,1})\big)\\
&\le & 2\beta_{R,\tau}\sqrt{2 e H\gamma_{T}(k_R,\cZ)T}+  2L\beta_{P,\tau}\sqrt{2emH \gamma_{mT}(k_P,\tilde{\cZ})T}+ (LD+2B_RH) \sqrt{2T\ln(3/\delta)},
\eeqan
where $\beta_{R,\tau}=B_R + \dfrac{\sigma_R}{\sqrt{H}}\sqrt{2\big(\ln(3/\delta)+\gamma_{(\tau-1)H}(k_R,\cZ)\big)}$ and $\beta_{P,\tau} =  B_P + \dfrac{\sigma_P}{\sqrt{mH}}\sqrt{2\big(\ln(3/\delta)+\gamma_{m(\tau-1)H}(k_P,\tilde{\cZ})\big)}$.
Now the result follows by defining $\gamma_T(R) \bydef \gamma_{T}(k_R,\cZ)$ and $\gamma_{mT}(P) \bydef \gamma_{mT}(k_P,\tilde{\cZ})$.


\subsubsection{Bayes Regret of PSRL under RKHS Priors: Proof of Theorem \ref{thm:regret-bound-PSRL}} 
\label{appendix:PSRL-agnostic}

$\Phi\equiv(\Phi_R,\Phi_P)$ is the distribution of the unknown MDP $M_\star=\lbrace \cS,\cA,R_\star,P_\star,H \rbrace$, where $\Phi_R$ and $\Phi_P$ are specified by distributions over real valued functions on $\cZ$ and $\tilde{\cZ}$ respectively with a sub-Gaussian noise model in the sense that
\begin{itemize}

\item The reward distribution is $R_\star:\cS \times \cA \ra \Real$, with mean $\overline{R}_\star \in H_{k_R}(\cZ)$, $\norm{\overline{R}_\star}_{k_R} \le B_R$ and additive $\sigma_R$-sub-Gaussian noise.
\item The transition distribution is $P_\star:\cS \times \cA \ra \cS$, with mean $\overline{P}_\star \in H_{\tilde{k}_P}(\tilde{\cZ})$, $\norm{\overline{P}_\star}_{k_P} \le B_P$ and component-wise additive and independent $\sigma_P$-sub-Gaussian noise.
\end{itemize} 


At the start of episode $l$, PSRL samples an MDP $M_l$ from  $\Phi_l$, where $\Phi_l\equiv (\Phi_{R,l},\Phi_{P,l})$ is the corresponding posterior distribution conditioned on the history of observations $\cH_{l-1}\bydef \lbrace s_{j,k},a_{j,k},r_{j,k}\rbrace_{1 \le j \le l-1,1 \le k \le H}$.  
Therefore, conditioned on $\cH_{l-1}$, both $M_\star$ and $M_l$ are identically distributed. Hence for any $\sigma(\cH_{l-1})$ measurable function $g$, $\expect{g(M_\star)\given \cH_{l-1}}=\expect{g(M_l)\given \cH_{l-1}}$ and hence by the \textit{tower property},
\beq
\label{eqn: TS-identity-1}
\expect{g(M_\star)}=\expect{g(M_l)}.
\eeq
See that, conditioned on $\cH_{l-1}$, the respective optimal policies $\pi_\star$ and $\pi_l$ of $M_\star$ and $M_l$ are identically distributed. Since $s_{l,1}$ is deterministic, (\ref{eqn: TS-identity-1}) implies that
$\expect{V^{M_\star}_{\pi_\star,1}(s_{l,1})}=\expect{V^{M_l}_{\pi_l,1}(s_{l,1})}$. Hence for every episode $l \ge 1$, 
\beqa
\expect{V^{M_\star}_{\pi_\star,1}(s_{l,1})-V^{M_\star}_{\pi_l,1}(s_{l,1})}&=&\expect{V^{M_\star}_{\pi_\star,1}(s_{l,1})-V^{M_l}_{\pi_l,1}(s_{l,1})}+\expect{V^{M_l}_{\pi_l,1}(s_{l,1})-V^{M_\star}_{\pi_l,1}(s_{l,1})}\nonumber\\
&=& \expect{V^{M_l}_{\pi_l,1}(s_{l,1})-V^{M_\star}_{\pi_l,1}(s_{l,1})}.
\label{eqn:regret-breakup-TS-one}
\eeqa
Now, from Lemma \ref{lem:common}, for any $\tau \ge 1$,
\beq
\expect{\sum_{l=1}^{\tau}\Big[V^{M_l}_{\pi_l,1}(s_{l,1})-V^{M_\star}_{\pi_l,1}(s_{l,1})\Big]} \le \expect{\sum_{l=1}^{\tau}\sum_{h=1}^{H}\Big[ \abs{\overline{R}_{M_l}(z_{l,h}) - \overline{R}_\star(z_{l,h})}+L_{M_l} \norm{\overline{P}_{M_l}(z_{l,h}) - \overline{P}_\star(z_{l,h})}_2+\Delta_{l,h}\Big]},
\label{eqn:regret-breakup-TS-two}
\eeq
where $z_{l,h} \bydef (s_{l,h},a_{l,h})$ and
$\Delta_{l,h} \bydef \mathbb{E}_{s' \sim P_\star(z_{l,h})}\Big[V^{M_l}_{\pi_l,h+1}(s')-V^{M_\star}_{\pi_l,h+1}(s')\Big]-\Big(V^{M_l}_{\pi_l,h+1}(s_{l,h+1})-V^{M_\star}_{\pi_l,h+1}(s_{l,h+1})\Big)$.
From (\ref{eqn:expect-martingale}), see that $\expect{\Delta_{l,h}\given \cG_{l,h-1},M_\star,M_l}=0$, where $\cG_{l,h-1} \bydef \cH_{l-1} \cup \lbrace s_{l,k},a_{l,k},r_{l,k},s_{l,k+1}\rbrace_{1 \le k \le h-1}$ denotes the history of all observations till episode $l$ and period $h-1$. Now by tower property $\expect{\Delta_{l,h}}=0, l \ge 1, 1 \le h \le H$. Hence, combining (\ref{eqn:regret-breakup-TS-one}) and (\ref{eqn:regret-breakup-TS-two}), for any $\tau \ge 1$,
\beq
\expect{\sum_{l=1}^{\tau}\Big[V^{M_\star}_{\pi_\star,1}(s_{l,1})-V^{M_\star}_{\pi_l,1}(s_{l,1})\Big]} \le \expect{\sum_{l=1}^{\tau}\sum_{h=1}^{H}\Big[ \abs{\overline{R}_{M_l}(z_{l,h}) - \overline{R}_\star(z_{l,h})}+L_{M_l} \norm{\overline{P}_{M_l}(z_{l,h}) - \overline{P}_\star(z_{l,h})}_2\Big]}.
\label{eqn:regret-breakup-TS-three}
\eeq
Now fix any $0 < \delta \le 1$ and for each $l \ge 1$, define two events
$E_{\star}\bydef\big\lbrace \overline{R}_\star \in \cC_{R,l},\overline{P}_\star \in \cC_{P,l} \; \forall l \ge 1\big\rbrace$ and $E_{M}\bydef\big\lbrace \overline{R}_{M_l} \in \cC_{R,l},\overline{P}_{M_l}\in \cC_{P,l}\; \forall l \ge 1\big\rbrace$, where $\cC_{R,l},\cC_{P,l},l\ge 1$ are the confidence sets constructed by GP-UCRL as defined in (\ref{eqn:confidence-set-agnostic}). Now from Lemma \ref{lem:kernel-least-squares}, $\prob{E_{\star}} \ge 1-2\delta/3$ and hence by (\ref{eqn: TS-identity-1}) $\prob{E_{M}} \ge 1-2\delta/3$. Further define $E \bydef E_{\star}\cap E_{M}$ and by union bound, see that
\beq
\label{eqn:use-one}
\prob{E^c}\le \prob{E_{\star}^c}+\prob{E_{M}^c} \le 4\delta/3.
\eeq
(\ref{eqn:regret-breakup-TS-three}) and (\ref{eqn:use-one}) together imply,
\beqa
\expect{\sum_{l=1}^{\tau}\Big[V^{M_\star}_{\pi_\star,1}(s_{l,1})-V^{M_\star}_{\pi_l,1}(s_{l,1})\Big]} &\le& \expect{\sum_{l=1}^{\tau}\sum_{h=1}^{H}\Big[ \abs{\overline{R}_{M_l}(z_{l,h}) - \overline{R}_\star(z_{l,h})}\given E}\nonumber\\&+& \expect{L_{M_l} \norm{\overline{P}_{M_l}(z_{l,h}) - \overline{P}_\star(z_{l,h})}_2\Big]\given E}+8\delta B_R\tau H/3,
\label{eqn:use-two}
\eeqa
where we have used that $V^{M_\star}_{\pi_\star,1}(s_{l,1})-V^{M_\star}_{\pi_l,1}(s_{l,1}) \le 2B_RH$, since $\abs{\overline{R}_\star(z)} \le B_R$ for all $z \in \cZ$.
Now from Lemma \ref{lem:kernel-least-squares} and construction of $\cC_{R,l}, l \ge 1$,
\beqa
\expect{\sum_{l=1}^{\tau}\sum_{h=1}^{H}\abs{\overline{R}_{M_l}(z_{l,h}) - \overline{R}_\star(z_{l,h})}\given E}
&\le & \sum_{l=1}^{\tau}\sum_{h=1}^{H}2\beta_{R,l}\sigma_{R,l-1}(z_{l,h})\nonumber\\
&\le & 2\beta_{R,\tau}\sum_{l=1}^{\tau}\sum_{h=1}^{H}\sigma_{R,l-1}(z_{l,h})\\ 
& \le & 2\beta_{R,\tau} \sqrt{2 e \tau H^2\gamma_{\tau H}(k_R,\cZ)}
\label{eqn:use-three},
\eeqa
where the last step follows from Lemma \ref{lem:predictive-variance-sum}. Now form (\ref{eqn: TS-identity-1}), $\expect{L_{M_l}}=\expect{L_\star}$ and therefore
$\expect{L_{M_l}\given E}\le \expect{L_{M_l}}/\prob{E}\le \expect{L_\star}/(1-4\delta/3)$. Similarly from Lemma \ref{lem:kernel-least-squares} and construction of $\cC_{P,l}, l \ge 1$, 
\beqa
\expect{\sum_{l=1}^{\tau}\sum_{h=1}^{H}L_{M_l}\norm{\overline{P}_{M_l}(z_{l,h}) - \overline{P}_\star(z_{l,h})}_2\given E} &\le & \sum_{l=1}^{\tau}\sum_{h=1}^{H}\expect{L_{M_l}\given E}2\beta_{P,l}\norm{\sigma_{P,l-1}(z_{l,h})}_2\nonumber\\
&\le & \dfrac{\expect{L_\star}}{1-4\delta/3} 2\beta_{P,\tau}\sum_{l=1}^{\tau}\sum_{h=1}^{H} \norm{\sigma_{P,l-1}(z_{l,h})}_2\nonumber \\
&\le & \dfrac{\expect{L_\star}}{1-4\delta/3} 2\beta_{P,\tau}\sqrt{2em\tau H^2\gamma_{m\tau H}(k_P,\tilde{\cZ})},
\label{eqn:use-four}
\eeqa
where the last step follows from Lemma \ref{lem:predictive-variance-sum}.
Combining (\ref{eqn:use-two}), (\ref{eqn:use-three}) and (\ref{eqn:use-four}), for any $0 < \delta \le 1$ and $\tau \ge 1$,
\beqan
\expect{\sum_{l=1}^{\tau}\Big[V^{M_\star}_{\pi_\star,1}(s_{l,1})-V^{M_\star}_{\pi_l,1}(s_{l,1})\Big]} \le 2\beta_{R,\tau} \sqrt{2 e \tau H^2\gamma_{\tau H}(k_R,\cZ)}+\dfrac{\expect{L_\star}}{1-4\delta/3} 2\beta_{P,\tau}\sqrt{2em\tau H^2\gamma_{m\tau H}(k_P,\tilde{\cZ})}\\+8\delta B_R\tau H/3,
\eeqan
where $\beta_{R,\tau}=B_R + \dfrac{\sigma_R}{\sqrt{H}}\sqrt{2\big(\ln(3/\delta)+\gamma_{(\tau-1)H}(k_R,\cZ)\big)}$ and $\beta_{P,\tau} =  B_P + \dfrac{\sigma_P}{\sqrt{mH}}\sqrt{2\big(\ln(3/\delta)+\gamma_{m(\tau-1)H}(k_P,\tilde{\cZ})\big)}$. 
See that the left hand side of the above is independent of $\delta$. Now using $\delta=1/\tau H$, the Bayes regret of PSRL after $\tau$ episodes, i.e. after $T=\tau H$ timesteps is
\beqan
\expect{Regret(T)}&=&\sum_{l=1}^{\tau}\expect{V^{M_\star}_{\pi_\star,1}(s_{l,1})-V^{M_\star}_{\pi_l,1}(s_{l,1})}\\
&\le & 2 \alpha_{R,\tau} \sqrt{2 e H\gamma_{T}(k_R,\cZ)T} + 3\expect{L_\star}\alpha_{P,\tau}\sqrt{2em H\gamma_{mT}(k_P,\tilde{\cZ})T}+3B_R,
\eeqan
since $1/(1-4/3\tau H) \le 3/2$ as $\tau \ge 2$, $H \ge 2$. Here $\alpha_{R,\tau}\bydef B_R + \dfrac{\sigma_R}{\sqrt{H}}\sqrt{2\big(\ln(3T)+\gamma_{(\tau-1)H}(k_R,\cZ)\big)}$, $\alpha_{P,\tau} =  B_P + \dfrac{\sigma_P}{\sqrt{mH}}\sqrt{2\big(\ln(3T)+\gamma_{m(\tau-1)H}(k_P,\tilde{\cZ})\big)}$. Now the result follows by defining $\gamma_T(R) \bydef \gamma_{T}(k_R,\cZ)$ and $\gamma_{mT}(P) \bydef \gamma_{mT}(k_P,\tilde{\cZ})$.

\section{BAYES REGRET UNDER GAUSSIAN PROCESS PRIORS}
\label{appendix:Bayesian}
In this section, we develop the Bayesian RL analogue of Gaussian process bandits, i.e., learning under the assumption that MDP dynamics and reward behavior are sampled according to Gaussian process priors. 

\subsection{Regularity and Noise assumptions}

Each of our results in this section will assume that the mean reward function $\overline{R}_\star$ and the mean transition function $\overline{P}_\star$ are randomly sampled from from Gaussian processes $GP_\cZ(0,k_R)$ and $GP_{\tilde{\cZ}}(0,k_P)$, respectively, where $\cZ \bydef \cS \times \cA$ and $\tilde{\cZ}\bydef \cZ \times \lbrace 1,\ldots,m \rbrace$. Further, we will assume the noise sequences $\left\lbrace\epsilon_{R,l,h}\right\rbrace_{l\ge 1, 1 \le h \le H}$ are iid Gaussian $\cN(0,\lambda_R)$ and $\left\lbrace\epsilon_{P,l,h}\right\rbrace_{l\ge 1, 1 \le h \le H}$ are iid Gaussian $\cN(0,\lambda_PI)$.
Note that the same GP priors and noise models were used to design our algorithms (see Section \ref{subsec:uncertainty}). Thus, in this case the algorithm is assumed to have exact knowledge of the data generating process (the `fully Bayesian' setup). 

Further, in order to achieve non-trivial regret for continuous state/action MDPs, we need the following smoothness assumptions similar to those made by \citet{srinivas2009gaussian} on the kernels. We assume that $\cS \subseteq [0,c_1]^m$ and $\cA \subseteq [0,c_2]^{n}$ are compact and convex, and that the kernels $k_R$ and $k_P$ satisfy\footnote{This assumption holds for stationary kernels $k(z,z')\equiv k(z-z')$ that are four times differentiable, such as SE and Mat$\acute{e}$rn kernels with $\nu \ge 2$.} the following high probability bounds on the derivatives of GP sample paths $\overline{R}_\star$ and $\overline{P}_\star$, respectively:

\beq
\label{eqn:kernel-reward}
\prob{\sup\limits_{z\in \cZ}\abs{\partial \overline{R}_\star(z)/\partial z_j} > L_R} \le a_R e^{-(L_R/b_R)^2}
\eeq
holds for all $ 1 \le j \le m+n$ and for any $L_R >0$ corresponding to some $a_R,b_R > 0$, and 

\beq
\label{eqn:kernel-state}
\prob{\sup\limits_{z\in \cZ}\abs{\partial \overline{P}_\star(z,i)/\partial z_j} > L_P} \le a_P e^{-(L_P/b_P)^2} 
\eeq
holds for al $1 \le j \le m+n$, $1 \le i \le m$ and for any $L_P >0$ corresponding to some $a_P,b_P > 0$. Also we assume that

\beq
\label{eqn:kernel-reward-two}
\prob{\sup_{z\in \cZ}\abs{\overline{R}_\star(z)} > L} < a e^{-(L/b)^2},
\eeq
holds for any $L \ge 0$ for some corresponding $a,b> 0$\footnote{This is a mild assumption \citep{bogunovic2016time} on the kernel $k_R$, since $\overline{R}_\star(z)$ is Gaussian and thus has exponential tails.}.

\subsection{Choice of confidence sets for GP-UCRL}
For any fixed $0 < \delta \le 1$, at the beginning of each episode $l$, GP-UCRL construct the confidence set $\cC_{R,l}$ as

\beq
\label{eqn:confidence-set-reward-gp}
\cC_{R,l}=\big\lbrace f: \abs{f(z)-\mu_{R,l-1}([z]_l)} \le \beta_{R,l}\sigma_{R,l-1}([z]_l) +1/l^2, \forall z \big\rbrace,
\eeq
where $\mu_{R,l-1}(z)$, $\sigma_{R,l-1}(z)$ are defined as in (\ref{eqn:post-reward}) and $\beta_{R,l} \bydef \sqrt{2\ln\big(\abs{\cS_l}\abs{\cA_l}\pi^2l^2/\delta\big)}$. Here $(\cS_l)_{l \ge 1}$ and $(\cA_l)_{l \ge 1}$ are suitable discretizations of state space $\cS$ and action space $\cA$ respectively, $[z]_l \bydef ([s]_l,[a]_l)$,
where $[s]_l$ is the closest point in $\cS_l$ to $s$ and $[a]_l$ is the closest point in $\cA_l$ to $a$. Also we have $\abs{\cS_l}=\max \Big\lbrace \Big(2c_1ml^2b_R\sqrt{ln\big(6(m+n)a_R/\delta\big)}\Big)^m, \Big(2c_1ml^2b_P\sqrt{ln\big(6m(m+n)a_P/\delta\big)}\Big)^m \Big\rbrace$ and \\ $\abs{\cA_l}=\max \Big\lbrace \Big(2c_2nl^2b_R\sqrt{ln\big(6(m+n)a_R/\delta\big)}\Big)^n,\Big(2c_2nl^2b_P\sqrt{ln\big(6m(m+n)a_P/\delta\big)}\Big)^n \Big\rbrace$.
\par
Similarly GP-UCRL construct the confidence set $\cC_{P,l}$ as
\beq
\label{eqn:confidence-set-state-gp}
\cC_{P,l}=\big\lbrace f: \norm{f(z)-\mu_{P,l-1}([z]_l)}_2 \le \beta_{P,l}\norm{\sigma_{P,l-1}([z]_l)}_2 +\frac{\sqrt{m}}{l^2}, \forall z \big\rbrace,
\eeq
where $\beta_{P,l} \bydef \sqrt{2\ln\big(\abs{\cS_l}\abs{\cA_l}m\pi^2l^2/\delta\big)}$, $\mu_{P,l}(z)\bydef [\mu_{P,l-1}(z,1),\ldots,\mu_{P,l-1}(z,m)]^T$ and $\sigma_{P,l}(z)\bydef [\sigma_{P,l-1}(z,1),\ldots,\sigma_{P,l-1}(z,m)]^T$ with $\mu_{P,l-1}(z,i)$ and $\sigma_{P,l-1}(z,i)$ be defined as in (\ref{eqn:post-state})

\subsection{Regret Bounds}

\begin{mytheorem}[Bayesian regret bound for GP-UCRL under GP prior]
\label{thm:regret-bound-GP}
Let $M_\star=\lbrace \cS,\cA,R_\star,P_\star,H\rbrace$ be an MDP with period $H$, state space $\cS \subseteq [0,c_1]^m$ and action space $\cA \subseteq [0,c_2]^{n}$, $m,n\in \mathbb{N}, c_1,c_2 > 0$. Let $\cS $ and $\cA $ be compact and convex. Let the mean reward function $\overline{R}_\star$ be a sample from $GP_\cZ(0,k_R)$, where $\cZ \bydef \cS \times \cA$ and let the noise variables $\epsilon_{R,l,h}$ be iid Gaussian $\cN(0,\lambda_R)$. Let the mean transition function $\overline{P}_\star$ be a sample from $GP_{\tilde{\cZ}}(0,k_P)$ , where $\tilde{\cZ}\bydef \cZ \times \lbrace 1,\ldots,m \rbrace$ and let the noise variables $\epsilon_{P,l,h}$ be iid Gaussian $\cN(0,\lambda_PI)$. Further let the kernel $k_R$ satisfy (\ref{eqn:kernel-reward}), (\ref{eqn:kernel-reward-two}) and the kernel $k_P$ satisfy  (\ref{eqn:kernel-state}). Also let $k_R(z,z) \le 1$, $k_P((z,i),(z,i)) \le 1$ for all $z \in \cZ$ and $1 \le i \le m$. Then for any $0 \le \delta \le 1$, GP-UCRL, with confidence sets (\ref{eqn:confidence-set-reward-gp}) and (\ref{eqn:confidence-set-state-gp}), enjoys, with probability at least $1-\delta$, the regret bound 

\beqan
Regret(T) \le  2\beta_{R,\tau}\exp\big(\gamma_{H-1}(R)\big)\sqrt{(2\lambda_R+1)\gamma_{T}(R)T}+ 2L\beta_{P,\tau}\exp\big(\gamma_{mH-1}(P)\big)\sqrt{(2\lambda_P+1) \gamma_{mT}(P)T}\\+ (L\sqrt{m}+1)H\pi^2/3+(LD+2C H) \sqrt{2T \ln(6/\delta)},
\eeqan
where $C\bydef b\sqrt{\ln(6a/\delta)}$, $\beta_{R,l} \bydef \sqrt{2\ln\big(\abs{\cS_l}\abs{\cA_l}\pi^2l^2/\delta\big)}$ and $\beta_{P,l} \bydef \sqrt{2\ln\big(\abs{\cS_l}\abs{\cA_l}m\pi^2l^2/\delta\big)}$.
\end{mytheorem}


\begin{mytheorem}[Bayes regret of PSRL under GP prior]
\label{thm:regret-bound-GP-PSRL}

Let $M_\star$ be an MDP as in Theorem \ref{thm:regret-bound-GP} and $\Phi$ be a (known) prior distibution over MDPs $M_\star$. Then the Bayes regret of PSRL (Algorithm \ref{algo:GP-PSRL}) satisfies
\beqan
\expect{Regret(T)} \le  2 \alpha_{R,\tau} \exp\big(\gamma_{H-1}(R)\big)\sqrt{(2\lambda_R+1) \gamma_{T}(R)T} + 3 \;\expect{L_\star}\alpha_{P,\tau} \exp\big(\gamma_{mH-1}(P)\big)\sqrt{(2\lambda_P+1) \gamma_{mT}(P)T}\\+3C+(1+\sqrt{m}\expect{L_\star})\pi^2 H,
\eeqan
where $C=\expect{\sup_{z\in \cZ}\abs{\overline{R}_\star(z)}}$, $\alpha_{R,\tau} \bydef \sqrt{2\ln\big(\abs{\cS_\tau}\abs{\cA_\tau}\pi^2\tau^2T\big)}$ and $\alpha_{P,\tau} \bydef \sqrt{2\ln\big(\abs{\cS_\tau}\abs{\cA_\tau}m\pi^2\tau^2T\big)}$.
\end{mytheorem}


\subsection{Regret Analysis}
Here the state space $\cS \subseteq [0,c_1]^{m}$ and the action space $\cA \subseteq [0,c_2]^{n}$ for $c_1,c_2 > 0$. Both $\cS$ and $\cA$ are assumed to be compact and convex. Then at every round $l$, we can construct (by Lemma 15 of \citet{desautels2014parallelizing}) two discretization sets $\cS_l$ and $\cA_l$ of $\cS$ and $\cA$ respectively, with respective sizes $\cS_l$ and $\cA_l$, such that for all $s \in \cS$ and $a \in \cA$, the following holds:
\beqan
\norm{s-[s]_l}_1 &\le& c_1m/\abs{\cS_l}^{1/m},\\
\norm{a-[a]_l}_1 &\le& c_2n/\abs{\cA_l}^{1/n},
\eeqan 
where $[s]_l \bydef \argmin_{s'\in \cS_l}\norm{s-s'}_1$, is the closest point in $S_l$ to $s$ and $[a]_l \bydef \argmin_{a'\in \cA_l}\norm{a-a'}_1$ is the closest point in $A_l$ to $a$ (in the sense of $1$-norm). Now for any $s \in \cS$ and $a \in \cA$, we define $z\bydef[s^T,a^T]^T$ and correspondingly $[z]_l\bydef\big[[s]_l^T,[a]_l^T\big]^T$. Further define $\cZ \bydef \cS \times \cA \bydef \lbrace z=[s^T,a^T]^T : s \in \cS, a \in \cA\rbrace$ and $\cZ_l \bydef \cS_l \times \cA_l \bydef \lbrace z=[s^T,a^T]^T : s \in \cS_l, a \in \cA_l\rbrace$. See that $z,[z]_l \in \Real^{m+n}$ and $\cZ,\cZ_l \subset \Real^{m+n}$.

\begin{mylemma}[Samples from GPs are Lipschitz]
\label{lem:lipschitz}
Let $\cS \subseteq [0,c_1]^m$ and $\cA \subseteq [0,c_2]^{n}$ be compact and convex, $m,n\in \mathbb{N}, c_1,c_2 > 0$. Let $\overline{R}_\star$ be a sample from $GP_\cZ(0,k_R)$, where $\cZ \bydef \cS \times \cA$, $\overline{P}_\star$ be a sample from $GP_{\tilde{\cZ}}(0,k_P)$, where $\tilde{\cZ}\bydef \cZ \times \lbrace 1,\ldots,m \rbrace$. Further let the kernels $k_R$ and $k_P$ satisfy (\ref{eqn:kernel-reward}) and (\ref{eqn:kernel-state}) respectively. Then, for any $0 < \delta \le 1$, the following holds:
\beqa
&&\prob{\forall z \in \cZ,\forall l \ge 1 \abs{\overline{R}_\star(z)-\overline{R}_\star([z]_l)} \le 1/l^2} \ge 1-\delta/6,\label{eqn:GP-lipschitz-reward}\\
&&\prob{\forall z \in \cZ,\forall 1 \le i \le m,\forall l \ge 1 \abs{\overline{P}_\star(z,i)-\overline{P}_\star([z]_l,i)} \le 1/l^2} \ge 1-\delta/6.
\label{eqn:GP-lipschitz-state}
\eeqa
\end{mylemma}

\begin{proof} From (\ref{eqn:kernel-reward}), recall the assumption on kernel $k_R$:
\beqn
\prob{\sup\limits_{z\in \cZ}\abs{\partial \overline{R}_\star(z)/\partial z_j} > L_R} \le a_R e^{-(L_R/b_R)^2},  1 \le j \le m+n,
\eeqn
holds for any $L_R >0$ for some corresponding $a_R,b_R > 0$. 
 Now using union bound,
\beqn
\prob{\forall 1 \le j \le m+n \; \sup_{z \in \cZ} \abs{\partial \overline{R}_\star(z)/\partial z_j} \le L_R} \ge 1-(m+n)a_R e^{-(L_R/b_R)^2}.
\eeqn
From Mean-Value Theorem, this implies that with probability at least $1-(m+n) a_R e^{-(L_R/b_R)^2}$, 
\beqn
\forall z,z' \in \cZ, \abs{\overline{R}_\star(z)-\overline{R}_\star(z')} \le L_R \norm{z-z'}_1.
\eeqn
Therefore, with probability at least $1-(m+n) a_R e^{-(L_R/b_R)^2}$, 
\beqan
\forall l \ge 1,\forall z \in \cZ, \abs{\overline{R}_\star(z)-\overline{R}_\star([z]_l)} &\le & L_R \norm{z-[z]_l}_1\\
&=& L_R\big(\norm{s-[s]_l}_1+\norm{a-[a]_l}_1\big)\\
&\le& L_R\big(c_1m/\abs{\cS_l}^{1/m}+c_2n/\abs{\cA_l}^{1/n}\big).
\eeqan
Now for any $0 < \delta \le 1$ choose $L_R=b_R\sqrt{ln\big(6(m+n)a_R/\delta\big)}$. Then with probability at least $1-\delta/6$,
\beq
\label{eqn:combine-one-GP}
\forall l \ge 1,\forall z \in \cZ , \abs{\overline{R}_\star(z)-\overline{R}_\star([z]_l)} \le b_R\sqrt{ln\big(6(m+n)a_R/\delta\big)}\Big(c_1m/\abs{\cS_l}^{1/m}+c_2n/\abs{\cA_l}^{1/n}\Big).
\eeq

Similarly from (\ref{eqn:kernel-state}), recall the assumption on kernel $k_P$:
\beqn
\prob{\sup\limits_{z\in \cZ}\abs{\partial \overline{P}_\star(z,i)/\partial z_j} > L_P} \le a_P e^{-(L_P/b_P)^2}, 1 \le j \le m+n, 1 \le i \le m, 
\eeqn
holds for any $L_P >0$ for some corresponding $a_P,b_P > 0$. Now using union bound, 
\beqn
\prob{\forall 1 \le j \le m+n, \forall 1 \le i \le m \; \sup_{z \in \cZ} \abs{\partial \overline{P}_\star(z,i)/\partial z_j} \le L_P} \ge 1-m(m+n) a_P e^{-(L_P/b_P)^2}.
\eeqn
Now from Mean-Value Theorem, this implies that with probability at least $1-m(m+n) a_P e^{-(L_P/b_P)^2}$,
\beqn
\forall z,z' \in \cZ,\forall 1 \le i \le m, \abs{\overline{P}_\star(z,i)-\overline{P}_\star(z',i)} \le L_P \norm{z-z'}_1.
\eeqn
Therefore, with probability at least $1-m(m+n) a_P e^{-(L_P/b_P)^2}$, we have
\beqan
\forall l \ge 1,\forall z \in \cZ,\forall 1 \le i \le m, \abs{\overline{P}_\star(z,i)-\overline{P}_\star([z]_l,i)} &\le & L_P\norm{z-[z]_l}_1\\
&=& L_P\big(\norm{s-[s]_l}_1+\norm{a-[a]_l}_1\big)\\
&\le& L_P\big(c_1m/\abs{\cS_l}^{1/m}+c_2n/\abs{\cA_l}^{1/n}\big).
\eeqan
Now choose $L_P=b_P\sqrt{ln\big(6m(m+n)a_P/\delta\big)}$. Then with probability at least $1-\delta/6$,
\beq
\label{eqn:combine-two-GP}
\forall l \ge 1,\forall z \in \cZ,\forall 1 \le i \le m, \abs{\overline{P}_\star(z,i)-\overline{P}_\star([z]_l,i)} \le b_P\sqrt{ln\big(6m(m+n)a_P/\delta\big)}\big(c_1m/\abs{\cS_l}^{1/m}+c_2n/\abs{\cA_l}^{1/n}\big).
\eeq
Now by using $\abs{\cS_l}=\max \Big\lbrace \Big(2c_1ml^2b_R\sqrt{ln\big(6(m+n)a_R/\delta\big)}\Big)^m, \Big(2c_1ml^2b_P\sqrt{ln\big(6m(m+n)a_P/\delta\big)}\Big)^m \Big\rbrace$ and $\abs{\cA_l}=\max \Big\lbrace \Big(2c_2nl^2b_R\sqrt{ln\big(6(m+n)a_R/\delta\big)}\Big)^n,\Big(2c_2nl^2b_P\sqrt{ln\big(6m(m+n)a_P/\delta\big)}\Big)^n \Big\rbrace$ in (\ref{eqn:combine-one-GP}) and (\ref{eqn:combine-two-GP}), we get \ref{eqn:GP-lipschitz-reward} and \ref{eqn:GP-lipschitz-state} respectively.
\end{proof}

Now recall that at each episode $l \ge 1$, GP-UCRL constructs confidence sets $\cC_{R,l}$ and $\cC_{P,l}$ as 
\beq
\label{eqn:confidence-set-GP}
\begin{aligned}
\cC_{R,l}&=\big\lbrace f: \cZ \ra \Real \given \forall z \in \cZ, \abs{f(z)-\mu_{R,l-1}([z]_l)} \le \beta_{R,l}\sigma_{R,l-1}([z]_l) +1/l^2 \big\rbrace,\\
\cC_{P,l}&=\big\lbrace f: \cZ \ra \Real^m \given \forall z \in \cZ,\norm{f(z)-\mu_{P,l-1}([z]_l)}_2 \le \beta_{P,l}\norm{\sigma_{P,l-1}([z]_l)}_2 +\sqrt{m}/l^2 \big\rbrace,
\end{aligned}
\eeq
where $\mu_{R,0}(z)=0$, $\sigma^2_{R,0}(z)=k_R(z,z)$ and for each $l \ge 1$,  
\beq
\label{eqn:post-reward-GP}
\begin{aligned}
\mu_{R,l}(z) &= k_{R,l}(z)^T(K_{R,l} + \lambda_RI)^{-1}R_{l},\\
\sigma^2_{R,l}(z) &= k_R(z,z) - k_{R,l}(z)^T(K_{R,l} + \lambda_R I)^{-1} k_{R,l}(z).
\end{aligned}
\eeq
Here $I$ is the $(lH)\times (lH)$ identity matrix, $R_l = [r_{1,1},\ldots,r_{l,H}]^T$ is the vector of rewards observed at $\cZ_{l} = \lbrace z_{j,k} \rbrace_{1 \le j \le l,1 \le k \le H} =\lbrace z_{1,1},\ldots,z_{l,H}\rbrace$, the set of all state-action pairs available at the end of episode $l$. $k_{R,l}(z) = [k_R(z_{1,1},z),\ldots,k_R(z_{l,H},z)]^T$ is the vector kernel evaluations between $z$ and elements of the set $\cZ_{l}$ and $K_{R,l} = [k_R(u,v)]_{u,v \in \cZ_{l}}$ is the kernel matrix computed at $\cZ_{l}$. Further
$\mu_{P,l}(z) = [\mu_{P,l-1}(z,1),\ldots,\mu_{P,l-1}(z,m)]^T$ and $\sigma_{P,l}(z) = [\sigma_{P,l-1}(z,1),\ldots,\sigma_{P,l-1}(z,m)]^T$, where $\mu_{P,0}(z,i)=0$, $\sigma_{P,0}(z,i)=k_P\big((z,i),(z,i)\big)$ and for each $l \ge 1$,
\beq
\label{eqn:post-state-GP}
\begin{aligned}
\mu_{P,l}(z,i) &= k_{P,l}(z,i)^T(K_{P,l} + \lambda_P I)^{-1}S_l,\\
\sigma^2_{P,l}\big((z,i))\big) &= k_P((z,i),(z,i)) - k_{P,l}(z,i)^T(K_{P,l} + \lambda_P I)^{-1} k_{P,l}(z,i).
\end{aligned}
\eeq
Here $I$ is the $(mlH)\times (mlH)$ identity matrix, $S_{l} = [s_{1,2}^T,\ldots,s_{l,H+1}^T]^T$ denotes the vector of state transitions at $\cZ_{l}=\lbrace z_{1,1},\ldots,z_{l,H}\rbrace$, the set of all state-action pairs available at the end of episode $l$. $k_{P,l}(z,i) = \Big[k_P\big((z_{1,1},1),(z,i)\big),\ldots,k_P\big((z_{l,H},m),(z,i)\big)\Big]^T$ is the vector of kernel evaluations between $(z,i)$ and elements of the set $\tilde{\cZ}_{l} = \big\lbrace (z_{j,k},i) \big\rbrace_{1 \le j \le l,1 \le k \le H, 1 \le i \le m} =\big\lbrace  (z_{1,1},1),\ldots,(z_{l,H},m)\big\rbrace $ and $K_{P,l} = \big[k_P(u,v)\big]_{u,v \in \tilde{\cZ}_{l}}$ is the kernel matrix computed at $\tilde{\cZ}_{l}$.
Here for any $0 < \delta \le 1$ ,
$\beta_{R,l} \bydef \sqrt{2\ln\big(\abs{\cS_l}\abs{\cA_l}\pi^2l^2/\delta\big)}$ and $\beta_{P,l} \bydef \sqrt{2\ln\big(\abs{\cS_l}\abs{\cA_l}m\pi^2l^2/\delta\big)}$ are properly chosen confidence parameters of $\cC_{R,l}$ and $\cC_{P,l}$ respectively, where both $\abs{\cS_l}$ and $\abs{\cA_l}$ are approximately $O\Big(\big(l^2\ln(1/\delta)\big)^d\Big)$ with $d=\max\lbrace m,n \rbrace$.

\begin{mylemma}[Posterior Concentration of Gaussian Processes]
\label{lem:gp-concentration}
Let $M_\star=\lbrace \cS,\cA,R_\star,P_\star,H\rbrace$ be an MDP with period $H$, state space $\cS \subseteq [0,c_1]^m$ and action space $\cA \subseteq [0,c_2]^{n}$, $m,n\in \mathbb{N}, c_1,c_2 > 0$. Let $\cS $ and $\cA $ be compact and convex. Let the mean reward function $\overline{R}_\star$ be a sample from $GP_\cZ(0,k_R)$, where $\cZ \bydef \cS \times \cA$ and let the noise variables $\epsilon_{R,l,h}$ be iid Gaussian $\cN(0,\lambda_R)$. Let the mean transition function $\overline{P}_\star$ be a sample from $GP_{\tilde{\cZ}}(0,k_P)$ , where $\tilde{\cZ}\bydef \cZ \times \lbrace 1,\ldots,m \rbrace$ and let the noise variables $\epsilon_{P,l,h}$ be iid Gaussian $\cN(0,\lambda_PI)$. Further let the kernels $k_R$ and $k_P$ satisfy (\ref{eqn:kernel-reward}) and (\ref{eqn:kernel-state}) respectively. Then, for any $0 < \delta \le 1$, the following holds:
\beqa
\prob{\forall z \in \cZ, \forall l \ge 1, \abs{\overline{R}_\star(z)-\mu_{R,l-1}([z]_l)}\le \beta_{R,l}\;\sigma_{R,l-1}([z]_l)+1/l^2}&\ge& 1-\delta/3,\label{eqn:concentartion-agnostic-reward-specified}\\
\prob{\forall z \in \cZ, \forall l \ge 1, \norm{\overline{P}_\star(z)-\mu_{P,l-1}([z]_l)}_2 \le \beta_{P,l}\norm{\sigma_{P,l-1}([z]_l)}_2+\sqrt{m}/l^2} &\ge& 1-\delta/3,\label{eqn:concentartion-agnostic-state-specified}
\eeqa
\end{mylemma}
\begin{proof} Note that conditioned on $\cH_{l-1} \bydef \lbrace s_{j,k},a_{j,k},r_{j,k},s_{j,k+1}\rbrace_{1 \le j \le l-1,1 \le k \le H}$, $\overline{R}_\star(z) \sim \cN\big(\mu_{R,l-1}(z),\sigma^2_{R,l-1}(z)\big)$. If $a \sim \cN(0,1)$, $c \ge 0$, then $\prob{\abs{a}\ge c} \le \exp(-c^2/2)$. Using this Gaussian concentration inequality and a union bound over all $l \ge 1$ and all $z \in \cZ_l$, with probability at least $1-\delta/6$, we have
\beq
\label{eqn:GP-posterior-reward-concentration}
\forall l \ge 1, \forall z \in \cZ_l, \abs{\overline{R}_\star(z)-\mu_{R,l-1}(z)} \le \beta_{R,l}\sigma_{R,l-1}(z).
\eeq
Now as $[z]_l\in \cZ_l$, using union bound in  (\ref{eqn:GP-lipschitz-reward}) and (\ref{eqn:GP-posterior-reward-concentration}), we have with probability at least $1-\delta/3$,
\beqn
\forall l \ge 1, \forall z \in \cZ, \abs{\overline{R}_\star(z)-\mu_{R,l-1}([z]_l)} \le \beta_{R,l}\sigma_{R,l-1}([z]_l)+1/l^2.
\eeqn
Similarly, conditioned on $\cH_{l-1}$, $\overline{P}_\star(z,i) \sim \cN\big(\mu_{R,l-1}(z,i),\sigma^2_{P,l-1}(z,i)\big)$ for all $z \in \cZ$ and $1 \le i \le m$. Then using the Gaussian concentration inequality and a union bound over all $l \ge 1$, all $z \in \cZ_l$ and all $i=1,\ldots,m$, with probability at least $1-\delta/6$, we have
\beq
\label{eqn:GP-posterior-state-concentration}
\forall l \ge 1, \forall z \in \cZ_l, \forall 1 \le i \le m, \abs{\overline{P}_\star(z,i)-\mu_{P,l-1}(z,i)} \le \beta_{P,l}\sigma_{P,l-1}(z,i).
\eeq
Now as $[z]_l\in \cZ_l$, using union bound with (\ref{eqn:GP-lipschitz-state}) and (\ref{eqn:GP-posterior-state-concentration}), we have with probability at least $1-\delta/3$,
\beqn
\forall l \ge 1, \forall z \in \cZ, \forall 1 \le i \le m, \abs{\overline{P}_\star(z,i)-\mu_{P,l-1}([z]_l,i)} \le \beta_{P,l}\sigma_{P,l-1}([z]_l,i)+1/l^2.
\eeqn
Now Recall that  $\overline{P}_\star(z)=[\overline{P}_\star(z,1),\ldots,\overline{P}_\star(z,m)]^T$, $\mu_{P,l-1}(z)=[\mu_{P,l}(z,1),\ldots,\mu_{P,l-1}(z,m)]^T$ and $\sigma_{P,l-1}(z)=[\sigma_{P,l}(z,1),\ldots,\tilde{\sigma}_{P,l-1}(z,m)]^T$. Then with probability at least $1-\delta/3$, for all $l \ge 1$ and for all $z \in \cZ$,
\beqan
\norm{\overline{P}_\star(z)-\mu_{P,l-1}([z]_l)}_2  \le
\sqrt{\sum_{i=1}^{m}\Big(\beta_{P,l}\sigma_{P,l-1}([z]_l,i)+\dfrac{1}{l^2}\Big)^2}
&\le &\sqrt{\sum_{i=1}^{m}\beta^2_{P,l}\sigma^2_{P,l-1}([z]_l,i)}+\sqrt{\sum_{i=1}^{m}\dfrac{1}{l^4}}\\
&=&\beta_{P,l}\norm{\sigma_{P,l-1}([z]_l)}_2 +\dfrac{\sqrt{m}}{l^2}.
\eeqan
\end{proof}

\begin{mylemma}[Sum of predictive variances upper bounded by Maximum Information Gain]
\label{lem:predictive-variance-sum-GP}
Let $\sigma_{R,l}$ and $\sigma_{P,l}$ be defined as in \ref{eqn:post-reward-GP} and \ref{eqn:post-state-GP} respectively and let the kernels $k_R$ and $k_P$ satisfy $k_R(z,z) \le 1$ and $k_P((z,i),(z,i)) \le 1$ for all $z \in \cZ$ and $1 \le i \le m$. Then, for any $\tau \ge 1$, 
\beqa
\label{eqn:sum-of-sd-reward-GP}
\sum_{l=1}^{\tau}\sum_{h=1}^{H} \sigma_{R,l-1}([z_{l,h}]_l) &\le & \exp\big(\gamma_{H-1}(k_R,\cZ)\big)\sqrt{(2\lambda_R+1)\tau H \gamma_{\tau H}(k_R,\cZ)}\\
\label{eqn:sum-of-sd-state-GP}
\sum_{l=1}^{\tau}\sum_{h=1}^{H} \norm{\sigma_{P,l-1}([z_{l,h}]_l)}_2 &\le & \exp\big(\gamma_{mH-1}(k_P,\tilde{\cZ})\big)\sqrt{(2\lambda_P+1)\tau H \gamma_{m\tau H}(k_P,\tilde{\cZ})},
\eeqa
where $\gamma_{t}(k_R,\cZ)$ denotes the maximum information gain about an $f \sim GP_{\cZ}(0,k_R)$ after $t$ noisy observations with iid Gaussian noise $\cN(0,\lambda_R)$ and $\gamma_{t}(k_P,\tilde{\cZ})$ denotes the maximum information gain about an $f \sim GP_{\tilde{\cZ}}(0,k_P)$ after $t$ noisy observations with iid Gaussian noise $\cN(0,\lambda_P)$.
\end{mylemma}
\begin{proof}
Note that $\sigma_{R,l-1}(z) = \sigma_{R,l,0}(z)$, where $\sigma_{R,l,0}(z)$ is defined in (\ref{eqn:post-reward-def}). From Lemma \ref{lem:conditional-mi}, 
$\dfrac{\sigma_{R,l,0}(z)}{\sigma_{R,l,h-1}(z)} \le \exp\big(\gamma_{h-1}(k_R,\cZ)\big)$ for all $z \in \cZ$ and $1 \le h \le H$. This implies
\beqa
\sum_{l=1}^{\tau}\sum_{h=1}^{H} \sigma^2_{R,l-1}([z_{l,h}]_l)=\sum_{l=1}^{\tau}\sum_{h=1}^{H} \sigma^2_{R,l,0}([z_{l,h}]_l) &\le& \sum_{l=1}^{\tau}\sum_{h=1}^{H} \exp\big(2\gamma_{h-1}(k_R,\cX)\big) \sigma^2_{R,l,h-1}([z_{l,h}]_l)\nonumber\\
&\le& \exp\big(2\gamma_{H-1}(k_R,\cZ)\big)\sum_{l=1}^{\tau}\sum_{h=1}^{H}\sigma^2_{R,l,h-1}([z_{l,h}]_l)\nonumber\\
&\le&\exp\big(2\gamma_{H-1}(k_R,\cZ)\big)(2\lambda_R+1)\gamma_{\tau H}(k_R,\cZ),\label{eqn:comb-one-gp}
\eeqa
where the second last inequality follows from  (\ref{eqn:info-gain-three}). Further by Cauchy-Schwartz inequality
\beq
\label{eqn:comb-two-gp}
\sum_{l=1}^{\tau}\sum_{h=1}^{H} \sigma_{R,l-1}([z_{l,h}]_l) \le \sqrt{\tau H \sum_{l=1}^{\tau}\sum_{h=1}^{H} \sigma^2_{R,l-1}([z_{l,h}]_l)}.
\eeq
Now (\ref{eqn:sum-of-sd-reward-GP}) follows by combining (\ref{eqn:comb-one-gp}) and (\ref{eqn:comb-two-gp}).

Similarly Note that $\sigma_{P,l-1}(z,i) = \sigma_{P,l,1,0}(z,i)$, where $\sigma_{P,l,1,0}(z,i)$ is defined in (\ref{eqn:post-state-def}). Now from Lemma \ref{lem:conditional-mi}, we have 
$\dfrac{\sigma_{P,l,1,0}(z,i)}{\sigma_{P,l,h,b-1}(z,i)} \le \exp\big(\gamma_{m(h-1)+b-1}(k_P,\lambda_P,\tilde{\cZ})\big)$ for all $z \in \cZ$, $1 \le i \le m$, $1 \le h \le H$ and $1 \le b \le m$. This implies
\beqa
\sum_{l=1}^{\tau}\sum_{h=1}^{H}\sum_{b=1}^{m} \sigma^2_{P,l-1}([z_{l,h}]_l,b)&=&\sum_{l=1}^{\tau}\sum_{h=1}^{H}\sum_{b=1}^{m} \sigma^2_{P,l,1,0}([z_{l,h}]_l,b)\nonumber\\ 
&\le& \sum_{l=1}^{\tau}\sum_{h=1}^{H}\sum_{b=1}^{m} \exp\big(2\gamma_{m(h-1)+b-1}(k_P,\tilde{\cZ})\big) \sigma^2_{P,l,h,b-1}([z_{l,h}]_l,b)\nonumber\\
&\le& \exp\big(2\gamma_{m(H-1)+m-1}(k_P,\tilde{\cZ})\big)\sum_{l=1}^{\tau}\sum_{h=1}^{H}\sum_{b=1}^{m}\sigma^2_{P,l,h,b-1}([z_{l,h}]_l,b)\nonumber\\
&\le& \exp\big(2\gamma_{mH-1}(k_P,\tilde{\cZ})\big)(2\lambda_P+1)\gamma_{m\tau H}(k_P,\tilde{\cZ})\label{eqn:comb-three-gp},
\eeqa
where the second last inequality follows from  (\ref{eqn:info-gain-three}). Further by the Cauchy-Schwartz inequality 
\beq
\label{eqn:comb-four-gp}
\sum_{l=1}^{\tau}\sum_{h=1}^{H} \norm{\sigma_{P,l-1}([z_{l,h}]_l)}_2 \le \sqrt{\tau H \sum_{l=1}^{\tau}\sum_{h=1}^{H} \norm{\sigma_{P,l-1}([z_{l,h}]_l)}_2^2} = \sqrt{\tau H \sum_{l=1}^{\tau}\sum_{h=1}^{H}\sum_{b=1}^{m} \sigma^2_{P,l-1}([z_{l,h}]_l,b)}.
\eeq
Now (\ref{eqn:sum-of-sd-state-GP}) follows by combining (\ref{eqn:comb-three-gp}) and (\ref{eqn:comb-four-gp}). 
\end{proof}
\subsubsection{Bayesian Regret Bound for GP-UCRL under GP prior: Proof of Theorem \ref{thm:regret-bound-GP}}

Note that at every episode $l$, GP-UCRL (Algorithm \ref{algo:GP-UCRL}) selects the policy $\pi_l$ such that
\beq
\label{eqn:GP-UCRL-rule-bayes}
V^{M_l}_{\pi_l,1}(s_{l,1})=\max_{\pi}\max_{M \in \cM_l}V^{M}_{\pi,1}(s_{l,1}) 
\eeq
where $s_{l,1}$ is the initial state, $\cM_l$ is the family of MDPs constructed by GP-UCRL and $M_l$ is the most optimistic realization from $\cM_l$.
Further from \ref{eqn:kernel-reward-two} 
\beqn
\prob{\sup_{z\in \cZ}\abs{\overline{R}_\star(z)} > L} < a e^{-(L/b)^2},
\eeqn
holds for any $L \ge 0$ for some corresponding $a,b> 0$. Thus for any $0 < \delta \le 1$, setting $L=b\sqrt{\ln(6a/\delta)}$, with probability at least $1-\delta/6$, for all $z \in \cZ$ 
\beq
\label{eqn:gp-norm-bound}
\abs{\overline{R}_\star(z)} \le b\sqrt{\ln(6a/\delta)}.
\eeq
Now (\ref{eqn:gp-norm-bound}), Lemma \ref{lem:common} and Lemma \ref{lem:GP-UCRL} together with an union bound imply that for any $\tau \ge 1$ , with probability at least $1-\delta/3$,
\beqa
\sum_{l=1}^{\tau}\big(V^{M_l}_{\pi_l,1}(s_{l,1})-V^{M_\star}_{\pi_l,1}(s_{l,1})\big) \le \sum_{l=1}^{\tau}\sum_{h=1}^{H}\Big( \abs{\overline{R}_{M_l}(z_{l,h}) - \overline{R}_\star(z_{l,h})}+L_{M_l} \norm{\overline{P}_{M_l}(z_{l,h}) - \overline{P}_\star(z_{l,h})}_2\Big)\nonumber\\+ (LD+2C H) \sqrt{2\tau H \ln(6/\delta)}\label{eqn:combine-gp-one},
\eeqa
where $C \bydef b\sqrt{\ln(6a/\delta)}$.
Now for each $l \ge 1$, we define the following events:
\beqan
E_{R,l} &\bydef& \lbrace \forall z \in \cZ, \abs{\overline{R}_\star(z)-\mu_{R,l-1}([z]_l)}\le \beta_{R,l}\;\sigma_{R,l-1}([z]_l)+1/l^2\rbrace,\\
E_{P,l} &\bydef& \lbrace \forall z \in \cZ, \abs{\overline{P}_\star(z)-\mu_{P,l-1}([z]_l)}\le \beta_{P,l}\norm{\sigma_{P,l-1}([z]_l)}_2+\sqrt{m}/l^2\rbrace.
\eeqan 
By construction of the set of MDPs $\cM_l$ in Algorithm \ref{algo:GP-UCRL}, it follows that when both the events $E_{R,l}$ and $E_{P,l}$ hold for all $l \ge 1$, the unknown MDP $M_\star$ lies in $\cM_l$ for all $l \ge 1$. Thus (\ref{eqn:GP-UCRL-rule-bayes}) implies $V^{M_l}_{\pi_l,1}(s_{l,1}) \ge V^{M_\star}_{\pi_\star,1}(s_{l,1})$ for all $l \ge 1$. This in turn implies, for every episode $l \ge 1$,
\beq
\label{eqn:GP-UCRL-rule-imply-bayes}
V^{M_\star}_{\pi_\star,1}(s_{l,1})-V^{M_\star}_{\pi_l,1}(s_{l,1})\le V^{M_l}_{\pi_l,1}(s_{l,1})-V^{M_\star}_{\pi_l,1}(s_{l,1}).
\eeq
Further when $E_{R,l}$ holds for all $l \ge 1$, then
\beq
\label{eqn:combine-gp-two}
\abs{\overline{R}_{M_l}(z_{l,h}) - \overline{R}_\star(z_{l,h})} \le \abs{\overline{R}_{M_l}(z_{l,h})-\mu_{R,l-1}([z_{l,h}]_l)}+\abs{\overline{R}_\star(z_{l,h})-\mu_{R,l-1}([z_{l,h}]_l)}
\le 2\beta_{R,l}\sigma_{R,l-1}([z_{l,h}]_l)+2/l^2,
\eeq
since the mean reward function $\overline{R}_{M_l}$ lies in the confidence set $\cC_{R,l}$ as defined in (\ref{eqn:confidence-set-GP}). Similarly when $E_{P,l}$ holds for all $l \ge 1$,
\beqa
\label{eqn:combine-gp-three}
\norm{\overline{P}_{M_l}(z_{l,h}) - \overline{P}_\star(z_{l,h})}_2 &\le& \norm{\overline{P}_{M_l}(z_{l,h})-\mu_{P,l-1}([z_{l,h}]_l)}_2+\norm{\overline{P}_\star(z_{l,h})-\mu_{P,l-1}([z_{l,h}]_l)}_2\\
&\le &2 \beta_{P,l}\norm{\sigma_{P,l-1}([z_{l,h}]_l)}_2 + 2\sqrt{m}/l^2,
\eeqa
since the mean transition function $\overline{P}_{M_l}$ lies in the confidence set $\cC_{P,l}$ as defined in (\ref{eqn:confidence-set-GP}). Now combining (\ref{eqn:combine-gp-one}), (\ref{eqn:GP-UCRL-rule-imply-bayes}), (\ref{eqn:combine-gp-two}) and (\ref{eqn:combine-gp-three}), when both the events $E_{R,l}$ and $E_{P,l}$ hold for all $l \ge 1$, then with probability at least $1-\delta/3$,
\beqan
\sum_{l=1}^{\tau}\big(V^{M_\star}_{\pi_\star,1}(s_{l,1})-V^{M_\star}_{\pi_l,1}(s_{l,1})\big)\le 2\sum_{l=1}^{\tau}\sum_{h=1}^{H} \big(\beta_{R,l}\sigma_{R,l-1}([z_{l,h}]_l)+1/l^2+L_{M_l}\beta_{P,l}\norm{\sigma_{P,l-1}([z_{l,h}]_l)}_2+L_{M_l}\sqrt{m}/l^2\big)\\+ (LD+2CH) \sqrt{2\tau H \ln(6/\delta)}.
\eeqan
Now Lemma \ref{lem:gp-concentration} implies that $\prob{\forall l\ge 1, E_{R,l}} \ge 1-\delta/3$ and $\prob{\forall l\ge 1, E_{P,l}} \ge 1-\delta/3$. Hence, by a union bound, for any $\tau \ge 1$, with probability at least $1-\delta$, 
\beqa
\label{eqn:GP-UCRL-bays}
\sum_{l=1}^{\tau}\big(V^{M_\star}_{\pi_\star,1}(s_{l,1})-V^{M_\star}_{\pi_l,1}(s_{l,1})\big) &\le & 2\beta_{R,\tau}\sum_{l=1}^{\tau}\sum_{h=1}^{H} \sigma_{R,l-1}([z_{l,h}]_l)+2L\beta_{P,\tau}\sum_{l=1}^{\tau}\sum_{h=1}^{H}\norm{\sigma_{P,l-1}([z_{l,h}]_l)}_2\nonumber\\&& + (L\sqrt{m}+1)H\pi^2/3 + (LD+2CH) \sqrt{2\tau H \ln(6/\delta)}.
\eeqa
Here we have used the fact that both $\beta_{R,l}$ and $\beta_{P,l}$ are non-decreasing with the number of episodes $l$, $\sum_{l=1}^{\tau}1/l^2 \le \pi^2/6$ and that $L_{M_l} \le L$ by construction of $\cM_l$ (and since $M_l \in \cM_l$).  Now from Lemma \ref{lem:predictive-variance-sum-GP}, we have 
$\sum_{l=1}^{\tau}\sum_{h=1}^{H} \sigma_{R,l-1}([z_{l,h}]_l) \le \exp\big(\gamma_{H-1}(k_R,\cZ)\big)\sqrt{(2\lambda_R+1)\tau H \gamma_{\tau H}(k_R,\cZ)}$ and \\
$\sum_{l=1}^{\tau}\sum_{h=1}^{H} \norm{\sigma_{P,l-1}([z_{l,h}]_l)}_2 \le \exp\big(\gamma_{mH-1}(k_P,\tilde{\cZ})\big)\sqrt{(2\lambda_P+1)\tau H \gamma_{m\tau H}(k_P,\tilde{\cZ})}$. Therefore
with probability at least $1-\delta$, the cumulative regret of GP-UCRL after $\tau$ episodes, i.e. after $T=\tau H$ timesteps is
\beqan
Regret(T) &=& \sum_{l=1}^{\tau}\big(V^{M_\star}_{\pi_\star,1}(s_{l,1})-V^{M_\star}_{\pi_l,1}(s_{l,1})\big)\\
&\le & 2\beta_{R,\tau}\exp\big(\gamma_{H-1}(k_R,\cZ)\big)\sqrt{(2\lambda_R+1)\gamma_{T}(k_R,\cZ)T}\\&&+ 2L\beta_{P,\tau}\exp\big(\gamma_{mH-1}(k_P,\tilde{\cZ})\big)\sqrt{(2\lambda_P+1) \gamma_{mT}(k_P,\tilde{\cZ})T}\\&& + (L\sqrt{m}+1)H\pi^2/3+(LD+2C H) \sqrt{2T \ln(6/\delta)},
\eeqan
where $C\bydef b\sqrt{\ln(6a/\delta)}$, $\beta_{R,\tau} \bydef \sqrt{2\ln\big(\abs{\cS_\tau}\abs{\cA_\tau}\pi^2\tau^2/\delta\big)}$ and $\beta_{P,\tau} \bydef \sqrt{2\ln\big(\abs{\cS_\tau}\abs{\cA_\tau}m\pi^2\tau^2/\delta\big)}$.
Now the result follows by defining $\gamma_T(R) \bydef \gamma_{T}(k_R,\cZ)$ and $\gamma_{mT}(P) \bydef \gamma_{mT}(k_P,\tilde{\cZ})$.

\subsubsection{Bayes Regret of PSRL under GP prior: Proof of Theorem \ref{thm:regret-bound-GP-PSRL}}
\label{subsec:PSRL-specified}

$\Phi\equiv(\Phi_R,\Phi_P)$ is the distribution of the unknown MDP $M_\star=\lbrace \cS,\cA,R_\star,P_\star,H \rbrace$, where $\Phi_R$ and $\Phi_P$ are specified by GP priors $ GP_{\cZ}(0,k_R)$ are $GP_{\tilde{\cZ}}(0,k_P)$ respectively with a Gaussian noise model in the sense that
\begin{itemize}

\item The reward distribution is $R_\star:\cS \times \cA \ra \Real$, with mean $\overline{R}_\star \sim GP_{\cZ}(0,k_R)$, and additive $\cN(0,\lambda_R)$ Gaussian noise.
\item The transition distribution is $P_\star:\cS \times \cA \ra \cS$, with mean $\overline{P}_\star \sim GP_{\tilde{\cZ}}(0,\tilde{k}_P)$,  and component-wise additive and independent $\cN(0,\lambda_R)$ Gaussian noise.
\end{itemize} 
Conditioned on the history of observations $\cH_{l-1}\bydef \lbrace s_{j,k},a_{j,k},r_{j,k}\rbrace_{1 \le j \le l-1,1 \le k \le H}$ both $M_\star$ and $M_l$ are identically distributed with $\Phi_l\equiv (\Phi_{R,l}, \Phi_{P,l})$, where  $\Phi_{R,l}$ and $\Phi_{P,l}$ are specified by GP posteriors $GP_{\cZ}(\mu_{R,l-1},k_{R,l-1})$ and $GP_{\tilde{\cZ}}(\mu_{P,l-1},k_{P,l-1})$ respectively.

In this case we use the confidence sets $\cC_{R,l}$ and $\cC_{P,l}$ as given in \ref{eqn:confidence-set-GP} and define an event $E \bydef E_{\star}\cap E_{M}$, where 
$E_{\star}\bydef\big\lbrace \overline{R}_\star \in \cC_{R,l},\overline{P}_\star \in \cC_{P,l} \; \forall l \ge 1\big\rbrace$ and $E_{M}\bydef\big\lbrace \overline{R}_{M_l} \in \cC_{R,l},\overline{P}_{M_l}\in \cC_{P,l}\; \forall l \ge 1\big\rbrace$. Now from Lemma \ref{lem:gp-concentration}, $\prob{E_{\star}} \ge 1-2\delta/3$ and hence $\prob{E} \ge 1-4\delta/3$ similarly as in the proof of Theorem \ref{thm:regret-bound-PSRL}. Further (\ref{eqn:regret-breakup-TS-three}) implies
\beqa
\expect{\sum_{l=1}^{\tau}\Big[V^{M_\star}_{\pi_\star,1}(s_{l,1})-V^{M_\star}_{\pi_l,1}(s_{l,1})\Big]} &\le& \expect{\sum_{l=1}^{\tau}\sum_{h=1}^{H}\Big[ \abs{\overline{R}_{M_l}(z_{l,h}) - \overline{R}_\star(z_{l,h})}\given E}\nonumber\\&& +\expect{L_{M_l} \norm{\overline{P}_{M_l}(z_{l,h}) - \overline{P}_\star(z_{l,h})}_2\Big]\given E}+8\delta C\tau H/3,\quad
\label{eqn:use-two-GP}
\eeqa
where we have used that $\expect{V^{M_\star}_{\pi_\star,1}(s_{l,1})-V^{M_\star}_{\pi_l,1}(s_{l,1})} \le 2 C H$, where $C=\expect{\sup_{z\in \cZ}\abs{\overline{R}_\star(z)}}$.
From Lemma \ref{lem:gp-concentration} and construction of $\cC_{R,l}, l \ge 1$,
\beqa
\expect{\sum_{l=1}^{\tau}\sum_{h=1}^{H}\abs{\overline{R}_{M_l}(z_{l,h}) - \overline{R}_\star(z_{l,h})}\given E}
&\le & \sum_{l=1}^{\tau}\sum_{h=1}^{H}\Big(2\beta_{R,l}\sigma_{R,l-1}([z_{l,h}]_l)+2/l^2\Big)\nonumber\\
&\le & 2 \beta_{R,\tau} \exp\big(\gamma_{H-1}(k_R,\cZ)\big)\sqrt{(2\lambda_R+1)\tau H \gamma_{\tau H}(k_R,\cZ)}+\dfrac{\pi^2H}{3},\quad \quad \label{eqn:use-three-GP}
\eeqa
where the last step follows from Lemma \ref{lem:predictive-variance-sum-GP}. Similarly from Lemma \ref{lem:gp-concentration} and construction of $\cC_{P,l}, l \ge 1$, 
\beqa
&&\expect{\sum_{l=1}^{\tau}\sum_{h=1}^{H}L_{M_l}\norm{\overline{P}_{M_l}(z_{l,h}) - \overline{P}_\star(z_{l,h})}_2\given E}\nonumber \\ &\le&  \sum_{l=1}^{\tau}\sum_{h=1}^{H}\expect{L_{M_l}\given E}\Big(2\beta_{P,l}\norm{\sigma_{P,l-1}([z_{l,h}]_l)}_2+2\sqrt{m}/l^2\Big).\nonumber\\
&\le & \dfrac{\expect{L_\star}}{1-4\delta/3} \Big(2\beta_{P,\tau}\exp\big(\gamma_{mH-1}(k_P,\tilde{\cZ})\big)\sqrt{(2\lambda_P+1)\tau H \gamma_{m\tau H}(k_P,\tilde{\cZ})}+\sqrt{m}\pi^2H/3\Big)\label{eqn:use-four-GP},
\eeqa
where the last step follows from Lemma \ref{lem:predictive-variance-sum-GP} and from the proof of Theorem \ref{thm:regret-bound-PSRL}. Combining (\ref{eqn:use-two-GP}), (\ref{eqn:use-three-GP}) and (\ref{eqn:use-four-GP}), for any $0 < \delta \le 1$ and $\tau \ge 1$,
\beqan
\expect{\sum_{l=1}^{\tau}\Big[V^{M_\star}_{\pi_\star,1}(s_{l,1})-V^{M_\star}_{\pi_l,1}(s_{l,1})\Big]} &\le& 2 \beta_{R,\tau} \exp\big(\gamma_{H-1}(k_R,\cZ)\big)\sqrt{(2\lambda_R+1)\tau H \gamma_{\tau H}(k_R,\lambda_R,\cZ)}\\&&+2\beta_{P,\tau}\dfrac{\expect{L_\star}}{1-4\delta/3} \exp\big(\gamma_{mH-1}(k_P,\tilde{\cZ})\big)\sqrt{(2\lambda_P+1)\tau H \gamma_{m\tau H}(k_P,\tilde{\cZ})}\\&&+8\delta C\tau H/3+\frac{\big(1+\frac{\expect{L^\star}}{1-4\delta/3}\sqrt{m}\big)\pi^2 H}{3},
\eeqan
where $\beta_{R,\tau} \bydef \sqrt{2\ln\big(\abs{\cS_\tau}\abs{\cA_\tau}\pi^2\tau^2/\delta\big)}$ and $\beta_{P,\tau} \bydef \sqrt{2\ln\big(\abs{\cS_\tau}\abs{\cA_\tau}m\pi^2\tau^2/\delta\big)}$. See that the left hand side is independent of $\delta$. Now using $\delta=1/\tau H$, the Bayes regret of PSRL after $\tau$ episodes, i.e. after $T=\tau H$ timesteps is
\beqan
\expect{Regret(\tau)}&=&\sum_{l=1}^{\tau}\expect{V^{M_\star}_{\pi_\star,1}(s_{l,1})-V^{M_\star}_{\pi_l,1}(s_{l,1})}\\ &\le& 2 \alpha_{R,\tau} \exp\big(\gamma_{H-1}(k_R,\cZ)\big)\sqrt{(2\lambda_R+1) \gamma_{T}(k_R,\cZ)T} \\&&+ 3 \;\expect{L_\star}\alpha_{P,\tau} \exp\big(\gamma_{mH-1}(k_P,\tilde{\cZ})\big)\sqrt{(2\lambda_P+1) \gamma_{mT}(k_P,\tilde{\cZ})T}+3C+(1+\sqrt{m}\expect{L_\star})\pi^2 H,
\eeqan
since $1/(1-4/3\tau H) \le 3/2$ as $\tau \ge 2$, $H \ge 2$. Here $C=\expect{\sup_{z\in \cZ}\abs{\overline{R}_\star(z)}}$, $\alpha_{R,\tau} \bydef \sqrt{2\ln\big(\abs{\cS_\tau}\abs{\cA_\tau}\pi^2\tau^2T\big)}$, $\alpha_{P,\tau} \bydef \sqrt{2\ln\big(\abs{\cS_\tau}\abs{\cA_\tau}m\pi^2\tau^2T\big)}$. Now the result follows by defining $\gamma_T(R) \bydef \gamma_{T}(k_R,\cZ)$ and $\gamma_{mT}(P) \bydef \gamma_{mT}(k_P,\tilde{\cZ})$.

\end{appendices}
\end{document}